\documentclass[conference,compsoc]{IEEEtran}

\ifCLASSINFOpdf
\else
\fi

\usepackage{amsfonts}       
\usepackage{nicefrac}       
\usepackage{microtype}      
\usepackage{xcolor}         
\usepackage{bm}
\usepackage{mdframed}
\usepackage{algorithm, algorithmic}
\usepackage{amsfonts,amsmath}
\usepackage{subfig}
\usepackage{graphicx}
\usepackage{amsthm}
\usepackage{enumitem}
\usepackage{thm-restate}
\usepackage{hyperref}
\usepackage{comment} 

\usepackage{lipsum}

\usepackage{xcolor}
\usepackage{soul} 

\usepackage{environ}         
\usepackage{etoolbox}        
\usepackage{graphicx}        

\renewcommand{\paragraph}[1]{\medskip\noindent\textbf{#1.}}

\setlist[description]{leftmargin=\parindent,topsep=0ex,itemsep=0ex,partopsep=1ex,parsep=1ex}

\setitemize{noitemsep,topsep=2pt,parsep=2pt,partopsep=2pt,leftmargin=3ex}
\setenumerate{noitemsep,topsep=2pt,parsep=2pt,partopsep=2pt,leftmargin=4ex}

\usepackage[lambda,adversary,asymptotics,sets,landau,probability,operators]{cryptocode}
\usepackage{cleveref}

\let\OLDthebibliography\thebibliography
\renewcommand\thebibliography[1]{
  \OLDthebibliography{#1}
  \setlength{\parskip}{0pt}
  \setlength{\itemsep}{0pt plus 0.3ex}
}

\newtheorem{thm}{Theorem}[section]
\newtheorem{remark}{Remark}[section]
\newtheorem{assumption}{Assumption}[section]
\newtheorem{lemma}{Lemma}[section]
\newtheorem{definition}{Definition}[section]
\newtheorem{example}{Example}[section]
\newtheorem{cor}{Corollary}[section]

\author{
    \IEEEauthorblockN{Hanshen Xiao, Jun Wan, and Srinivas Devadas}
    \IEEEauthorblockA{MIT
    \\\{hsxiao, junwan, devadas\}@mit.edu}
    
}

\begin{document}
\pagenumbering{arabic}

\title{Differentially Private Deep Learning with ModelMix}
\maketitle

\begin{abstract}
Training large neural networks with meaningful/usable differential privacy security guarantees is a demanding challenge. 
In this paper, we tackle this problem by revisiting the two key operations in Differentially Private Stochastic Gradient Descent (DP-SGD): 1) iterative perturbation and 2) gradient clipping.  
We propose a generic optimization framework, called {\em ModelMix}, which performs random aggregation of intermediate model states.
It strengthens the composite privacy analysis utilizing the entropy of the training trajectory
    and improves the $(\epsilon, \delta)$ DP security parameters by an order of magnitude. 

We provide rigorous analyses for both the utility guarantees and privacy amplification of ModelMix. 
In particular, we present a formal study on the effect of gradient clipping in DP-SGD, which provides theoretical instruction on how hyper-parameters should be selected. 
We also introduce a refined gradient clipping method, which can further sharpen the privacy loss in private learning when combined with ModelMix.  

{
Thorough experiments with significant privacy/utility improvement are presented to support our theory.
We train a Resnet-20 network on CIFAR10 with $70.4\%$ accuracy via ModelMix given $(\epsilon=8, \delta=10^{-5})$ DP-budget, compared to the same performance but with $(\epsilon=145.8,\delta=10^{-5})$ using regular DP-SGD; assisted with additional public low-dimensional gradient embedding, one can further improve the accuracy to $79.1\%$ with $(\epsilon=6.1, \delta=10^{-5})$ DP-budget, compared to the same performance but with $(\epsilon=111.2, \delta=10^{-5})$ without ModelMix. 
}
\end{abstract}


%
\IEEEpeerreviewmaketitle

\section{Introduction}
\vspace{-0.05 in}
\noindent Privacy concerns when learning with sensitive data are receiving increasing attention. 
Many practical attacks have shown that without proper protection, the  model's parameters \cite{modelinversion2015, modelinversion2017}, leakage on gradients during training \cite{deepleakage}, or just observations on the prediction results \cite{inference_leakage} may enable an adversary to successfully distinguish and even reconstruct the private samples used for learning. 
As an emergent canonical definition, Differential Privacy (DP) \cite{dwork2006a, dwork2006b} provides a semantic privacy metric to quantify how hard it is for an adversary to infer the participation of an individual in an aggregate statistic. 
As one of the most popular approaches, Differentially-Private Stochastic Gradient Descent (DP-SGD) \cite{dpsgd, FOCS2014} and its variants \cite{CCS2016, Wang2017,heavy_tail2020,heavy_tail2021,amplification_iteration,balle2020privacy} have been widely studied over the last decade. 
DP-SGD can be applied to almost all optimization problems in machine learning to produce rigorous DP guarantees without additional assumptions regarding the objective function or dataset. 
However, despite its broad applicability, DP-SGD also suffers notoriously large utility loss especially when training cutting-edge deep models. 
Its practical implementation is also known to be sensitive to hyper-parameter selections \cite{imagesetDP,tempered2021,clipSong2020}. Indeed, even in theory, the effects of the two artificial privatization operations applied in DP-SGD, {\em iterative gradient perturbation} and {\em gradient clipping}, are still not fully-understood.

To understand why these two artificial modifications are the key to differentially privatize iterative methods, we need to first introduce the concept of {\em sensitivity}, which plays a key role in DP. 
The sensitivity captures the maximum impact an individual sample from an input dataset may have on an algorithm's output. 
It is the foundation of almost all DP mechanisms, including the Laplace/Gaussian and Exponential Mechanisms \cite{Dwork-algorithimc}, where the sensitivity determines how much randomization is needed to to hide any individual amongst the population with desired privacy.
Unfortunately, in many practical optimization problems, the end-to-end sensitivity is intractable or can only be loosely bounded. 

To this end, DP-SGD proposes an alternative solution by assuming a more powerful adversary.
In most private (centralized) learning applications, the standard black-box adversary can only observe the final model revealed.
DP-SGD, on the other hand, assumes an adversary who can observe the intermediate updates during training. For convenience, we will call such an adversary a white-box adversary.
Provided such an empowered adversary, DP-SGD clips the gradient evaluated by each individual sample and adds random noises to the updates in each iteration.
Clipping guarantees that the sensitivity is bounded within each iteration.
The total privacy loss is then upper bounded by a composition of the leakage from all iterations. 

For convex optimization with Lipschitz continuity, where the norms of gradients are uniformly bounded by some given constant, DP-SGD is known to produce an asymptotically tight privacy-utility tradeoff \cite{FOCS2014}.
However, it is, in general, impractical to assume Lipschitz continuity in tasks such as deep learning.
Either asymptotically or non-asymptotically, the study of practical implementations of DP-SGD with more realistic and specific assumptions remains very active \cite{amplification_iteration,imagesetDP, clipSong2020,DanICLR2020} and demanding. 
Much research effort has been dedicated to tackling the following two fundamental questions. 
First, {\em provided that we do not need to publish the intermediate computation results, how conservative is the privacy claim offered by DP-SGD?}  
Second, during practical implementation, {\em how to properly select the training model and hyper-parameters?} 

Regarding the first question, many prior works \cite{inference_leakage, lowerboundSP2021, membership2018} tried to empirically simulate what the adversary can infer from models trained by DP-SGD.
In particular, \cite{lowerboundSP2021} examined the respective power of “black-box” and “white-box” adversaries, and suggested that a substantial gap between the DP-SGD privacy bound and the actual privacy guarantee  may exist. 
Unfortunately, beyond DP-SGD, there are few known ways to produce, let along improve, rigorous DP analysis for a training process.
Most existing analyses need to assume either access to additional public data \cite{amplification_iteration, pate1, pate2}, or strongly-convex loss functions to enable objective perturbation \cite{objective_perturb}.
Thus, for general applications, we still have to adopt the conservative DP-SGD analysis for the worst-case DP guarantees.

The second question is of particular interest to practitioners. 
The implementation of DP-SGD is tricky as the performance of DP-SGD is highly sensitive to the selection of the training model and hyper-parameters. 
The lack of theoretical analysis on gradient clipping makes it hard to find good parameters and optimize model architectures instructively, though many heuristic observations and optimizations on these choices are reported. 
\cite{tempered2021, DanICLR2020,papernotfit2019} showed how to find proper model architectures to balance learning capability and utility loss when the dataset is given.
\footnote{Most prior works report the best model and parameter selection by grid searching, where the private data is reused multiple times and the selection of parameters itself is actually sensitive. 
This additional privacy leakage, partially determined by the prior knowledge on the training data is in general very hard to quantify, though in practice it might be small \cite{DanICLR2020,privateselection}.}
Recent work \cite{adptive_clipping} also demonstrated empirical improvements through selecting the clipping threshold adaptively. 
However, even with these efforts, there is still a long way to go if we want to practically train large neural networks with rigorous and usable privacy guarantees.
The biggest bottlenecks include
\renewcommand\labelitemi{{\boldmath$\cdot$}}
\begin{itemize}
      \item the huge model dimension, which may be even larger than the size of the training dataset, while the magnitude of noise required for gradient perturbation is proportional to the square root of model dimension, and 
      \item the long convergence time, which implies a massive composition of privacy loss and also forces DP-SGD to add formidable noise resulting in intolerable accuracy loss.
\end{itemize}

To this end, Tramer and Boneh in \cite{DanICLR2020} argued that within the current framework, for most medium datasets ($<$ 500K datapoints) such as CIFAR10/100 and MNIST, the utility loss caused by DP-SGD will offset the powerful learning capacity offered by deep models.
Therefore, simple linear models usually outperform the modern Convolutional Neural Network (CNN) for these datasets. 
How to privately train a model while still being able to enjoy the state-of-the-art success in modern machine learning is one of the key problems in the application of DP. 

  \vspace{-0.05 in}
\subsection{Our Strategy and Results}
  \vspace{-0.05 in}
\noindent In this paper, we set out to provide a systematic study of DP-SGD from both theoretical and empirical perspectives to understand the two important but artificial operations: (1) iterative gradient perturbation and (2) gradient clipping.
We propose a generic technique, {\em ModelMix}, to significantly sharpen the utility-privacy tradeoff.
Our theoretical analysis also provides instruction on how to select the clipping parameter and quantify privacy amplification from other practical randomness.

We will stick to the worst-case DP guarantee without any relaxation, but view the private iterative optimization process from a different angle. 
In most practical deep learning tasks, with proper use of randomness, we will have a good chance of finding some reasonable (local) minimum via SGD regardless of the initialization (starting point) and the subsampling \cite{ge2015escaping}. 
In particular, for convex optimization, we are guaranteed to approach the global optimum with a proper step size. 
In other words, {\em there are an infinite number of potential training trajectories \footnote{We will use {\em training trajectory} in the following to represent the sequences of intermediate updates produced by SGD.} pointing to some (local) minimum of good generalization, and we are free to use any one of them to find a good model.} 
Thus, even without DP perturbation, the training trajectory has potential entropy if we are allowed to do random selection. 

From this standpoint, a slow convergence rate when training a large model is not always bad news for privacy. 
This might seem counter-intuitive. 
But, in general, slow convergence means that the intermediate updates wander around a relatively large domain for a longer time before entering a satisfactory neighborhood of (global/local) optimum. Training a larger model may produce a more fuzzy and complicated convergence process, which could compensate the larger privacy loss composition caused in DP-SGD. 
We have to stress that our ultimate goal is to privately publish a good model, while DP-SGD with exposed updates is merely a tool to find a trajectory with analyzable privacy leakage. 
The above observation inspires a way to find a better DP guarantee even under the conservative “whitebox” adversary model: \textit{can we utilize the potential entropy of the training trajectory while still bounding the sensitivity to produce rigorous DP guarantees?}

To be specific, different from standard DP-SGD which randomizes a particular trajectory with noise, we aim to privately construct an {\em envelope} of {\em training trajectories}, spanned by the many trajectories converging to some (global/local) minimum, and randomly generate one trajectory to amplify privacy. 
To achieve this, we must carefully consider the tradeoff between (1) controlling the worst-case sensitivity in the trajectory generalization and (2) the learning bias resultant from this approach. 
We summarize our contributions as follows. 
\begin{enumerate}[label=(\alph*), leftmargin=*]
    \item We present a generic optimization framework, called {\em ModelMix}, which iteratively builds an envelope of training trajectories through post-processing historical updates, and randomly aggregates those model states before applying gradient descent. 
    We provide rigorous convergence and privacy analysis for {\em ModelMix}, which enables us to quantify  $(\epsilon, \delta)$-DP budget of our protocol. 
    The refined privacy analysis framework proposed can also be used to capture the privacy amplification of a large class of {\em training-purpose-oriented} operations commonly used in deep learning.
    This class of operations include data augmentation \cite{dataaugment2019} and stochastic gradient Langevin dynamics (SGLD) \cite{SGLD2018}, which cannot produce reasonable worst-case DP guarantees by themselves. 
    
    \item We study the influence of gradient clipping in private optimization and present the first generic convergence rate analysis of clipped DP-SGD in deep learning. 
    To our best knowledge, this is the first analysis of non-convex optimization via clipped DP-SGD with only mild assumptions on the concentration of stochastic gradient. 
    We show that the key factor in clipped DP-SGD is the {\em sampling noise} \footnote{The noise corresponds to using a minibatch of samples to estimate the true full-batch gradient.} of the stochastic gradient. 
    We then demonstrate why implementation of DP-SGD by clipping individual sample gradients can be unstable and sensitive to the selection of hyper-parameters. Those analyses can be used to instruct how to select hyper-parameters and improve network architecture in deep learning with DP-SGD.   
    
    \item ModelMix is a fundamental improvement to DP-SGD, which can be applied to almost all applications together with other advances in DP-SGD, such as low-rank or low-dimensional gradient embedding \cite{embedyu2021,yu_public2020} and fine-tuning based transfer learning \cite{DanICLR2020,  CVPR2021} (if additional public data is provided).  
    In our experiments, we focus on computer vision tasks, a canonical domain for private deep learning. 
    We evaluate our methods on CIFAR-10, FMNIST and SVHN datasets using various neural network models and compare with the state-of-the-art results.
    Our approach improves the privacy/utility tradeoff significantly. For example, provided a privacy budget $(\epsilon=8,\delta=10^{-5})$, we are able to train Resnet20 on CIFAR10 with accuracy $70.4\%$ compared to $56.1\%$ when applying regular DP-SGD. 
    As for private transfer learning on CIFAR10, we can improve the $(\epsilon=2,\delta=10^{-5})$-DP guarantee in \cite{DanICLR2020} to $(\epsilon=0.64,\delta=10^{-5})$ producing the same $92.7\%$ accuracy. 

\end{enumerate}

The remainder of this paper is organized as follows. 
In Section \ref{sec:pre}, we introduce background on statistical learning, differential privacy and DP-SGD. 
In Section \ref{sec:modelmix}, we formally present the ModelMix framework, whose utility in both convex and non-convex optimizations is studied in Theorem \ref{thm:convergence} and Theorem \ref{thm:convergence_MM},  respectively. 
In Section \ref{sec: privacy}, we show how to efficiently compute the amplified $(\epsilon, \delta)$ DP security parameters in Theorem \ref{thm:RDP} and a non-asymptotic amplification analysis is given in Theorem \ref{thm:amplification}. 
Further experiments with detailed comparisons to the state-of-the-art works are included in Section \ref{sec:exp}.
Finally, we conclude and discuss future work in Section \ref{sec:conclusion}. 

  \vspace{-0.05 in}
\subsection{Related Works}
  \vspace{-0.05 in}
\noindent \textbf{Theoretical (Clipped) DP-SGD Analysis}: When DP-SGD was first proposed \cite{dpsgd}, and in most theoretical studies afterwards \cite{FOCS2014,Wang2017}, the objective loss function is assumed to be $L$-Lipschitz continuous, where the $L_2$ norm of the gradient is uniformly bounded by $L$.
This enables a straightforward privatization on SGD by simply perturbing the gradients. 
In particular, for convex optimization on a dataset of $n$ samples, a training loss of $\Theta({\sqrt{d\log(1/\delta)}}/{(\epsilon n)})$ is known to be tight under an $(\epsilon, \delta)$ DP guarantee \cite{FOCS2014}. 

However, it is hard to get a (tight) Lipschitz bound for general learning tasks. 
A practical version of DP-SGD was then presented in \cite{CCS2016}, where the Lipschitz assumption is replaced by gradient clipping to ensure bounded sensitivity.
This causes a disparity between the practice and the theory as classic results \cite{FOCS2014,Wang2017} assuming bounded gradients cannot be directly generalized to clipped DP-SGD. 
Some existing works tried to narrow this gap by providing new analysis.
\cite{clipZhang2019} presented a convergence analysis of smooth optimization with clipped SGD when the sampling noise in stochastic gradient is bounded. 
But \cite{clipZhang2019} requires the clipping threshold $c$ to be $\Omega(T)$ where $T$ is the total number of iterations.
This could be a strong requirement, as in practice, the iteration number $T$ can be much larger than the constant clipping threshold $c$ selected.
\cite{clipChen2020} relaxed the requirement with an assumption that the sampling noise is symmetric. \cite{clipSong2020} studied the special case where clipped DP-SGD is applied to generalized linear functions. 
In this paper, we give the first generic analysis of clipped DP-SGD with only mild assumptions on the concentration property of stochastic gradients. 

\noindent \textbf{Assistance with Additional Public Data}: 
When additional unrestricted (unlabeled) public data is available, an alternative model-agnostic approach is {\em Private Aggregation of Teacher Ensembles} (PATE) \cite{pate1, pate2,knn}. 
PATE builds a teacher-student framework, where private data is first split into multiple (usually hundreds) disjoint sets, and a teacher model is trained over each set separately. 
Then, one can apply those teacher models to privately label public data via a private majority voting.
Those privately labeled samples are then used to train a student model, as a postprocessing of labeled samples. 
Another line of works considers improving the noise added in DP-SGD with public data. 
For example, in private transfer learning, we can first pretrain a large model with public data and then apply DP-SGD with private data to fine-tune a small fraction of the model parameters \cite{DanICLR2020,CVPR2021}.
However, both PATE and private transfer learning have to assume a large amount of public data. Another idea considers the projection of the private gradient into a low-rank/dimensional subspace, approximated by public samples, to reduce the magnitude of noise \cite{embedyu2021,yu_public2020,zhou2020bypassing,asi2021private}. 
When the public samples are limited, DP-SGD with low-rank gradient representation usually outperforms the former methods. 

Except for PATE, our methods can be, in general, used to further enhance those state-of-the-art DP-SGD improvements with public data. 
For example, using 2K ImageNet public samples, the low-rank embedding method in \cite{yu_public2020} can train Resnet20 on CIFAR10 with $79.1\%$ accuracy at a cost of $(\epsilon=111.2,\delta=10^{-5})$-DP; while ModelMix can improve the DP guarantee to $(\epsilon=6.1,\delta=10^{-5})$-DP with the same accuracy as shown in Section \ref{sec:pub_data}.


\section{Preliminaries}
  \vspace{-0.05 in}
\label{sec:pre}
\noindent \textbf{Empirical Risk Minimization:} In statistical learning, the model to be trained is commonly represented by a parameterized function $f(w,x): (\mathcal{W},\mathcal{X}) \to \mathbb{R}$, mapping feature $x$ from input domain $\mathcal{X}$ into an output (prediction/ classification) domain. In the following, we will always use $d$ to represent the dimensionality of the parameter $w$, i.e., $w \in \mathbb{R}^d$. 
For example, one may consider $f(w,x)$ as a neural network with a sequence of linear layers connected by non-linear activation layers, and $w$ represents the weights to be trained. Given a set $\mathcal{D}$ of $n$ samples $\{(x_i, y_i), i=1,2,...,n\}$, we define the problem of Empirical Risk Minimization (ERM) for some loss function $l(\cdot,\cdot)$ as follows, 
\begin{equation}\small
\label{obg_loss_1}
   \min_{w} F(w) = \min_{w} \frac{1}{n} \cdot \sum_{i=1}^n l(f(w,x_i),y_i). 
\end{equation}
For convenience, we simply use $f(w, x_i, y_i)$ to denote the objective loss function $l(f(w,x_i),y_i)$ in the rest of the paper. 
Below, we formally introduce the definitions of \textit{Lipschitz continuity, smoothness} and \textit{convexity}, which are commonly used in optimization research. 

\begin{definition}[Lipschitz Continuity]
A function $g$ is $L$-Lipschitz if for all $w, w'\in\mathcal{W}$, $|g(w)-g(w')|\leq L\|w-w'\|_2$.
\end{definition}
\begin{definition}[Smoothness] A function $g$ is $\beta$-smooth on $\mathcal{W}$ if for all $w, w'\in \mathcal{W}$, $g(w')\leq g(w)+\langle \nabla g(w), w'-w \rangle+\frac{\beta}{2}\|w'-w\|_2^2$.
\end{definition}
\begin{definition}[Convexity] A function $g$ is convex on $\mathcal{W}$ if for all $w, w'\in \mathcal{W}$ and $t \in (0,1)$, $f(t w + (1-t) w')\leq t g(w)+ (1-t) g(w')$.
\end{definition}
\noindent
In the following, we will simply use $\|\cdot\|$ to denote the $l_2$ norm unless specified otherwise.

\noindent \textbf{Differential Privacy (DP):} We first formally define $(\epsilon, \delta)$-DP and $(\alpha, \epsilon)$-Rényi DP as follows.  \begin{definition}[Differential Privacy]
\label{def:DP}
Given a data universe $\mathcal{X}^*$, we say that two datasets $\mathcal{D},\mathcal{D}'\subseteq \mathcal{X}^*$ are neighbors, denoted as $\mathcal{D} \sim \mathcal{D}'$, if $\mathcal{D} = \mathcal{D}' \cup s$ or $\mathcal{D}' = \mathcal{D} \cup s$ for some additional datapoint $s$. A randomized algorithm $\mathcal{A}$ is said to be $(\epsilon,\delta)$-differentially private (DP) if for any pair of neighboring datasets $\mathcal{D},\mathcal{D}'$ and any event $S$ in the output space of $\mathcal{A}$, it holds that
	\begin{equation*}
	\mathbb{P}(\mathcal{A}(\mathcal{D})\in S)\leq e^{\epsilon}\cdot \mathbb{P}(\mathcal{A}(\mathcal{D}')\in S)+\delta.
	\end{equation*}
\end{definition}

\begin{definition}[Rényi Differential Privacy \cite{RenyiDP}]
\label{def:RDP}
A randomized algorithm $\mathcal{A}$ satisfies $(\alpha, \epsilon)$-Rényi Differential Privacy (RDP), $\alpha >1$, if for any pair of neighboring datasets $\mathcal{D} \sim \mathcal{D}'$, 
 \[ \small
 \epsilon \geq \mathsf{D}_{\alpha}(\mathcal{M}(\mathcal{D}) \| \mathcal{M}(\mathcal{D}') ).   
 \]
 Here, 
 \begin{equation}\small
 \label{Rényi-diver}
 \mathsf{D}_{\alpha}(\mathsf{P} \| \mathsf{Q}) = \frac{1}{\alpha-1} \log \int \mathsf{q}(o) (\frac{\mathsf{p}(o)}{\mathsf{q}(o)})^{\alpha} ~d o,
 \end{equation}
represents $\alpha$-Rényi Divergence between two distributions $\mathsf{P}$ and $\mathsf{Q}$ whose density functions are $\mathsf{p}$ and $\mathsf{q}$, respectively. 
\end{definition}

In Definition \ref{def:DP} and \ref{def:RDP}, if two neighboring datasets $\mathcal{D}$ and $\mathcal{D}'$ are defined in a form that $\mathcal{D}$ can be obtained by arbitrarily replacing an datapoint in $\mathcal{D}'$, then 
they become the definitions of bounded DP \cite{dwork2006a, dwork2006b} and RDP, respectivaly. In this paper, we adopt the unbounded DP version to match existing DP deep learning works \cite{CCS2016, DanICLR2020, yu_public2020} with a fair comparison. 

In practice, to achieve meaningful privacy guarantees, $\epsilon$ is usually selected as some small one-digit constant and $\delta$ is asymptotically $O(1/|\mathcal{D}|) = O(1/n)$. 
To randomize an algorithm, the most common approaches in DP are Gaussian or Laplace Mechanisms \cite{Dwork-algorithimc}, where a Gaussian or Laplace noise proportional to the sensitivity is added to perturb the algorithm's output. 
In many applications, including the DP-SGD analysis, we need to quantify the cumulative privacy loss across sequential queries of some differentially private mechanism on one dataset. The following theorem provides an upper bound on the overall privacy leakage.

\begin{thm}[Advanced Composition \cite{composition}]
\label{thm:composition}
For any $\epsilon>0$ and $\delta \in (0,1)$, the class of $(\epsilon, \delta)$-differentially private mechanisms satisfies $(\tilde{\epsilon}, T\delta + \tilde{\delta})$-differential privacy under $T$-fold adaptive composition for any $\tilde{\epsilon}$ and $\tilde{\delta}$ such that
\[
\tilde{\epsilon} = \sqrt{2T\log(1/\tilde{\delta})}\cdot\epsilon + T\epsilon(e^{\epsilon}-1).
\]
\end{thm}

\begin{thm}[Advanced Composition via RDP \cite{RenyiDP}]
\label{thm:composition-RDP}
For any $\alpha>1$ and $\epsilon>0$, the class of $(\alpha, \epsilon)$-RDP mechanisms satisfies $(\tilde{\epsilon}, \tilde{\delta})$-differential privacy under $T$-fold adaptive composition for any $\tilde{\epsilon}$ and $\tilde{\delta}$ such that
\[
\tilde{\epsilon} = T\epsilon - \log(\tilde{\delta})/(\alpha-1).
\]
\end{thm}

Theorem \ref{thm:composition} provides a good characterization on how the privacy loss increases with composition.
For small $(\epsilon,\delta)$, we still have an $\tilde{O}(\sqrt{T}\epsilon, T\delta)$ DP guarantee after a $T$ composition.
In practice using RDP, Theorem \ref{thm:composition-RDP} usually produces tighter constants in the privacy bound.

\noindent \textbf{DP-SGD:} (Stochastic) Gradient Descent ((S)GD) is a very popular approach to optimize a function. 
Suppose we try to solve the ERM problem and minimize some function $F(w) = \frac{1}{n} \sum_{i=1}^n f(w,x_i, y_i)$. 
SGD can be described as the following iterative protocol. 
In the $(k+1)$-th iteration, we apply Poisson sampling, i.e., each datapoint is i.i.d. sampled by a constant rate $q$, and a minibatch of $B_{k}$ samples is produced from the dataset $\mathcal{D}$, denoted as $S_k$.
We calculate the stochastic gradient as
\begin{equation}
    \label{stochastic gradient}
    G_{k} \gets \sum_{(x_i, y_i) \in S_k} \nabla f(w_{k},x_i,y_i). 
\end{equation}
Then, a gradient descent update is applied using
\begin{equation}
    \label{SGD}
    w_{k+1} = w_{k} - \eta \cdot G_k,
\end{equation}
for some stepsize $\eta$. 
In particular, if the minibatch is selected to be the full batch, i.e., $S_k = \mathcal{D}$, then Equation (\ref{SGD}) becomes the standard gradient descent procedure.

We make the following assumption regarding the sampling noise $\| \nabla f(w,x,y) - \nabla F(w)\|$ when the minibatch size equals $1$.
This assumption, which will be shown to be necessary in Example \ref{ex:non-convergence},  will be used in Theorem \ref{thm:convergence_MM} when we derive the concrete convergence rate for clipped SGD.
\begin{assumption}[Stochastic Gradient of Sub-exponential Tail] 
\label{assp:subexp}
There exists some constant $\kappa>0$ such that for any $w$, if we randomly select a datapoint $(x,y)$ from $\mathcal{D}$, then 
\[\Pr(\| \nabla f(w,x,y) - \nabla F(w)\| \geq t) \leq e^{-t/\kappa}.\] 
\end{assumption} 

In Assumption \ref{assp:subexp}, a larger $\kappa$ implies stronger concentration, i.e., a faster decaying tail of the stochastic gradient. 
The modification from GD/SGD to its corresponding DP version is straightforward. 
When the loss function $f$ is assumed to be $L$-Lipschitz \cite{dpsgd, FOCS2014}, i.e., $\|\nabla f(w, x_i,y_i)\| \leq L$ for any $w$, the worst-case sensitivity in Equation (\ref{SGD}) is bounded by $\eta L$ in each iteration.
One can derive a tighter bound \cite{FOCS2014} using existing results on the privacy amplification from sampling \cite{subsampling, mironov2019r}.
Thus, SGD can be made private via iterative perturbation by replacing Equation (\ref{SGD}) with the following:
\begin{equation}
   \label{dpsgd_without_clipping}
    w_{k+1} = w_{k} - \eta\cdot (G_k + \Delta_{k+1}),
\end{equation}
where $\Delta_{k}$ is the noise for the $k$-th iteration.
For example, if we want to use the Gaussian Mechanism to ensure $(\epsilon, \delta)$-DP when running $T$ iterations, then $\Delta_{k+1}$ can be selected to be i.i.d. generated from 
\[
\Delta_{k+1} \leftarrow \mathcal{N}(\bm{0}, O({L^2T\log(1/\delta)\over \epsilon^2}) \cdot \bm{I}_d).
\]
Here, $\bm{I}_d$ represents the $d \times d$ identity matrix. 

However, when we do not have the Lipschitz assumption, an alternative is to force a limited sensitivity through gradient clipping. 
Following the same notations as before, we describe DP-SGD with per-sample gradient clipping \cite{CCS2016} as follows,
\begin{equation}
\begin{aligned}
    \label{dpsgd_with_cliping}
    & G_k \gets \sum_{(x_i, y_i) \in S_k} \mathsf{CP}\big(\nabla f(w_k,x_i,y_i), c\big);\\
    & w_{k+1} = w_{k} - \eta\cdot (G_k + \Delta_{k+1}).
\end{aligned}  
\end{equation}  
Here, $\mathsf{CP}(\cdot, c)$ represents a clipping function of threshold $c$,
$$\mathsf{CP}(\nabla f(w,x,y), c) = \nabla f(w,x,y) \cdot \min\{1, {c \over \|\nabla f(w,x,y) \|} \}.$$
With clipping, the $l_2$ norm of each per-sample gradient is bounded by $c$. 
Thus, the clipping threshold $c$ virtually plays the role of the Lipschitz constant $L$ in clipped SGD for privacy analysis.

\section{ModelMix}
\label{sec:modelmix}
\noindent In this section, we formally introduce ModelMix, and explain how it sharpens the utility-privacy trade-off in DP-SGD. 

\subsection{Intuition}
\noindent We begin with the following observation. 
Suppose we run SGD twice on a least square regression $F(w) = 1/n \cdot \sum_{i=1}^n \|\langle w, x_i\rangle -y_i\|^2$ for $T$ iterations and obtain two training trajectories $\bm{w}=(w_1,w_2,...,w_T)$ and $\bm{w}'=(w'_1,w'_2,...,w'_T)$. 
Suppose both $w_T$ and $w'_T$ are $\sigma$-close to the optimum $w^* = \arg \min_{w} F(w)$, i.e.,
\[
\|w_T-w^*\|\leq \sigma \text{~and~} \|w'_T-w^*\|\leq \sigma.
\]
Due to the linearity of gradients in least square regression, if we mix $\bm{w}$ and $\bm{w}'$ to get $w''_k = \alpha w_{k} + (1-\alpha) w'_{k}$ ($k=1,2,...,T$) for some weight $\alpha \in  (0,1)$, then we produce a new SGD trajectory $\bm{w}''$ where $w''_T$ is also $\sigma$-close to $w^*$. 

This simple example gives us two inspirations. 
First, as mentioned earlier, with different randomness in initialization and subsampling, the training trajectory to find an optimum is not unique.
Even in DP-SGD where we need to virtually publish the trajectory for analytical purpose, we are not restricted to expose a particular one. 
Second, and more important, this means that we have more freedom to randomize the SGD process.
In the $k$-th iteration, instead of following the regular SGD rule in Equation (\ref{SGD}) where we simply start from the previously updated state $w_{k-1}$, we can randomly mix $w_{k-1}$ with some other $w'_{k-1}$ from another reasonable training trajectory, and then {\em move to a new trajectory to proceed}. 

However, the above idea to randomly mix the trajectories cannot be directly implemented as we do not have any prior knowledge on what good trajectories look like.
Any training trajectory generated by the private dataset is sensitive and potentially creates privacy leakage. 
Thus, we need to be careful about how we generate the needed envelope.
Recall that we use \textit{envelope} to describe the space spanned by the mixtures of training trajectories. Provided the property that DP is immune to post-processing, we consider approximating the trajectory envelope using intermediate states already published. This allows us to privately construct the envelope and advance the optimization, simultaneously.


\subsection{Algorithm and Observations}
\noindent We start with a straw-man solution where we virtually run DP-SGD to alternately train two models in turns. 
We initialize two states $\tilde{w}^1_0$ and $\tilde{w}^2_0$ with respect to (w.r.t) the parameterized function we aim to optimize.

At any odd iteration, i.e., iteration $2k + 1$ ($k \ge 0$), we randomly generate $\bm{\alpha}_{2k+1} \in (0,1)^d$ whose coordinate is i.i.d. uniform in $(0,1)$.
We then update the state of the first model $\tilde{w}^1$ as
{\small{
\begin{equation}
\label{distributed-manner-1}
    \bar{w}^1_{k} \gets
    \bm{\alpha}_{2k+1} \circ \bar{w}^1_{k-1}  + (\bm{1}-\bm{\alpha}_{2k + 1}) \circ \bar{w}^2_{k-1} - \eta \big( \nabla F(\tilde{w}^1_{k-1})+\Delta_{2k+1} \big),
\end{equation}
}}
where $\circ$ represents the Hadamard product and $\Delta$ represents the noise added.
Essentially, we mix the two states $\bar{w}^1_{k-1}$ and $\bar{w}^2_{k-1}$ coordinate-wise. 
Similarly, at any even iteration, i.e., iteration $2(k+1)$, we randomly generate some $\bm{\alpha}_{2k+2}$ and update the state of the second model $\tilde{w}^2$ as
{\small{
\begin{equation}
\label{distributed-manner-2}
    \bar{w}^2_{k} \gets \bm{\alpha}_{2k+2} \circ \bar{w}^1_{k}  + (\bm{1}-\bm{\alpha}_{2k + 2}) \circ \bar{w}^2_{k-1} - \eta \big( \nabla F(\tilde{w}^2_{k-1})+\Delta_{2k+2} \big).
\end{equation}
}}
We take turns mixing and updating two training trajectories privately using (\ref{distributed-manner-1}) and (\ref{distributed-manner-2}).
At the high level, this is similar to a distributed GD, where two agents collaboratively train the model. 

However, in a centralized scenario, it is unnecessary to artificially create two models and  train them in a distributed manner, meaning that the efficiency of the above mixture method is not optimal. 
In the following, we propose an improved method where in the $k$-{th} iteration, we simply post-process on $w_{k-1}$ and $w_{k-2}$, the updates already privately generated from the last two iterations, instead of two virtual models to approximate an envelope of training trajectory. 
The formal description of ModelMix is shown in Algorithm \ref{alg: ModelMix}. We provide an illustration of how Algorithm \ref{alg: ModelMix} works in Fig. \ref{pic:mix}.

\begin{algorithm}
\caption{Differentially Private Stochastic Gradient Descent with ModelMix}
\begin{algorithmic}[1]
\STATE \textbf{Input:} Objective function $F(w)= {1\over n} \cdot \sum_{i=1}^n f(w, x_i, y_i)$, dataset $\mathcal{D}=\{(x_i,y_i), i=1,2,...,n\}$, sampling rate $q$, step size $\eta$, clipping threshold $c$, total number of iterations $T$, mixing thresholds $\tau_{[1:T]}$, initialized model states $w_{-1}, w_{0} \in \mathbb{R}^d$ and noise sequence $\Delta_{[1:T]}$.
\FOR{$k=1,2,...,T$}
   \STATE Through i.i.d. sampling with a constant rate $q$, produce a minibatch $S_{k}$ of $B_k$ samples from $\mathcal{D}$ and calculate the stochastic gradient \[G_{k-1} =  \sum_{(x_i,y_i) \in S_k} \mathsf{CP}\big(\nabla f(w_{k-1},x_i,y_i), c \big).\] 
\FOR{$j=1,2,...,d$}
    \STATE $\bm{\alpha}_k(j)\leftarrow \mathcal{U}[0,1]$, a uniform distribution in [0,1].
    \IF{$|{w}_{k-1}(j) - {w}_{k-2}(j)| < \tau_k $}
        \STATE ${w}_{k-1}(j) \gets {w}_{k-1}(j)+ \text{sign}({w}_{k-1}(j)-{w}_{k-2}(j)) \cdot \tau_k/2$;
        \STATE ${w}_{k-2}(j) \gets {w}_{k-2}(j)- \text{sign}({w}_{k-1}(j)-{w}_{k-2}(j)) \cdot \tau_k/2$.
    \ENDIF 
    \STATE Update the weight as follows:
    \begin{equation}
    \label{main_update}
    \begin{aligned}
        w_{k}(j) 
         =&   
        \bm{\alpha}_k(j)\cdot{w}_{k-1}(j)  + (1-\bm{\alpha}_k(j))\cdot {w}_{k-2}(j) \\
        & -\eta\cdot (G_{k-1}(j)+\Delta_k(j)).
    \end{aligned}
    \end{equation} 
\ENDFOR   
\ENDFOR
\STATE \textbf{Output}: $w_{T}$. 

\end{algorithmic}
\label{alg: ModelMix}
\end{algorithm}

\begin{figure}
\includegraphics[scale = 0.45]{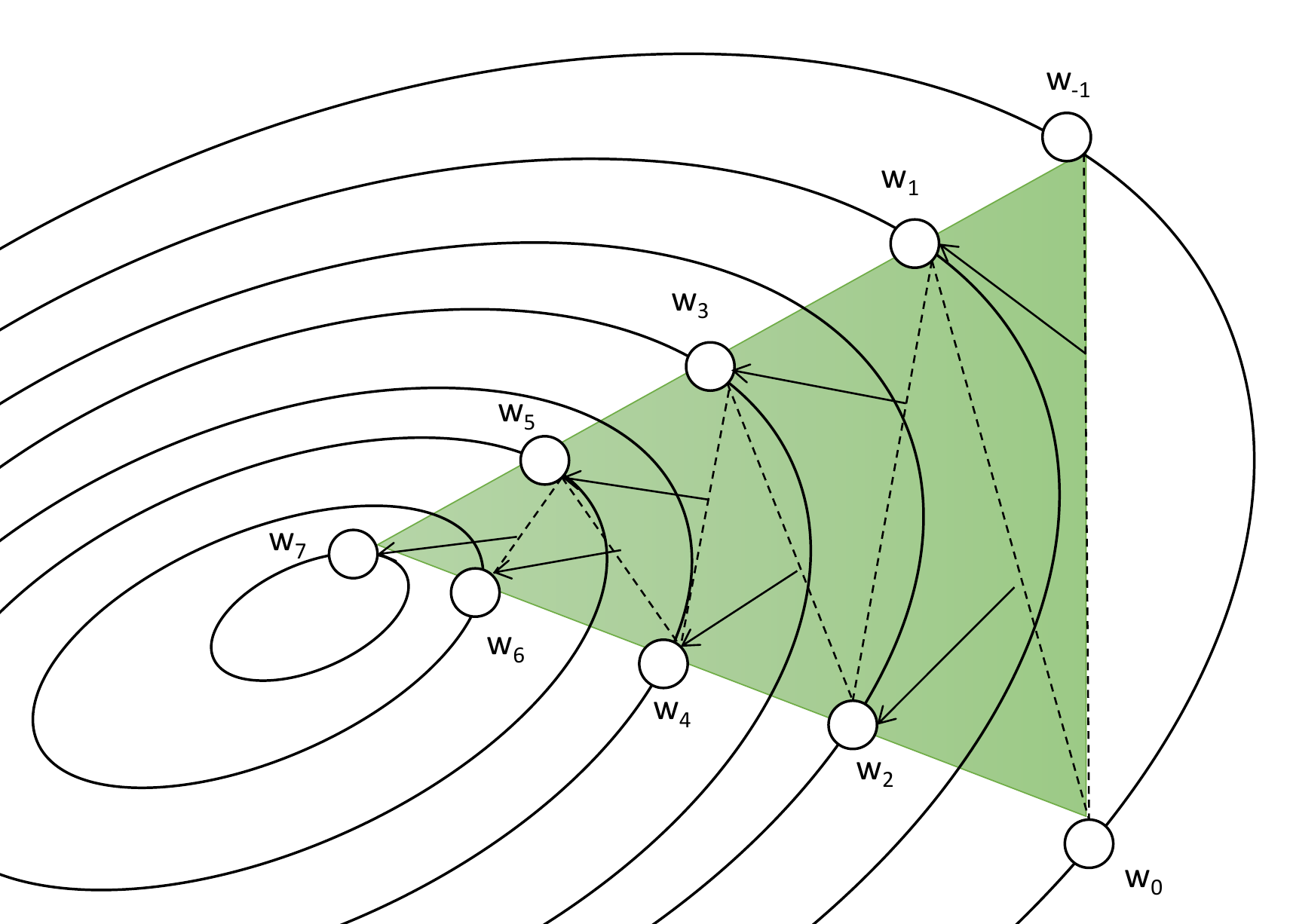}
\caption{An illustration of how ModelMix works.
Each node represents an intermediate state.
The dashed lines represent the mixing and the arrows represent gradient descent. The green area represents the approximated trajectory envelope.}
\label{pic:mix}
\end{figure}

From a more concise optimization standpoint, the mixed state in expectation is the average of the last two iterations' states $(w_{k-1}+w_{k-2})/2$.
Therefore, in the expectation of Equation (\ref{main_update}), we are essentially optimizing the original objective function $F(w)$ plus a proximal term in a form 
\begin{equation}
    \label{proximal}
    F(w) + \|w - w_{k-2} \|^2/{4},
\end{equation}
whose gradient at $w_{k-1}$ is $\nabla F(w_{k-1}) + (w_{k-1}-w_{k-2})/2$. 
{Thus, ModelMix can also be viewed as introducing a randomized proximal term into the original objective function $F(w)$. ModelMix operation $\bm{\alpha}_{k}w_{k-1}+ (\bm{1}-\bm{\alpha}_{k-1})w_{k-2}$ averages out the noise added in the previous iterations, which makes the convergence more stable. This enables us to apply a larger step size during training, as shown later in Fig.  {\ref{fig_modelmix_per}}}.

As will be shown later in Theorem \ref{thm:amplification}, the privacy amplification is determined by the distance between $w_{k-1}$ and $w_{k-2}$.
Therefore, for a privacy analysis purpose, in steps 5-8 of Algorithm \ref{alg: ModelMix}, we artificially ensure that the coordinate-wise distance between the states $w_{k-1}$ and $w_{k-2}$ is at least some parameter $\tau_k$. 
In practice, the gap between the two states $w_{k-1}$ and $w_{k-2}$ already exists even without these operations, meaning that ModelMix naturally enjoys a privacy amplification.
We only add those steps to enforce a worst-case lower bound on the coordinate-wise distance so that we can quantify the privacy amplification in a clean way.

{A natural question that follows is \textit{how do the parameters $\tau_{[{1:T]}}$ across $T$ iterations affect privacy and utility}?}
$\tau_k$ captures the distance between the two training trajectories, or in other words, the volume of the approximated envelope.
Ideally, we would like $\tau_{[{1:T]}}$ to match the size of the true envelope so that the impact of those artificial steps is minimized.
Empirically, a harder learning task often has a slower convergence rate, and thus also has an envelope of larger volume.
This allows us to select a larger $\tau_k$ and provide stronger privacy amplification. 
On the other hand, if the learning task is simple and we overestimate the size of the envelope with a large $\tau_k$, then the learning rate might be compromised.
This analysis is supported by Fig. \ref{fig_modelmix_per}, where we test the efficiency of ModelMix in two different tasks with various setups. 

In Fig. \ref{fig_modelmix_per} (a) and (b), we implement non-private SGD with or without ModelMix to train Resnet20 \cite{resnet} on CIFAR10 \footnote{https://www.cs.toronto.edu/~kriz/cifar.html} and SVHN \footnote{http://ufldl.stanford.edu/housenumbers/}, the two benchmark datasets. 
We set the sampling rate $q$ to be $0.05$ and run for 2,000 iterations.
Compared to CIFAR10, classification on SVHN is an easier task.
We fix the mixing threshold $\tau_k = \tau$ to be proportional to the step size $\eta$ in gradient descent (Equation (\ref{main_update})), specifically, $0.1\eta$ and $0.05\eta$. 

In both Fig. \ref{fig_modelmix_per} (a) and (b), different setups perform similarly in the initial stage (the first 400 iterations). 
However, in later epochs, as the training trajectories approach the (local) optimum, setups with larger $\tau$ suffer heavier losses. 
This phenomenon is clearer in Fig. \ref{fig_modelmix_per} (b) which trains on the easier SVHN task. 
This matches our previous analysis: the size of the trajectory envelope gets smaller when (1) the learning task is easier or (2) we approach the optimum. 
On the other hand, with proper selection of $\tau=0.05\eta$, in the non-private case ModelMix comes with almost no additional utility loss.

In Fig. \ref{fig_modelmix_per} (c) and (d), we implement the private case where we will see how ModelMix strengthens the robustness of DP-SGD. 
We clip the per-sample gradient norm down to $8$ and add the same amount of Gaussian noise to both DP-SGD cases, with or without ModelMix.
This ensures an $(\epsilon=8, \delta=10^{-5})$ guarantee for regular DP-SGD. 
{\em As shown later in Theorem \ref{thm:amplification}, ModelMix can achieve a much better DP guarantee under the same setup. 
But here, we focus on the performance of ModelMix given the same noise as regular DP-SGD to provide a clear comparison.}
With the virtual proximal term in Equation (\ref{proximal}), ModelMix allows us to use a larger stepsize resulting in a faster learning rate.
In  Fig. \ref{fig_modelmix_per} (c) and (d), we set the stepsize $\eta=0.2$ for ModelMix, and compare to the regular DP-SGD with $\eta=0.1$ (black line) and $\eta=0.2$ (red line). The larger stepsize $\eta=0.2$ worsens the performance of regular DP-SGD, whereas with proper selection of $\tau=0.05\eta$ (green line), ModelMix allows the application of larger stepsize and slightly out-performs regular DP-SGD. 
\begin{figure}[h]
  \centering
  \includegraphics[width=0.49\linewidth]{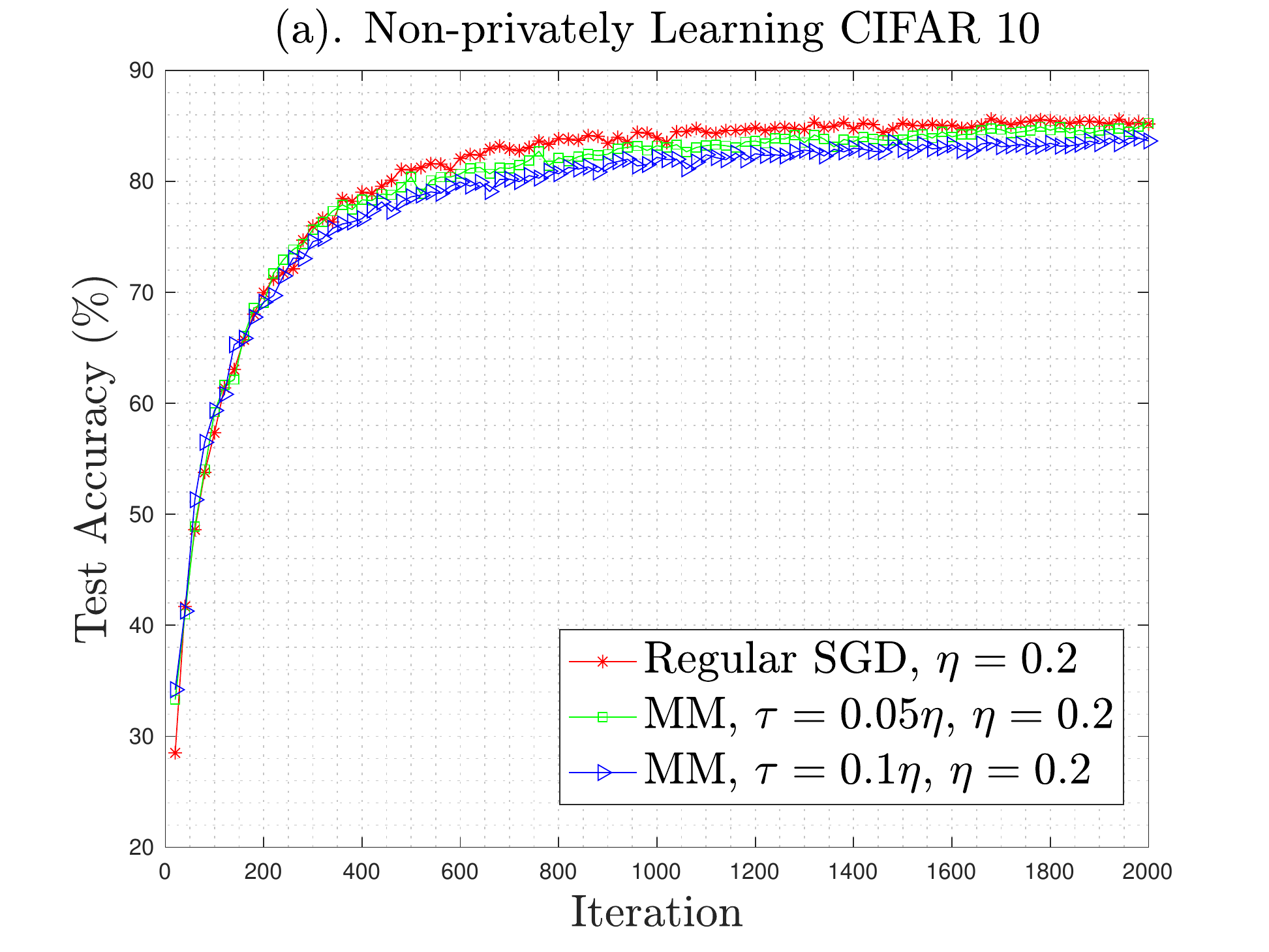}
  \includegraphics[width=0.49\linewidth]{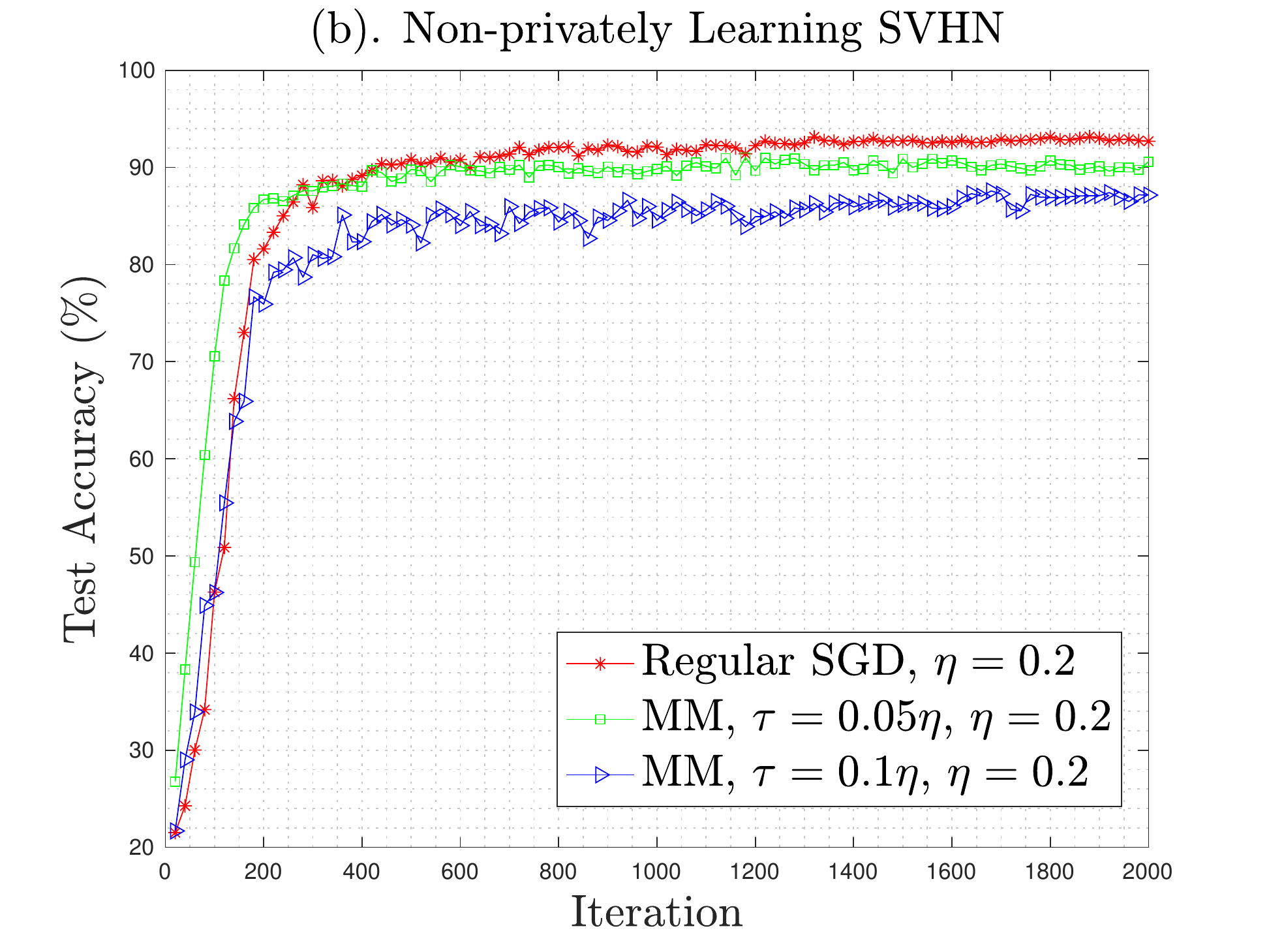}
   \includegraphics[width=0.49\linewidth]{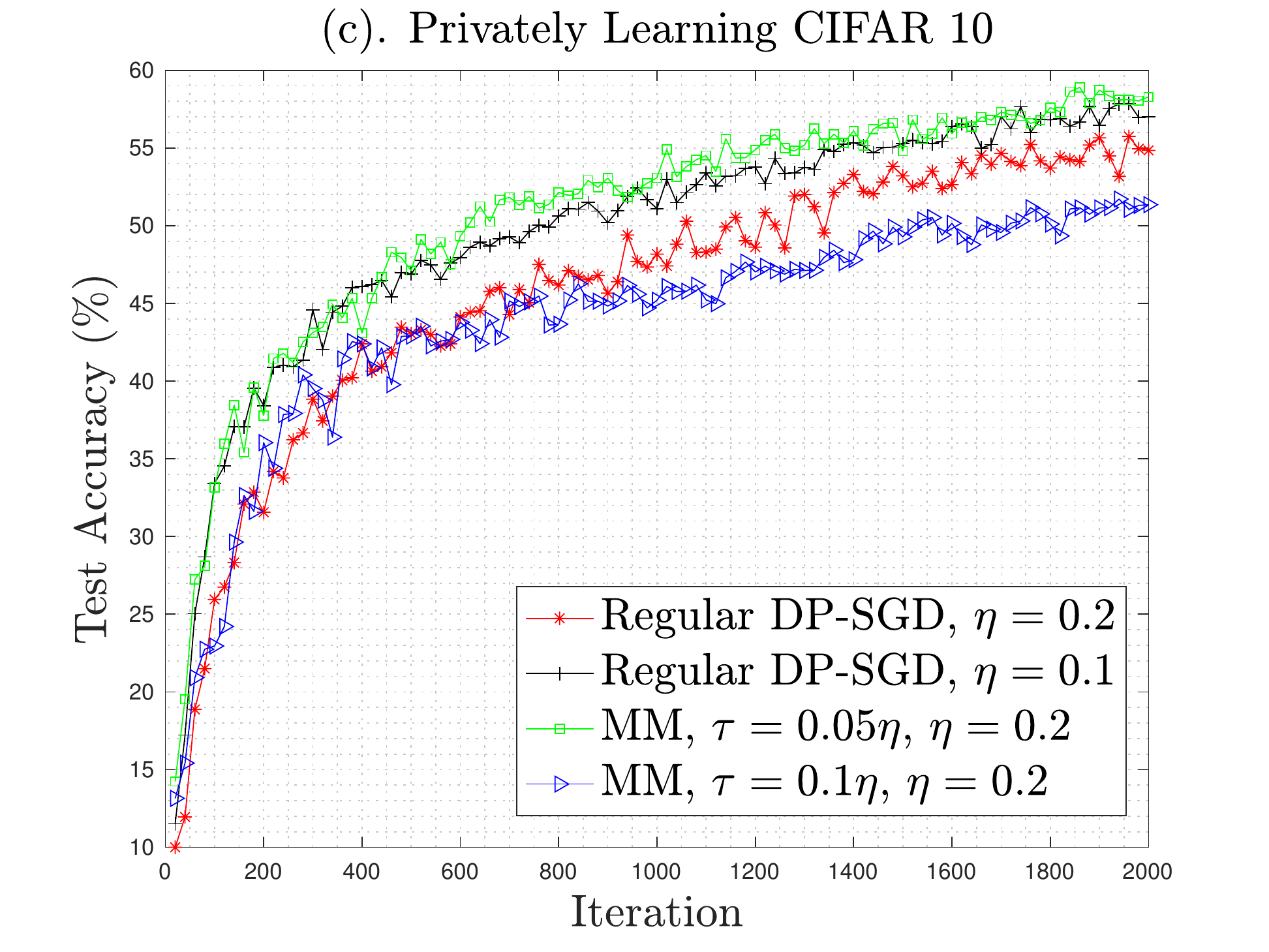}
  \includegraphics[width=0.49\linewidth]{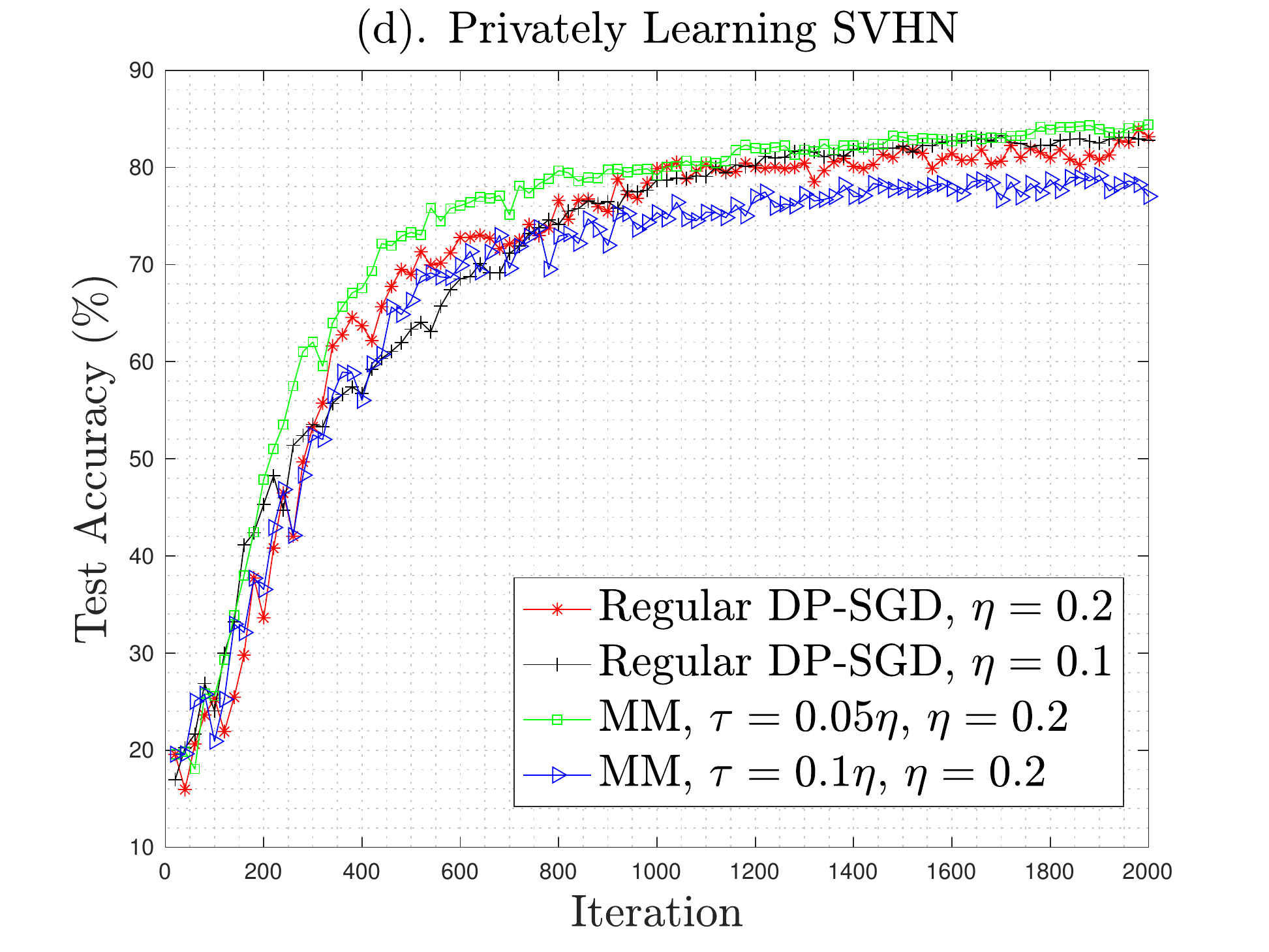} 
\caption{Performance comparison between DP-SGD with/without ModelMix (MM) provided with the same noise.}
  \vspace{-0.15 in}
\label{fig_modelmix_per}
\end{figure}

\subsection{Utility Guarantee in (non)-convex Optimization and Effect of Gradient Clipping}
\vspace{-0.1 in}
\noindent In this subsection, we will provide formal utility analysis of Algorithm \ref{alg: ModelMix}. 
As a warm-up, in Theorem \ref{thm:convergence}, we start with convex optimization with a Lipschitz and smooth assumption, and show Algorithm \ref{alg: ModelMix} enjoys an $1/\sqrt{T}$ convergence rate. 
Specifically, after $T$ iterations, if we set the average
\[
\bar{w} = \frac{\sum_{k=1}^T (w_{k-1}+w_{k-2})}{2T},
\]
where $w_{k}$ represents the weight in the $k$-{th} iteration, then the utility loss $F(\bar{w}) - F(w^*)$ is bounded by $O(1 / \sqrt{T})$. Based on Theorem \ref{thm:convergence}, we then move to the non-convex case without a Lipschitz assumption in Theorem \ref{thm:convergence_MM}, and show a generic convergence rate. 

\begin{thm}[Utility of ModelMix in Convex Optimization]
\label{thm:convergence}
Suppose $f(w,x,y)$ is $L$-Lipschitz and $\beta$-smooth convex.
We set the clipping threshold $c = L$. 
For any $\gamma>0$, if we set $\eta = \gamma /(nq\sqrt{T})$, then Algorithm \ref{alg: ModelMix} satisfies 
{\small{
\begin{equation}
\begin{aligned}
&
\mathbb{E} \big[ F(\bar{w}) - F(w^*) \big] \\
& \leq \frac{3\mathcal{W}^2_0+ \sum_{k=1}^Td \tau^2_{i}/12}{2\gamma \sqrt{T}} + \frac{\gamma(L^2/q^2+\mathbb{E}[\|\Delta\|^2]/(n^2q^2))}{\sqrt{T}} + \beta \cdot \\
& \big( \frac{12\mathcal{W}^2_0}{8T} + \frac{2\gamma \mathcal{W}_0(L+\mathbb{E}[\|\Delta\|]/n)}{q T^{3/2}} + \frac{11\gamma^2(L^2+\mathbb{E}[\|\Delta\|^2]/n^2)}{8Tq^2} \big) \\
& = O(\frac{\mathcal{W}^2_0+ \sum_{k=1}^Td\tau^2_{k}}{\gamma \sqrt{T}} + \frac{\gamma(L^2+\mathbb{E}[\|\Delta\|^2]/n^2)}{q^2\sqrt{T}})
\end{aligned}
\label{convex_convergence}
\end{equation}
}}\noindent
where  $\mathcal{W}_0 = \sup_{w}\|w-w^*\|$ denotes the initial divergence and $\mathbb{E}[\| \Delta\|^2] = \mathbb{E}[\| \Delta_{k}\|^2]$ denotes the variance of noise added in each iteration.  
\end{thm}

\begin{proof}
See Appendix \ref{app:pr_convergence}.
\end{proof}

From Equation (\ref{convex_convergence}), one can see that with ModelMix, Algorithm \ref{alg: ModelMix} still enjoys an $O({1}/{\sqrt{T}})$ convergence rate when $\tau_k = O(\eta) = O(1/\sqrt{T})$. 
Using Theorem \ref{thm:convergence}, we can upper bound the utility loss of Algorithm \ref{alg: ModelMix}. Since ModelMix can be seen as post-processing operations on DP-SGD, its privacy guarantee is at least as good as the privacy guarantee of DP-SGD.
If we set $\Delta_{k}$ to be Gaussian noise generated from $\mathcal{N}(0,O(q^2L^2T\log(1/\delta)/\epsilon^2) \cdot \bm{I}_{d} )$ \cite{CCS2016}, then by Theorem \ref{thm:convergence}, Algorithm \ref{alg: ModelMix} satisfies an $(\epsilon, \delta)$-DP guarantee and its utility loss is upper bounded by $O(\sqrt{d}\log(1/\delta)L/(n\epsilon))$.
This asymptotically matches the classic results in \cite{FOCS2014}. However, we will present a more fine-tuned analysis in Section \ref{sec: privacy} to show the randomness in ModelMix can significantly sharpen the composite privacy loss. 

In the following, we no longer assume the objective loss function to be either convex or Lipschitz continuous, and we try to understand the effect of gradient clipping. 
The following theorem gives a generic convergence analysis of ModelMix in non-convex optimization with clipped gradient. 

\begin{thm}[Utility of ModelMix in non-Convex Optimization with Gradient Clipping]
\label{thm:convergence_MM}
Suppose the objective loss function $F(w)$ is $\beta$-smooth and satisfies Assumption \ref{assp:subexp}, then there exists some constant $\psi > 0$ such that when the clipping threshold $c$ satisfies
\begin{equation}
\label{selection-of-c}
c \geq \max \{4\kappa\log(10), -\psi\kappa\log(\kappa)\log(\frac{\sqrt{d \log(1/\delta)}}{n\epsilon})\},
\end{equation}
then the convergence rate of Algorithm \ref{alg: ModelMix}, which applies per-sample gradient clipping up to $c$ and enjoys an $(\epsilon, \delta)$-DP guarantee, satisfies
{\small{
\begin{equation}
\begin{aligned}
\label{batch_main_1}
     &  ~~~~~\mathbb{E} \Big[ \frac{\sum_{k=0}^{T-1}  \min\big\{ 9/20 \cdot  \|\nabla F(w_k)\|^2, c/20 \cdot \| \nabla F(w_k)\| \big \}}{T}\Big]  \\
       & \leq (\frac{v}{2} +\frac{5}{2})\frac{c\sqrt{\mathcal{R}_{F}\frac{101}{12}\beta d \log(1/\delta)}}{n\epsilon} + \frac{28c\beta{d\log(1/\delta)}\tilde{\mathcal{W}}_0 }{12q (n\epsilon)^2} \\
    &  ~~~~+  \frac{c d\log(1/\delta)\sqrt{\frac{101}{12}}\beta^{3/2}}{qn\epsilon\sqrt{\mathcal{R}_{F}}} \big(\frac{\sum_{k=1}^T d\tau^2_{k}}{12} + \frac{21\tilde{\mathcal{W}}_0^2}{24}\big),
\end{aligned}
\end{equation}
}}\noindent
where $\mathcal{R}_{F} = \sup_{w}F(w)- \inf_{w} F(w)$, $v$ is some constant determined by the noise mechanism and $\tilde{\mathcal{W}}_0 = \|w_0-w_{-1}\|$ is the initial divergence. When $\tau_{k} = O(\eta)$, the right hand of  {(\ref{batch_main_1})} is $ O(\frac{c\sqrt{d \log(1/\delta)}}{n\epsilon})$.
\end{thm}

\begin{proof}
See Appendix \ref{app:pr_convergence_MM}. 
\end{proof}

In Theorem \ref{thm:convergence_MM}, we do not assume the objective function $F(w)$ to be convex or Lipschitz, but only smooth with concentrated stochastic gradients  (Assumption {\ref{assp:subexp}}).\footnote{{In the proof of Theorem {\ref{thm:convergence_MM}}, essentially we only need a high-probability bound of the sampling noise in the stochastic gradient, which is also necessary (see Example {\ref{ex:non-convergence}}). Therefore, one can relax Assumption {\ref{assp:subexp}} of a subexponential tail to any proper concentration assumption that allows the derivation of a high probability bound of the sampling noise.
}}
The utility loss is measured by the norm of the gradient $\|\nabla F(w_k)\|$ in Equation (\ref{batch_main_1}), commonly considered in non-convex optimization \cite{Wang2017, clipZhang2019}. 
There are several interesting observations from Theorem \ref{thm:convergence_MM}. 
First, in clipped DP-SGD, we should not simply consider the clipping threshold $c$ as a virtual Lipschitz constant. 
When the gradient norm $\|\nabla F(w_k)\|$ is large, Equation (\ref{batch_main_1}) suggests that $${\sum_{k=0}^{T-1}\|\nabla F(w_k)\|} = O(T \cdot \frac{\sqrt{\log(1/\delta)d}}{n\epsilon}),$$
which is independent of the clipping threshold $c$ as long as $c$ satisfies Equation (\ref{selection-of-c}). 
Therefore, when the objective function is hard to optimize, or when the gradient is large during the initial epochs, we should select a large clipping threshold to minimize the effect of clipping and let the state $w_k$ approach some neighborhood domain of a (local) minimum fast. 
When $\|\nabla F(w_k)\|$ becomes small, $c$ will reappear in the utility bound.
In this case, the equation becomes similar to the classic DP-SGD privacy-utility tradeoff \cite{FOCS2014} but replacing the Lipschitz constant $L$ by $c$, where 
$$\frac{\sum_{k=0}^{T-1}\|\nabla F(w_k)\|}{T} = O(\frac{c\sqrt{\log(1/\delta)d}}{n\epsilon}).$$ 
Therefore, we want to select some properly small $c$.

The second and also more important issue worth mentioning is the restriction on $c$ in Equation (\ref{selection-of-c}). 
We require the norm of the sampling noise in the stochastic gradient to be smaller than the clipping threshold $c$ with sufficient probability. 
{\em Indeed, this is actually a necessary condition for clipped DP-SGD to work. } 
In the following example, we show that even without any noise perturbation, clipped SGD could fail when $c$ is much smaller than the magnitude of sampling noise. 

\begin{example}
\label{ex:non-convergence}
Suppose the loss function is $f(w, x) = (w - x)^2 / 2$ and we have three samples $x_1=-20,x_2=-10$ and $x_3=90$.
The overall loss function is thus $F(w)= 1/3 \cdot \sum_{i=1}^3 (w-x_i)^2/2$.
Thus, $\nabla F(w)=w-20$ and the minimum is achieved when $w=20$. 
Moreover, the {standard deviation}  of the per-sample stochastic gradient, where we randomly sample $x'\leftarrow \{x_1, x_2, x_3\}$ and compute the gradient of $f(w, x')$, is $49.7$ for any $w$.  

We will show that if the clipping threshold is relatively small, for example, $c=1$, then the clipped gradient is dramatically different from the true gradient $\nabla F(w)$.
We first consider the case when the magnitude of true gradient is small, say at the optimum $w=20$ where $\nabla F(20)=0$. 
At $w=20$, the expectation of the clipped stochastic gradient is $1/3$ rather than $0$. 
On the other hand, when the magnitude of the true gradient is large, say $w=0$ and $\nabla F(0)=-20$, the expectation of the clipped stochastic gradient is still $1/3$.
But now it is in the opposite direction of the true gradient. 
\end{example}
From Example \ref{ex:non-convergence}, we know that if $c$ is not properly selected, clipped SGD could fail to converge regardless of the magnitude of the full-batch gradient. As a summary, Theorem \ref{thm:convergence_MM} and Example \ref{ex:non-convergence} suggest the following two key observations on clipped DP-SGD:
\begin{enumerate}
\item On hyper-parameter selection, to ensure stable convergence in general, we need the clipping threshold $c$ to be sufficiently larger than the sampling noise of stochastic gradient. 
\item On the other hand, to improve the performance of clipped DP-SGD, one promising direction is to reduce the sampling noise via network architecture optimization and data normalization. Indeed, many approaches supporting this goal are proposed and studied in deep learning research, for example, batchnorm layer \cite{batchnorm}, where the gradient is evaluated and normalized by a group of samples. We defer a systematical study on improvement via gradient variance reduction to our future work.
\end{enumerate}

\section{Privacy Guarantees and Amplifications}
\label{sec: privacy}
\subsection{Numerical Calculation and Asymptotic Amplification}
\noindent In this section, we first show how to numerically compute the $(\epsilon, \delta)$ bound via RDP and asymptotically analyze the amplification of ModelMix. 
We use Gaussian Mechanism to produce the noise $\Delta_k \sim \mathcal{N}(0 , \sigma^2 \cdot \bm{I}_{d})$ across iterations.
The following Theorem \ref{thm:RDP}  shows an efficiently-calculable upper bound of $(\epsilon, \delta)$ by measuring an $\alpha$-Rényi divergence between two one-dimensional distributions. 
For simplicity, we fix $\tau_k = \tau$ in the following. 

\begin{thm}[{Privacy Calculation of ModelMix via RDP}]
\label{thm:RDP} 
Suppose $\Delta_k$ is Gaussian noise generated from $\mathcal{N}(0 , \sigma^2 \cdot \bm{I}_{d})$. 
We define two distributions $\mathsf{P}_0$ and $\mathsf{P}_1$,
\[
\mathsf{P}_0 = \mathcal{N}(0,\sigma)*\mathcal{U}[-\tau/(2\eta),\tau/(2\eta)],
\]
\[
\mathsf{P}_1 = \mathcal{N}(c,\sigma)*\mathcal{U}[-\tau/(2\eta),\tau/(2\eta)],
\]
where $*$ represents the convolution of two probability distributions.
In other words, sampling from $\mathsf{P}_0$ is equivalent to sampling from Gaussian distribution $\mathcal{N}(0,\sigma)$ and uniform distribution $\mathcal{U}[-\tau/(2\eta),\tau/(2\eta)]$ independently, then summing the results together.
Similarly, $\mathsf{P}_1$ is the convolution of $\mathcal{N}(c,\sigma)$ and $\mathcal{U}[-\tau/(2\eta),\tau/(2\eta)]$.

Algorithm \ref{alg: ModelMix} for $T$ iterations with sampling rate $q$ satisfies $(\epsilon, \delta)$ DP for any $\epsilon$ and $\delta$ such that
$$\epsilon \leq \min_{\alpha \in \mathbb{Z}, \alpha>1} T\mathsf{D}_{\alpha}\big( (1-q)\mathsf{P}_0 + q \mathsf{P}_1 \| \mathsf{P}_0 \big) + {\log(1/\delta)\over (\alpha-1)}.$$
\end{thm}

\begin{proof}
See Appendix \ref{app:proof_RDP}. 
\end{proof}

Let $\mathcal{D}$ and $\mathcal{D}'$ be two neighboring datasets, where $(\tilde{x}, \tilde{y}) = \mathcal{D}-\mathcal{D}'$ is the differing element.
Theorem \ref{thm:RDP} shows that the worst output divergence between $\mathcal{D}$ and $\mathcal{D}'$ occurs when the gradient $\nabla f(w_k,\tilde{x}, \tilde{y})$ is a vector whose Hamming weight is 1, such as $(c,0,...,0)$. 
In other words, given an $l_2$-norm budget of $c$, the worst-case $\mathsf{D}_{\alpha}$ divergence happens when the gradient of $(\tilde{x}, \tilde{y})$ is concentrated on a single dimension. This is different from standard (subsampled) Gaussian Mechanism, where the randomization is only due to the Gaussian noise and subsampling \cite{mironov2019r}. 
We will exploit this property to design more fine-tuned gradient clipping for ModelMix in Section \ref{sec: l_infty}. In the following, we give an asymptotic analysis on the privacy amplification of ModelMix. 

\begin{thm}[{Asymptotic Privacy Amplification from ModelMix}]
\label{thm:amplification} 
If the sampling rate $q$ is some constant, the noise $\Delta_{k}$ is selected to be Gaussian noise $\mathcal{N}(\bm{0},\sigma^2 \bm{I}_d)$ and $\tau$ is sufficiently large, then Algorithm 
 \ref{alg: ModelMix} with $T$ iterations satisfies $(\epsilon, \delta)$-DP guarantee for any $\epsilon$ and $\delta$ such that
{\small{
\begin{equation}
\begin{aligned}
 \epsilon =  \tilde{O}\big(\frac{\eta T(c+\sigma)}{{\tau}n} & + \sqrt{
\frac{\eta T\log(1/\tilde{\delta})}{\tau}\cdot (\frac{c^4}{n^4\sigma^3}+ \frac{c^2}{n^2\sigma} + \sigma)  }\big),\\
\delta & = \tilde{O}(T\cdot \frac{(L/n)+\sigma}{\tau} + \tilde{\delta}). \footnote{\hl{In the case of clipped gradient, one can simply replace $L$ with $c$},}
\end{aligned}
\end{equation}
}}\noindent
where $\tilde{\delta}$ can be any value within $(0,1)$. 

In addition, if in Algorithm \ref{alg: ModelMix}, the per-sample gradient is clipped up to $c$ in $l_1$ norm and the noise $\Delta_{k}$ is selected to be Laplace noise, then Algorithm \ref{alg: ModelMix} satisfies $(\epsilon, \delta)$-DP, where 
{\small{
\begin{equation}
  \epsilon =   \tilde{O}\big(\frac{\eta T\epsilon_0(e^{\epsilon_0}-1)}{\tau} + \epsilon_0 \sqrt{\frac{\eta T \log(1/\delta)}{\tau}}\big),
\end{equation}
}}
{where $\epsilon_0$ represents the $\epsilon$ privacy loss of a single iteration of {(\ref{main_update})} without ModelMix (set $\bm{\alpha}_k$ to be constant)}.  
\end{thm}

\begin{proof}
See Appendix \ref{app:proof_amplification}.
\end{proof}

We can compare between Theorem \ref{thm:amplification} and Theorem \ref{thm:composition} (the case of regular DP-SGD without ModelMix).  If a single iteration of a DP-SGD is $\epsilon_0$-differentially-private, then 
classic advanced composition (Theorem \ref{thm:composition}) suggests that the composition of $T$ iterations produces $O(\epsilon_0\sqrt{T\log(1/\delta)},\delta)$-DP when $\epsilon_0$ is small. 
Theorem \ref{thm:amplification} states that under the same setup, if the DP-SGD is further incorporated with ModelMix, then this composition becomes $\tilde{O}(\epsilon_0\sqrt{T\log(1/\delta)/\tau}, \delta)$, where {the $\epsilon$ term decreases by a factor of $\tilde{O}(1/\sqrt{\tau})$}.

\subsection{Clipping with $l_2$ and $l_{\infty}$ Sensitivity Guarantee}
\label{sec: l_infty}
\begin{algorithm}[t]
\caption{$l_2$ and $l_{\infty}$ Norm Gradient Clipping}
\begin{algorithmic}[1]
\STATE \textbf{Input:} Individual gradient $\nabla f(w, x, y)$, $l_2$ norm clipping threshold $c$, $l_{\infty}$ norm truncation parameter $p \in \mathbb{Z}^{+}$. 
\STATE Clip $\nabla f(w,x,y) \gets \mathsf{CP}(\nabla f(w,x,y), c)$
\FOR{$j=1,2, ... , d$}
 \IF{$|\nabla f(w, x, y) (j)| > \frac{c}{\sqrt{p}}$}
 \STATE {$\nabla f(w, x, y) (j) \gets 
  \text{sign} (\nabla f(w, x, y) (j)) \frac{c}{\sqrt{p}}$}
 \ENDIF
\ENDFOR
\STATE \textbf{Output:} $\nabla f(w, x, y).$
\end{algorithmic}
\label{alg: l_infty clipping}
\end{algorithm}

\noindent In this subsection, we show that in the framework of ModelMix with Gaussian Mechanism, one can further strengthen the privacy amplification if both the $l_2$ and $l_{\infty}$ norm sensitivity can be guaranteed. 
We first present a gradient clipping algorithm, described in Algorithm \ref{alg: l_infty clipping}. 
Algorithm \ref{alg: l_infty clipping} is a simple generalization of standard $l_2$-norm clipping operator $\mathsf{CP}(\cdot, c)$. For an individual gradient $\nabla f(w,x,y)$,  we first clip the gradient up to $c$ in $l_2$ norm. After that, we truncate each coordinate such that its absolute value does not exceed $c/\sqrt{p}$ for some constant integer $p \geq 1$. As a consequence, Algorithm \ref{alg: l_infty clipping} ensures that the $l_2$ and $l_{\infty}$ norm of clipped gradient is upper bounded by $c$ and $c/\sqrt{p}$, respectively. 
From our empirical observations, compared to $l_2$-norm clipping on $c$, DP-SGD is much less sensitive to $l_\infty$-norm truncation on $p$. The reason behind this is that in practice the gradients obtained are rarely concentrated on few coordinates and thus $l_\infty$ truncation with large $p$, even in the hundreds, will hardly change the geometry of gradients. Roughly speaking, $p$ captures the number of significant coordinates in the gradient, where the  gradient dimension could be millions in deep learning. However, such truncation enables us to derive a stronger amplification bound with ModelMix, as shown in  Corollary \ref{cor: l_infty}. 

\begin{cor}
Under the same setup of Theorem \ref{thm:RDP}, if we further ensure that the $l_{\infty}$ norm of each individual gradient in Algorithm \ref{alg: ModelMix} is bounded by $c/\sqrt{p}$ for some $p \in \mathbb{Z}^+$, then it satisfies $(\epsilon, \delta)$ DP for any $\epsilon$ and $\delta$ such that
$$ \epsilon \leq \min_{\alpha \in \mathbb{Z}, \alpha>1} \frac{\log \big( 
 \sum_{k=0}^{\alpha} \tbinom{\alpha}{k}(1-q)^{\alpha-k}q^k \mathcal{A}_{k}\big) + \log(1/\delta)}{1-\alpha}, $$
 where $\mathcal{A}_{k} = \big(\mathbb{E}_{z \sim \mathsf{P}_0} \big[ (\mathsf{P}'_1(z)/{\mathsf{P}_0(z)})^k \big]\big)^{p}$. 
 Here, $\mathsf{P}_0 = \mathcal{N}(0,\sigma)*\mathcal{U}[-\tau/(2\eta),\tau/(2\eta)]$ and $\mathsf{P}'_1 = \mathcal{N}(c/\sqrt{p},\sigma)*\mathcal{U}[-\tau/(2\eta),\tau/(2\eta)]$, i.e., $\mathsf{P}_0$ shifted by $c/\sqrt{p}$. 
 \label{cor: l_infty}
\end{cor}

\begin{proof}
See Appendix \ref{app:proof_RDP}. 
\end{proof}

Corollary \ref{cor: l_infty} generalizes Theorem \ref{thm:RDP} and recomputes the worst case divergence when we have both $l_2$ and $l_{\infty}$ norm sensitivity guarantees. 
We prove that the worst case happens when the gradient $\nabla f(w_k,\tilde{x}, \tilde{y})$ of the differing element $(\tilde{x}, \tilde{y})$ is in a form whose Hamming weight is $p$ and each non-zero coordinate is equal to $\pm c/\sqrt{p}$. 
Corollary \ref{cor: l_infty} also shows an efficient way to numerically compute the privacy loss by measuring the divergence between two one-dimensional distributions. 

\subsection{Privacy Amplification Examples}
\noindent In Fig.  \ref{fig_amplification}, we provide concrete examples of DP-SGD with ModelMix under various setups. We set $n=50,000$, $\delta=10^{-5}$, the clipping threshold $c=20$, and run DP-SGD under the Gaussian Mechanism \cite{Dwork-algorithimc}. 
We measure the cumulative privacy loss in terms of $\epsilon$ under various sampling rates $q$, mixing thresholds $\tau$ and $l_{\infty}$ truncation parameter $p$. In all subfigures, we set the total number of iterations $T=5,000$, and the budget $\epsilon=200$ for regular DP-SGD. 

In Fig. \ref{fig_amplification} (a-c), we set the sampling rate $q = 0.02$ (a minibatch of size 1000 in expectation), and examine the privacy loss under $\tau = 0.075\eta, 0.15\eta, 0.3\eta$, where $\eta$ is the stepsize, respectively.
With ModelMix, especially with further help of Algorithm \ref{alg: l_infty clipping} to ensure both $l_2$ and $l_{\infty}$ norm sensitivity, we achieve orders of magnitude improvement. For example, in Fig. \ref{fig_amplification} (c) where $\tau = 0.3\eta$, compared to $(\epsilon=200, \delta=10^{-5})$ using regular DP-SGD, under the same setup, ModelMix, ModelMix with further $l_{\infty}$ truncation of $p=25$ and $p=100$ produce $\epsilon = 31.7, 5.4$ and $4.8$, respectively, with the same $\delta=10^{-5}$. The corresponding $\epsilon$ numbers produced in  Fig. \ref{fig_amplification} (a) and (b) are (57.2, 17.9, 15.3) and (40.4, 9.0, 7.9), repectively. 
The amplification factor of ModelMix matches our theoretical results, which is $\tilde{O}(1/\sqrt{\tau})$. On the other hand, the additional amplification from the $l_{\infty}$ norm truncation will reach some limit as $p$ increases. For large enough $p$, such combined amplification is empirically ${O}(1/{\tau})$. {In Fig. {\ref{fig_amplification}} (e-f), we change the sampling rate $q$ to be $0.04$ while keeping the remaining parameters the same as those in Fig. {\ref{fig_amplification}} (a-c). In general, given larger privacy budget and mixing threshold $\tau$, ModelMix can produce stronger privacy amplification. 

\begin{figure}[t]
  \centering
  \includegraphics[width=0.49\linewidth]{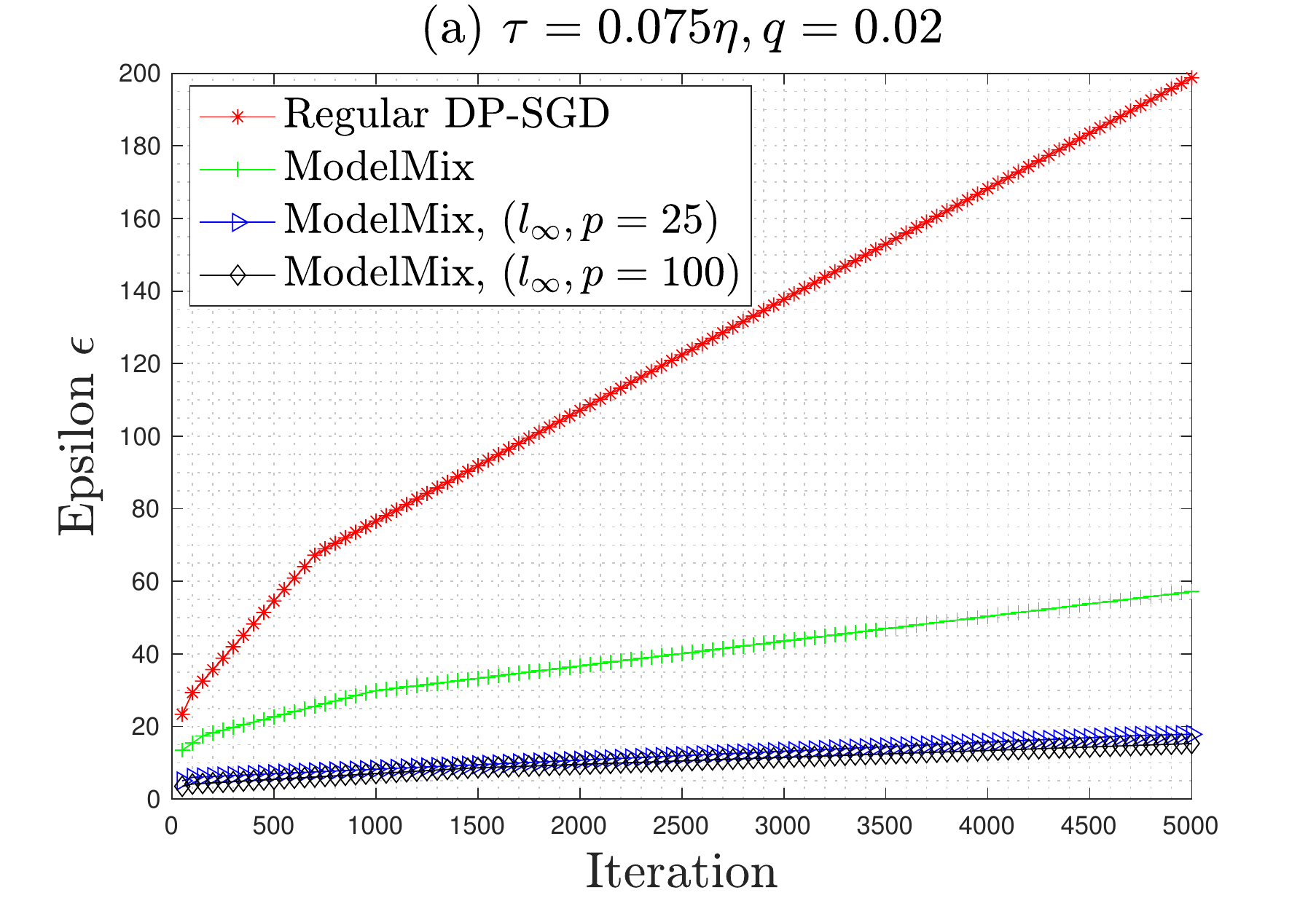}
  \includegraphics[width=0.49\linewidth]{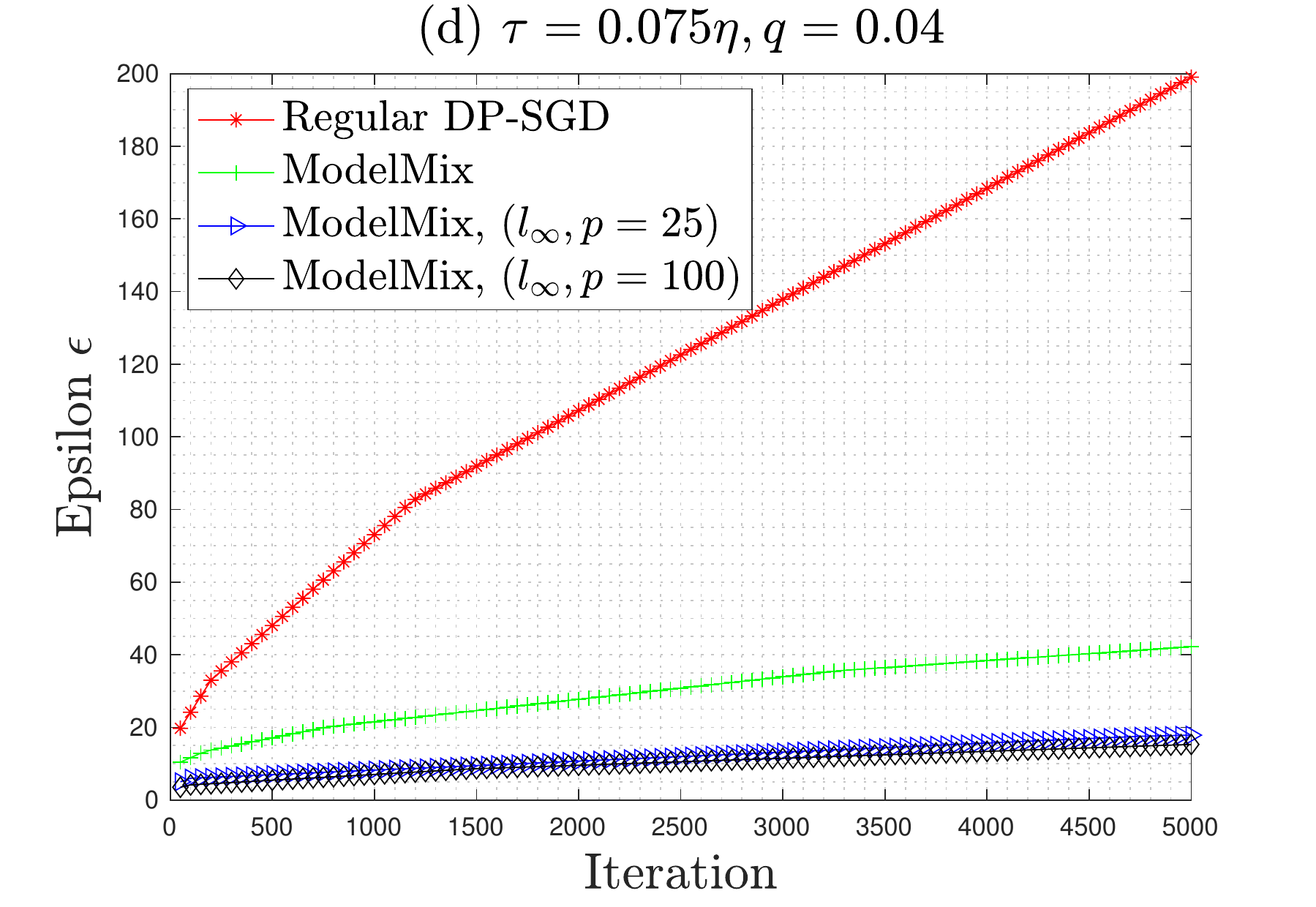}
  \includegraphics[width=0.49\linewidth]{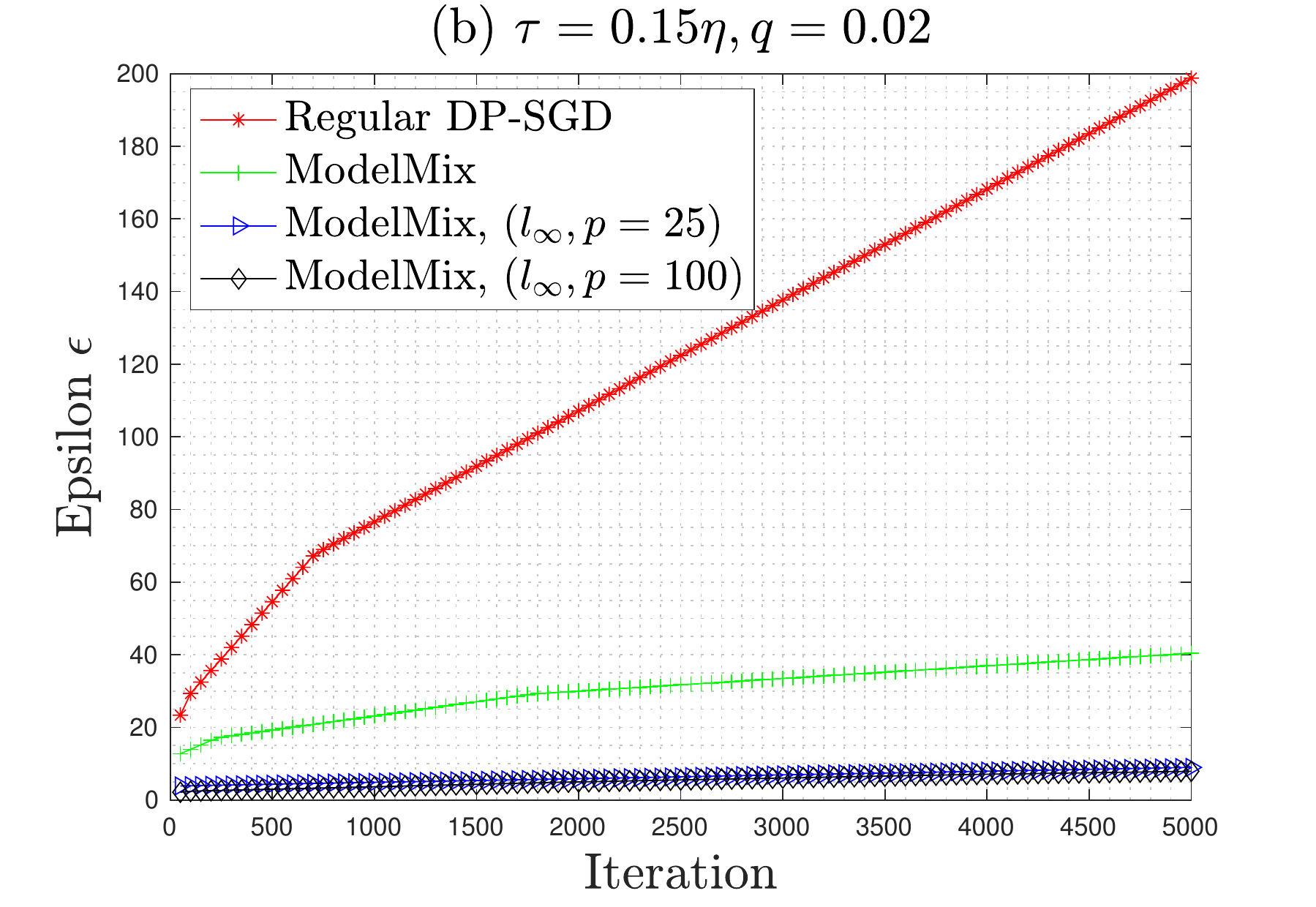}
  \includegraphics[width=0.49\linewidth]{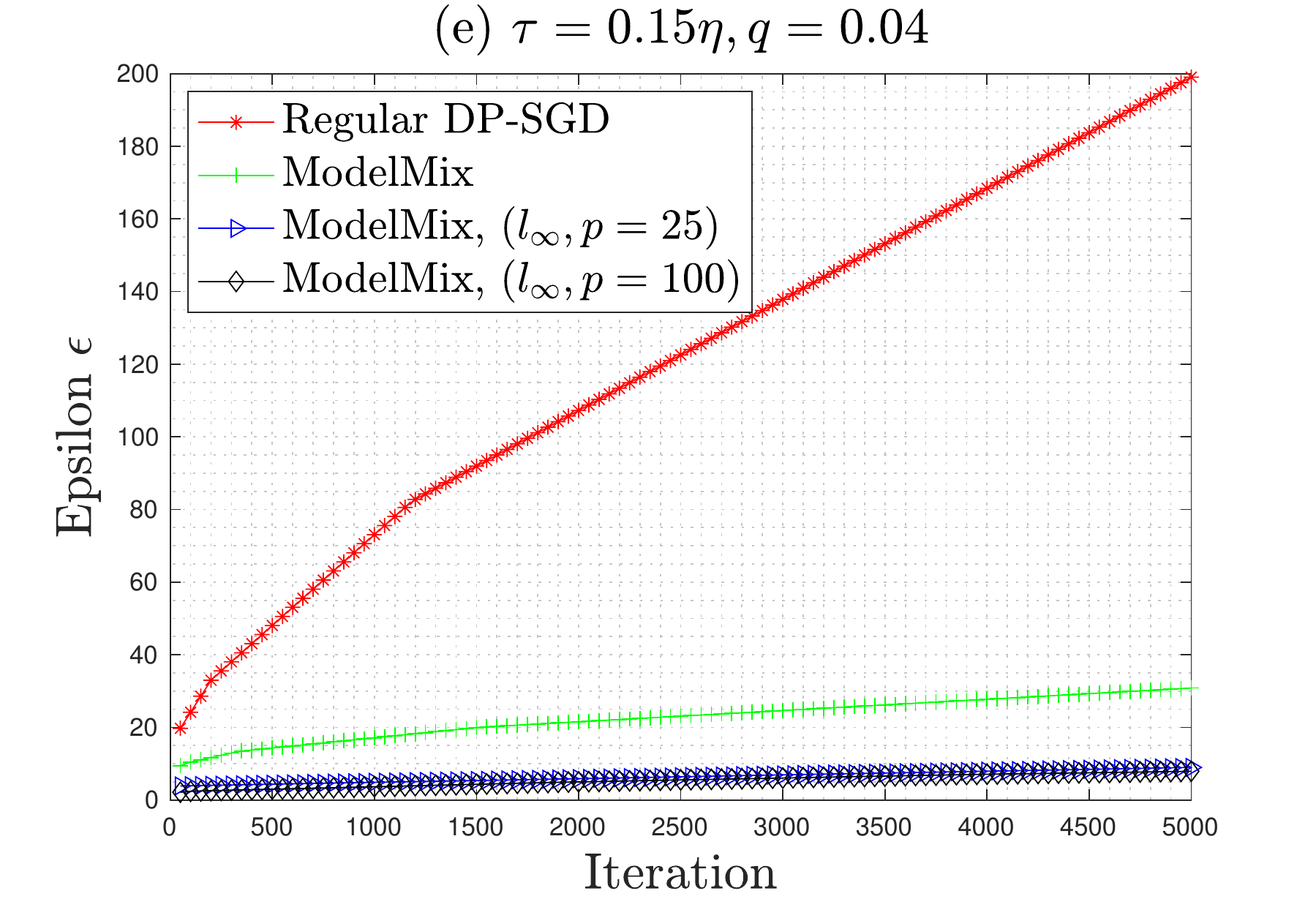} 
  \includegraphics[width=0.49\linewidth]{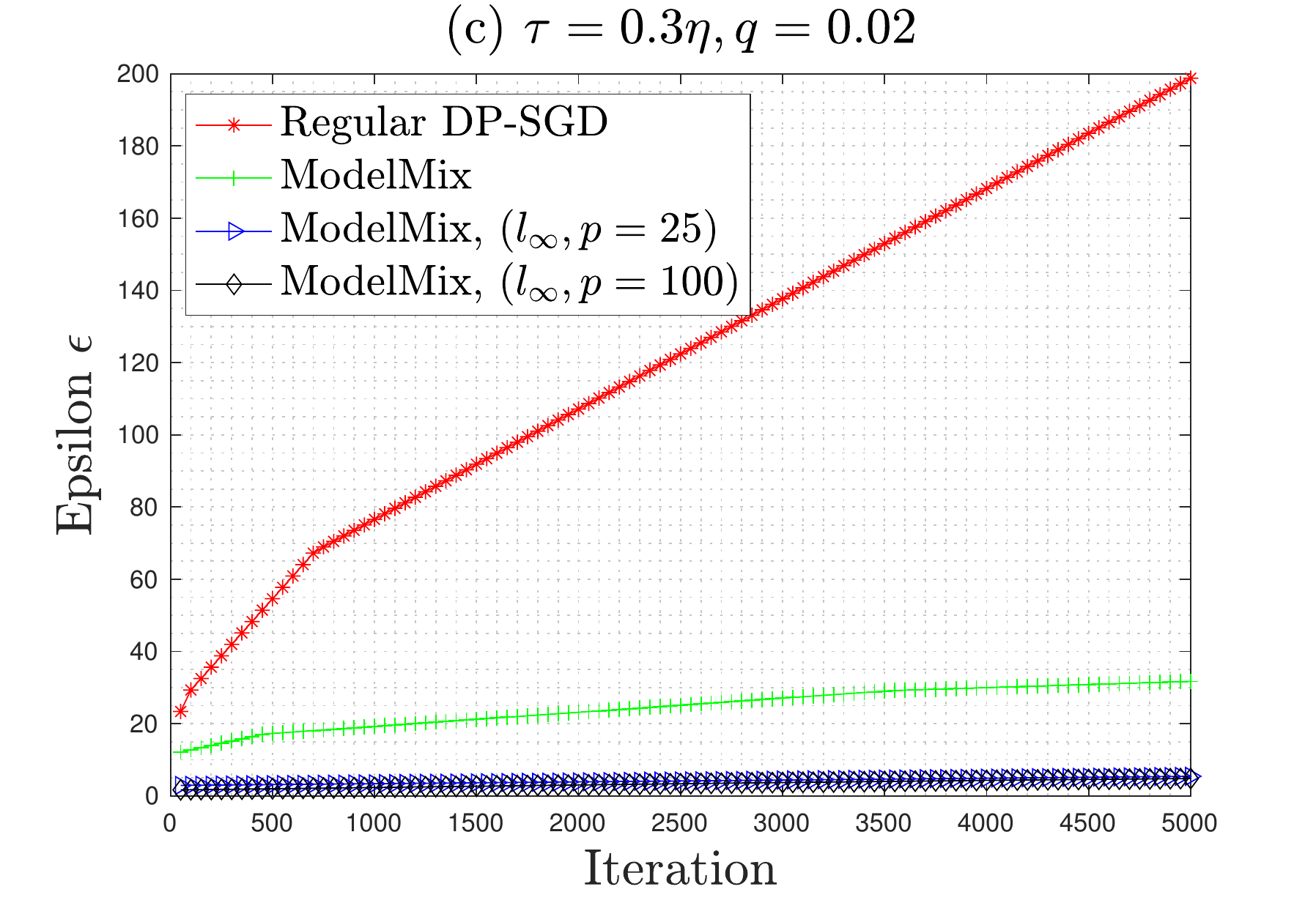}
  \includegraphics[width=0.49\linewidth]{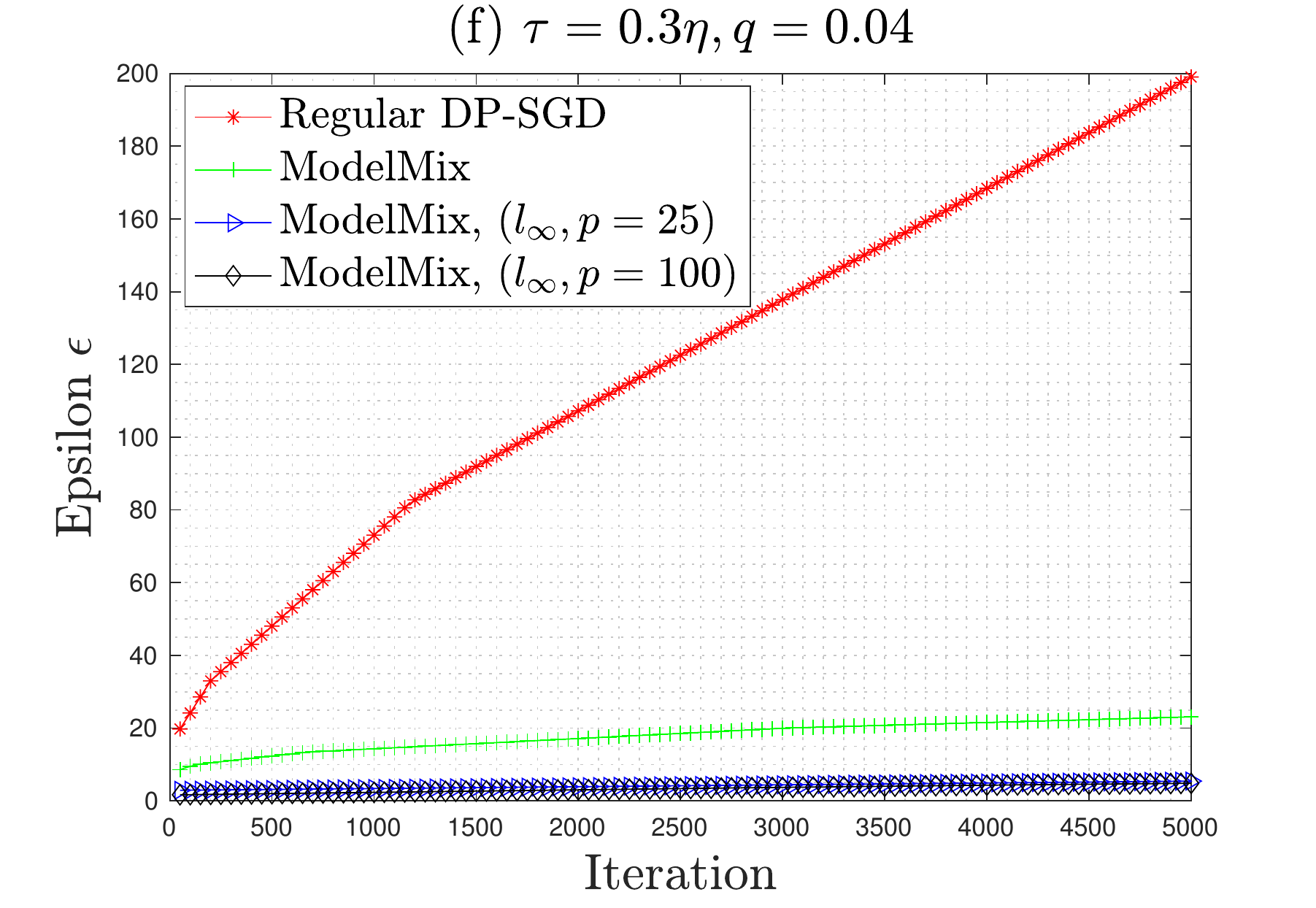} 
\caption{Privacy amplification from ModelMix on running DP-SGD with $n=50K$ samples under Gaussian Mechanism.}
\label{fig_amplification}
\vspace{-0.2 in}
\end{figure}
\subsection{More Insights on Privacy Amplification}

\noindent In the following, we provide more technical insights on how ModelMix improves the privacy. A quick intuition is that, compared to standard DP-SGD (Equation (\ref{dpsgd_with_cliping})) where $w_k$ is only randomized by the noise $\Delta_{k}$, the sources of randomness in Algorithm \ref{alg: ModelMix} are enriched, containing both the randomness in ModelMix and the noise perturbation.
More randomness often implies stronger privacy. However, {\em though privacy relies on randomness, not all kinds of randomness can produce worst-case DP guarantees.} ModelMix is an interesting example. 
ModelMix must be applied with other noise mechanisms to produce meaningful DP guarantees, as explained below. 

In ModelMix, the mixed state is within a bounded hull determined by earlier updates.
This means that the randomness is bounded and localized, which cannot produce reasonable worst-case DP guarantees, in contrast to Laplace/Gaussian noise.
To better illustrate this, consider the following example.
Suppose $\nabla F(w_k, \mathcal{D})=0$ and $\nabla F(w_k, \mathcal{D}')=1$ are two gradients evaluated by adjacent datasets $\mathcal{D}$ and $\mathcal{D}'$.
We randomize the gradients by introducing $\mathcal{U}[0,5]$, the uniform distribution between $[0,5]$. 
The perturbed gradients now become equivalent to being uniformly sampled from $\mathcal{U}[0,5]$ and $\mathcal{U}[1,6]$ respectively.
When we compare the divergence between the two distributions $\mathcal{U}[0,5]$ and $\mathcal{U}[1,6]$, we cannot derive a worst-case $\epsilon$ bound ($\epsilon=\infty$ in this example) or an efficiently-composable $(\epsilon, \delta)$ DP guarantee \footnote{To achieve a meaningful privacy guarantee in DP-SGD for $T$ iterations, we need the failure probability $\delta = o(1/(nT))$ in a single iteration.}. To this end, additional Laplace or Gaussian noise is needed.

However, even with additional noise, we cannot improve the privacy loss of a single iteration significantly. 
Continuing with the above example, in a single iteration of DP-SGD, if we add some Laplace noise $Lap(0,\lambda)$, which ensures $(\epsilon_0,0)$-DP, and further perturb the update by a uniform noise within $\mathcal{U}(0,5)$, the distributions of the two perturbed gradients $\nabla F(w_k, \mathcal{D})=0$ and $\nabla F(w_k, \mathcal{D}')=1$ now become $Lap(0,\lambda)*\mathcal{U}(0,5)$ and $Lap(0,\lambda)*\mathcal{U}(1,6)$, respectively. 
Here, $*$ represents the convolution of two distributions. Still, the worst-case is not improved where we can only claim $\epsilon_0$-DP or still some not efficiently-composable $(\epsilon,\delta)$-DP for a single iteration.

So how does the localized randomness in ModelMix amplify the privacy? While ModelMix cannot significantly improve the worst-case utility-privacy tradeoff in a single iteration {\protect{\cite{FOCS2014}}}, as shown in Theorems \ref{thm:RDP} and \ref{thm:amplification}, it smoothens the divergence between the output distributions from two arbitrary adjacent datasets. 
{\em Such an amplification is limited in a single iteration, but when we consider the composition of the privacy loss, the amplification accumulates to produce a strong improvement.} 
To be specific, for the DP-SGD update protocol (Equation (\ref{dpsgd_with_cliping})), denoted by $\mathcal{M}$, we consider the  {\em pointwise $\epsilon(w)$-loss} at a particular output $w$ \cite{CCS2016},
$$ \epsilon_{\mathcal{D},\mathcal{D}',\mathsf{Aux}}(w) = \log \frac{\mathbb{P}(\mathcal{M}(\mathcal{D}, \mathsf{Aux})= w) }{\mathbb{P}(\mathcal{M}(\mathcal{D'},\mathsf{Aux})= w)},$$
where $\mathcal{D},\mathcal{D}'$ are two adjacent datasets and $\mathsf{Aux}$ represents the other auxiliary inputs used in $\mathcal{M}$. 
When we take the supremum $\epsilon = \sup_{\mathcal{D},\mathcal{D}',\text{Aux}}\epsilon_{\mathcal{D},\mathcal{D}',\mathsf{Aux}}(w)$, it gives an equivalent $\epsilon$-DP guarantee. As explained before, incorporation of ModelMix  (Equation (\ref{main_update})) cannot improve this (supremum) worst-case loss. 
What we proved in Theorem \ref{thm:amplification} is that, for any two datasets $\mathcal{D},\mathcal{D}'$ and $\mathsf{Aux}$, the expectation $\mathbb{E}[\epsilon(w)]$  and the variance of $\text{Var}[\epsilon(w)]$ for $w \sim \mathcal{M}(\mathcal{D},\text{Aux})$ will scale by $\tilde{O}(1/\tau)$ when ModelMix is further applied. Provided such an $\epsilon(w)$ loss of smaller mean and stronger concentration, we can derive a stronger {\em high probability bound} in a composite $(\epsilon, \delta)$ form when we measure the cumulative $\epsilon$ loss across the states $w_{[1:T]}$ from $T$ iterations.

\subsection{Randomization beyond Noise}
\label{sec:beyond}
\noindent The analysis framework of ModelMix also sheds light on how to quantify the privacy amplification from a large class of “training-oriented” randomization commonly applied in deep learning. 
There are many reasons to introduce randomness in optimization and learning other than privacy preservation.
For example, in stochastic gradient Langevin dynamics (SGLD) \cite{SGLD2018,SGLD2019} for nonconvex optimization, it is common to utilize noisy gradient descent to escape saddle points \cite{saddle2017}. Randomization can also strengthen the learning performance, e.g., random dropout \cite{dropout} and data augmentation \cite{dataaugment2019}. 
In particular, data augmentation plays an important role in modern computer vision. 
Generally speaking, data augmentation represents a large class of methods to improve robustness and reduce memorization (instead of generalization) by generating virtual samples through random cropping \cite{cropping2012}, erasing \cite{erasing} or mixing \cite{mixup} the raw samples. 

All of the above-mentioned randomnesses are localized, which cannot produce meaningful DP guarantees for the same reason as ModelMix. 
For example, consider random erasing \cite{erasing}, where a rectangle region of an image is randomly erased and replaced with random values. 
When we process private images using the above mechanisms, the random transformation is multiplicative over the private input, meaning that the output is restricted to a bounded domain determined by the specific input processed. 
Given two different images with at least one differentiating pixel, one can still distinguish them after random erasing as long as at least one differentiating pixel is not erased. 
A similar argument holds for dropout, where a node in a neural network is ignored independently with some fixed rate, and the saddle point escaping algorithm \cite{saddle2017}, where the gradient is perturbed by a bounded noise uniformly selected from a sphere. Our privacy analysis framework shows a way to analyze these randomizations in practical learning algorithms combined with DP-SGD to produce sharpened composition bounds.

\section{Further Experiments}
\label{sec:exp}
\noindent ModelMix is a generic technique, which can be implemented in almost all applications of DP-SGD without further assumptions. 
In this section, we provide further experiments to measure its performance combined with other state-of-the-art advances in DP-SGD. 
Below, when we report the improvement of ModelMix over existing works, we mostly follow the optimal hyper-parameters each work suggests. 
We repeat each simulation five times and report the median of the results. 
Details of the hyper-parameter selections we used can be found in Appendix \ref{app:exp}. Based on the experiments conducted by previous works, to have a clear comparison, we will also test the proposed algorithms on the following three benchmark datasets, CIFAR10, SVHN and FMNIST\footnote{https://deepobs.readthedocs.io/en/stable/api/datasets/fmnist.html.}.  
CIFAR10 consists of 60,000 color images in 10 classes, where 50,000 are for training and 10,000 for test. 
The Street View House Numbers (SVHN) dataset has 73,257 images of real world house digits for training and 26,032 for test.\footnote{We do not use the 600,000 auxiliary samples provided in SVHN in our experiments.} 
Fashion MNIST (FMNIST) contains 70,000 greyscale images of fashion products from 10 categories with 60,000 for training and 10,000 for test. 
In the following experiments, we will assume that all the training samples in the above-mentioned datasets are private. 

\subsection{Shallow Network}
\label{sec:shallow}
\noindent Since the magnitude of gradient perturbation is adversely dependent on the model size, instead of using the cutting-edge deep models, a large number of works are devoted to building small models to carefully balance model capacity and utility loss caused by DP-SGD \cite{CCS2016, tempered2021, DanICLR2020,papernotfit2019}.  
One of the best existing results is given by \cite{DanICLR2020}. 
In \cite{DanICLR2020}, Tramer and Boneh showed that handcrafted features extracted from raw data can significantly strengthen the performance of training shallow models with DP-SGD. 
They proposed to first privately use ScatterNet \cite{scattering} to process the dataset and to then apply DP-SGD on the ScatterNet features. 
With a DP budget of $(\epsilon=3,\delta=10^{-5})$,  \cite{DanICLR2020} successfully trained a five-layer CNN with ScatterNet, which achieves $66.9\%$ and $87.2\%$ accuracy on CIFAR10 and FMNIST, respectively. 
In the following, we consider further improving their results by using ModelMix. 
With the same setup, we spend $\epsilon=0.695$ budget to privately estimate the statistics of datasets to apply ScatterNet. 
In Table \ref{tab:sca-CNN}, we compare the utility-privacy tradeoff when we run DP-SGD with/without ModelMix on training the same CNN suggested by \cite{DanICLR2020} on the ScatterNet features, where $\delta$ is always fixed to be $10^{-5}$ in all cases. Table \ref{tab:sca-CNN} shows that ModelMix can bring significant improvement even in simple model training with very low privacy budget.

\begin{table}
\begin{center}
\begin{tabular}{c r r r r r r} 
\hline
\hline 
\multicolumn{7}{c}{CIFAR10}\\
\hline
Method$\backslash$Privacy & $\epsilon=0.2$ & $0.4$ & $0.6$ & $0.8$ & $1.0$& $ \infty$ \\
 \hline
 \cite{DanICLR2020}   & 32.8 & 45.9  & 56.1  & 61.2  & 64.1 & 74.6 \\
 DP-SGD +MM  & 58.5 & 63.3  &  65.0  & 66.9 & 68.3 & 74.6 \\
\hline 
\multicolumn{7}{c}{FMNIST}\\
\hline
Method$\backslash$Privacy & $\epsilon=0.2$ & $0.4$ & $0.6$ & $0.8$ & $1.0$ & $\infty$ \\
 \hline
 \cite{DanICLR2020}  & 53.4 & 74.0  & 81.5  & 84.5  & 85.6 & 91.2 \\
  DP-SGD+MM  & 83.9 & 85.7  &  86.1  & 87.9 & 88.8 & 91.2 \\
\hline 
\hline 
\end{tabular}
\end{center}
\caption{\textbf{{Test Accuracy of} DP-SGD with/out ModelMix (MM) on training small CNN with ScatterNet Features.}}
\label{tab:sca-CNN}
 \vspace{-0.1 in}
\end{table}

\subsection{Deep Models}
\label{sec:deep}
\noindent In this subsection, we consider applying both ModelMix to help train large neural networks with DP-SGD. We take CIFAR10 and SVHN as examples. We will further apply Algorithm \ref{alg: l_infty clipping} to enhance privacy amplification. In particular, according to our analysis on the clipping threshold and sampling noise, we select $c=20$ and the $l_{\infty}$-norm truncation parameter $p=100$ in all experiments.  In Table \ref{tab:MM_BC}, we show the tradeoff between privacy and learning accuracy when training Resnet20 on CIFAR10 and SVHN with DP-SGD, respectively. Still, in all the cases, $\delta$ is fixed to be $10^{-5}$. With an $(\epsilon=8, \delta=10^{-5})$-DP budget, combined with ModelMix, we train Resnet20 which achieves $70.4\%$ and $90.1\%$ accuracy on CIFAR10 and SVHN, respectively. In comparison, regular DP-SGD can only produce $56.1\%$ and $84.9\%$ accuracy. In Table \ref{tab:MM_acc}, we also include results for particular test accuracy the privacy budget required. In general, ModelMix roughly brings 20$\times$ improvement on $\epsilon$ to produce usable security parameters.  
\begin{table}[t]
\begin{center}
\begin{tabular}{c r r r r r r} 
\hline
\hline 
\multicolumn{7}{c}{CIFAR10}\\
\hline
Method$\backslash$Privacy & $\epsilon=4$ & $5$ & $6$ & $7$ & $8$ & $\infty$  \\
 \hline
Regular DP-SGD  & 47.3 & 51.1 & 53.2 & 54.7  & 56.1 & 90.7 \\
 DP-SGD+MM  & 59.8 & 64.6 & 67.5 & 69.3  &  70.4 & 90.7  \\
\hline 
\multicolumn{7}{c}{SVHN}\\
\hline
Method$\backslash$Privacy & $\epsilon=2$ & $4$ & $5$ & $6$ & $8$ & $\infty$ \\
 \hline
Regular DP-SGD  & 41.2  & 67.5  & 78.4 & 81.9 & 84.9 & 92.3\\
 DP-SGD+MM  & 78.2 & 86.3 & 87.8 & 89.0 & 90.1  &  92.3 \\
\hline 
\hline 
\end{tabular}
\end{center}
\caption{\textbf{{Test Accuracy of} DP-SGD with/out ModelMix (MM) on training Resnet20}.}
\label{tab:MM_BC}
  \vspace{-0.1 in}
\end{table}

\begin{table}[t]
\begin{center}
\begin{tabular}{c r r r r r } 
\hline
\hline 
\multicolumn{6}{c}{CIFAR10}\\
\hline
Method$\backslash$Accuracy(\%) & 64 & 66 & 68 & 70 & 72 \\
 \hline
Regular DP-SGD  & $\epsilon=95.8$ & 108.3 & 118.2 & 140.6  & 182.9   \\
 DP-SGD+MM  & $\epsilon=4.7$ & 5.3 & 6.3  &  7.5 & 9.2  \\
\hline 
\multicolumn{6}{c}{SVHN}\\
\hline
Method$\backslash$Accuracy(\%) & $88$ & $88.5$ & $89$ & $89.5$ & $90.1$  \\
 \hline
Regular DP-SGD    & $\epsilon = 93.6$  & 101.3 & 112.4 & 128.9 & 144.5 \\
 DP-SGD+MM  & $\epsilon = 5.1 $ & 5.6 & 6.0 & 6.8  &  8 \\
\hline 
\hline 
\end{tabular}
\end{center}
\caption{\textbf{{Privacy Loss of} DP-SGD with/out ModelMix (MM) on training Resnet20}.}
\label{tab:MM_acc}
  \vspace{-0.1 in}
\end{table}

\subsection{Assistance with Public Data}
\label{sec:pub_data}
\noindent With access to additional public data, many elegant ideas have been proposed to significantly improve the performance of DP-SGD. In this subsection, we present results where we use ModelMix to further improve two representative works \cite{DanICLR2020} and \cite{yu_public2020}. \cite{DanICLR2020} shows a private transfer learning method on CIFAR10, where a SimCLR model \cite{chen2020simple} is first pretrained on unlabeled ImageNet and a linear model is then trained on the features extracted from the
penultimate layer of the SimCLR model using DP-SGD. With a privacy budget $(\epsilon=2,\delta=10^{-5})$, one can achieve $92.7\%$ on CIFAR10 \cite{DanICLR2020}. Similarly, we can apply ModelMix to improve this DP guarantee to $(\epsilon=0.64,\delta=10^{-5})$ with the same performance. On the other hand, if we assume the same privacy budget $(\epsilon=2,\delta=10^{-5})$ as \cite{DanICLR2020}, we can instead achieve $93.6\%$, close to the non-private optimal performance $94.3\%$. 

We also compare to \cite{yu_public2020}, which takes 2,000 ImageNet samples as public data and estimates a low-dimensional embedding of the private gradient when applying DP-SGD. Using the hyper-parameters suggested by \cite{yu_public2020}, we reproduce their experiments to privately train Resnet20 on CIFAR10, which achieves $73.2\%$ accuracy with an $(\epsilon=8,\delta=10^{-5})$ budget, and $79.1\%$ with $(\epsilon=111.2,\delta=10^{-5})$. With ModelMix, a sharpened tradeoff is produced, where with a budget $(\epsilon=2.9,\delta=10^{-5})$ and $(\epsilon=6.1,\delta=10^{-5})$, we achieve accuracy of $74.2\%$ and $79.1\%$, respectively. This is also close to the non-private optimal performance 82.3\% of \cite{yu_public2020}, with the hyper-parameters suggested.

\section{Conclusion and Prospects}
\label{sec:conclusion}
\noindent In this paper, we present a formal study on the privacy amplification from the trajectory entropy and the influence of gradient clipping in DP-SGD. 
We show fundamental improvements over DP-SGD, especially for deep learning, without assistance of other assumptions and additional data resources. 
This is a first step to consider the amplification from the potential entropy underlying the intermediate computation in DP-SGD and there are still many interesting directions for further generalization of ModelMix. 
For example, one may consider an adaptive selection of the envelope radius $\tau$ using techniques in \cite{adptive_clipping} or approximate it with public data after a projection to a low-rank subspace \cite{embedyu2021}. Moreover, the privacy analysis presented can also be generalized to quantify amplification from a large class of practical randomness used in deep learning, as explained in Section \ref{sec:beyond}; or to formalize the privacy guarantees of many heuristic privacy protections, such as Instahide \cite{instahide} and Datamix \cite{hansong_ECCV}. Another issue we pointed out, which is of more interest to practitioners, is the important role of sampling noise in clipped DP-SGD and the connection to the clipping threshold. The theory and the empirical study shown here could be meaningful to instruct further architecture-level improvement, for example the application of batch norm for the noise variation reduction.  We leave this to future work.

\small{
\bibliographystyle{unsrt}
\bibliography{ref.bib}
}

\appendices

\section{Proof of Theorem 3.1} \label{app:pr_convergence}
\noindent For a $\beta$-smooth function $F(w)$, we have the following fact \cite{lei2017non} that for any $w$ and $w'$,
{\small{
\begin{equation}
    \label{smooth_general}
   -\frac{\beta}{2} \|w-w' \|^2   \leq F(w) - F(w') - \langle \nabla F(w'), w-w' \rangle \leq \frac{\beta}{2} \|w-w' \|^2.
\end{equation}
}}\noindent
Equation (\ref{smooth_general}) will be constantly applied in the following proof. 
It is noted that $\mathbb{E}[G_k]=nq\nabla F(w_k)$ and we use $\eta'$ to denote $\eta \cdot (nq)$ in the following. 
Thus, $\mathbb{E}[\eta G_k] = \eta' \nabla F(w_k)$. 

For each $k \in [0:T-1]$, conditional on $w_{k}$ and $w_{k-1}$, with the updating rule defined in Algorithm \ref{alg: ModelMix}, we have
{\small{
\begin{equation}
\label{GD-convergence}
\begin{aligned}
& 
\mathbb{E}\big[ \|w_{k+1} - w^*\|^2\big]  \\
= & 
\mathbb{E}\big[ \| {\alpha}_{k+1} w_{k} + (1-{\alpha}_{k+1}) w_{k-1}-w^* - \eta (G_{k} + \Delta_{k+1}) + e_{\tau_{k+1}} \|^2\big] \\
= & 
\mathbb{E}\big[\| \alpha_{k+1} (w_{k}-w^*) + ({1}-{\alpha}_{k+1}) (w_{k-1}-w^*)\|^2 \\
&\quad\quad\quad
- 2\eta' \langle \frac{w_{k}+w_{k-1}}{2}-w^*, \nabla F(w_{k})) \rangle \\
&\quad\quad\quad
+ \frac{\eta'^2}{n^2q^2}\big(\mathbb{E}[\|G_{k}\|^2] + \mathbb{E}[\|\Delta_{k+1}\|^2]\big)  + \mathbb{E}[\|e_{\tau_{k+1}}\|^2]. 
\end{aligned}
\end{equation}
}}\noindent
Here, $e_{\tau_{k+1}}$ represents the additional bias caused by the modification which ensures the per coordinate distance between $w_{k}$ and $w_{k-1}$ is at least $\tau_{k+1}$, whose mean is zero. Therefore, the variance $\mathbb{E}[\|e_{\tau_{k+1}}\|^2]$ of $e_{\tau_{k+1}}$ is bounded by $d\tau^2/12$.  
Furthermore, we have that for each coordinate
{\small{
\begin{equation*}
\begin{aligned}
    &
    \|{\alpha}_{k+1}(j) (w_{k}(j)-w^*(j)) + (1-{\alpha}_{k+1}(j)) (w_{k-1}(j)-w^*(j))\|^2\\
    \leq 
    &
   \alpha_{k+1}(j)\| w_{k}(j)-w^*(j)\|^2 + (1-{\alpha}_{k+1}(j)) \| w_{k-1}(j)-w^*(j)\|^2,
\end{aligned}
\end{equation*}
}}\noindent
and thus,
{\small{
\begin{equation*}
\begin{aligned}
\mathbb{E}[\| {\alpha}_{k+1} & (w_{k}-w^*) + ({1}-{\alpha}_{k+1}) (w_{k-1}-w^*)\|^2] \\
\leq 
&
  \frac{\|w_{k}-w^* \|^2 + \|w_{k-1}-w^* \|^2}{2}.
\end{aligned}
\end{equation*}
}}\noindent
In addition, due to the convexity and the smoothness assumption,
{\small{
\begin{equation}
\begin{aligned}
&
\langle  \frac{w_{k}+w_{k-1}}{2}-w^*,  \nabla F(w_{k}) \rangle \\
= & \langle w_{k}-w^* + \frac{w_{k-1}-w_{k}}{2} ,  \nabla F(w_{k}) - \nabla F(w^*) \rangle \\
\geq 
&
F(w_{k}) - F(w^*) + \langle \frac{w_{k-1}-w_{k}}{2} ,  \nabla F(w_{k}) - \nabla F(w^*) \rangle\\
\geq & \frac{F(w_{k})+F(w_{k-1})}{2} - F(w^*) - \frac{\beta}{4}\|w_{k}-w_{k-1}\|^2. 
\end{aligned}
\label{cross_bound}
\end{equation}
}}\noindent
In (\ref{cross_bound}), we use (\ref{smooth_general}) and the following fact that $\langle w-w', \nabla F(w)-\nabla F(w') \rangle \geq F(w)-F(w')$ for any $w, w'$ and a convex function $F(\cdot)$. 
With the Lipschitz assumption, we know $\|G_k\|^2 \leq n^2L^2$.
Therefore,
{\small{
\begin{equation}
\begin{aligned}
    & \mathbb{E}\big[\frac{F(w_{k})+F(w_{k-1})}{2} - F(w^*) \big] \\
    \leq 
    &
    \frac{1}{2\eta'} \big( \frac{\|w_{k}-w^*\|^2+\|w_{k-1}-w^*\|^2}{2} - \|w_{k+1}-w^*\|^2 + \frac{d\tau_{k+1}^2}{12} \big) \\
    &
    + \frac{\eta'}{2n^2q^2} (n^2L^2 +  \mathbb{E}[\|\Delta_{k+1}\|^2]) + \frac{\beta}{4}\|w_{k}-w_{k-1}\|^2.
\end{aligned}
\label{GD-convergence-v3}
\end{equation}
}}\noindent
Summing up both sides of Equation (\ref{GD-convergence-v3}) from $k=0,1,...,T-1$, we take expectation across $w_{[-1:T]}$ and have 

{\small{
\begin{equation*}
\label{GD-convergence-3}
\begin{aligned}
    & 
    2 \mathbb{E} \big[ F(\frac{\sum_{k=1}^T w_{k-1}+w_{k-2}}{2T}) - F(w^*) \big] \\
    \leq 
    & 
    {2\over T}\cdot  \sum_{k=1}^T  \mathbb{E}\big[\frac{F(w_{k-1})+F(w_{k-2})}{2} - F(w^*) \big] \\
    \leq
    &
    \eta'^{-1} \big( \frac{\|w_{-1}-w^*\|^2+2\|w_{0}-w^*\|^2}{2T} +  \frac{\sum_{k=0}^{T-1}d\tau^2_{k+1}}{12T} \big) + \frac{\eta'L^2}{q^2} \\
    & + \frac{\eta'}{n^2q^2} ( \mathbb{E}[\sum_{k=0}^{T-1}\|\Delta_{k+1}\|^2/T]) + \frac{\beta\mathbb{E}[\sum_{k=0}^{T-1}\|w_{k}-w_{k-1}\|^2]}{2T}.\\
\end{aligned}
\end{equation*}
}}\noindent
It is noted that the updating rule can be written as,
{\small{
\begin{equation}
    \begin{aligned} 
    w_{k+1} - w_{k} 
    & = -(1-\alpha_{k+1}) (w_{k}-w_{k-1}) - \eta(G_k + \Delta_{k+1}) \\
    & = (1-\alpha_{k+1})(1-\alpha_{k})(w_{k-1}-w_{k-2})+ \\
    & ~~~ (1-\alpha_{k+1})\eta(G_{k-1} + \Delta_{k}) - \eta(G_k + \Delta_{k+1})
    \end{aligned}
\end{equation}
}}\noindent
By the recursion and the independce of different $\alpha_{k}$, we have a closed-form upper bound of $\| w_{k+1} - w_{k} \|^2$ conditional on the initialization $\|w_{0}-w_{-1}\|^2$, where 
{\small{
\begin{equation}
\begin{aligned}
    \label{difference recursion}
    & ~~~~\mathbb{E}[\|w_{k+1} - w_{k}\|^2] \\
    &\leq(\frac{1}{3})^{k+1}\|w_{0}-w_{-1}\|^2  + \frac{11}{2}\eta^2(n^2L^2+\mathbb{E}[\|\Delta\|^2]) \\
    & ~~~~+ 2(k+1)(1/2)^{k+1}\eta\|w_{0}-w_{-1}\|(nL+\mathbb{E}[\|\Delta\|]).  
\end{aligned}
\end{equation}
}}\noindent
Here, in (\ref{difference recursion}), we use the following fact that
{\small{
\[
\mathbb{E}[\|\sum_{i=1}^{k} \prod_{j=1}^{i-1} \alpha_{j}\|^2] = \sum_{i=1}^k (\frac{1}{3})^i + 2\sum_{i=1}^k \sum_{j=1}^{i} (\frac{1}{2})^j(\frac{1}{3})^{i-j} \leq \frac{3}{2} + 4 = \frac{11}{2},
\]
}}\noindent
Now putting all the above together, we have that $\mathbb{E} \big[  F(\frac{\sum_{k=1}^T w_{k-1}+w_{k-2}}{2T}) - F(w^*) \big]$ is upper bounded by 
\[
\small{
    \begin{aligned} 
      &\frac{\|w_{-1}-w^*\|^2+2\|w_{0}-w^*\|^2 + \sum_{k=1}^Td\tau^2_{i}/12}{2\gamma \sqrt{T}} \\
      & + \frac{\gamma(L^2/q^2+\mathbb{E}[\|\Delta\|^2]/(n^2q^2))}{2\sqrt{T}} + \frac{\beta\big( \frac{11}{2}\eta^2(n^2L^2+\mathbb{E}[\|\Delta\|^2]) \big)}{T}\\
     & + \frac{\beta \big(\frac{3}{2}\|w_0 - w_{-1}\|^2 + 4\eta \|w_{0}-w_{-1}\|(nL+\mathbb{E}[\|\Delta\|])}{4T} \\
     & = \frac{\|w_{-1}-w^*\|^2+2\|w_{0}-w^*\|^2 + \sum_{k=1}^T d \tau^2_{i}/12}{2\gamma \sqrt{T}}\\
     & ~~~~+  \frac{3\beta\|w_0 - w_{-1}\|^2+ 11\gamma^2(\frac{L^2}{q^2}+\mathbb{E}[\frac{\|\Delta\|^2}{n^2q^2}])}{8T} \\ 
     & ~~~~+ \frac{\beta\gamma \|w_{0}-w_{-1}\|(\frac{L}{q}+\mathbb{E}[\frac{\|\Delta\|}{nq}])}{T^{3/2}}+ \frac{\gamma(L^2+\mathbb{E}[\|\Delta\|^2])}{2\sqrt{T}}.
    \end{aligned}}
\]

\section{Proof of Theorem 3.2} \label{app:pr_convergence_MM}
\noindent We first prove the following theorem, which analyzes the standard clipped DP-SGD without ModelMix. 
\begin{thm}[Convergence of Clipped DP-SGD]
\label{thm:dp_sgd}
Suppose the objective loss function $F(w)$ is $\beta$-smooth and satisfies Assumption \ref{assp:subexp}, then there exists some constant $\psi > 0$ such that when the clipping threshold $c$ satisfies
{\small{
\begin{equation}
    \label{c_req}
    c \geq \max \{4\kappa\log(10), -\psi\kappa\log(\kappa)\log(\frac{\sqrt{d \log(1/\delta)}}{n\epsilon })\},
\end{equation}
}}\noindent
then the convergence rate of per-sample clipped DP-SGD with an $(\epsilon, \delta)$-DP guarantee satisfies
{\small{
\begin{equation}
\begin{aligned}
\label{batch_main}
     \mathbb{E} \Big[ &\frac{\sum_{k=0}^{T-1}  \min\big\{ 9/20 \cdot  \|\nabla F(w_k)\|^2, c/20 \cdot \| \nabla F(w_k)\| \big \}}{T}\Big]  \\
     &\leq (\frac{v}{2} +\frac{5}{2})\cdot \frac{c\sqrt{\mathcal{R}_{F}\beta d \log(1/\delta)}}{n\epsilon},
\end{aligned}
\end{equation}
}}\noindent
where $\mathcal{R}_{F} = \sup_{w}F(w)- \inf_{w} F(w)$ and $v$ is some constant determined by noise Mechanism. 
\end{thm}

\begin{proof}
At the $k$-th iteration, based on the smoothness assumption, we have 
{\small{
\begin{equation}\label{equ:smooth}
    F(w_{k+1}) \leq F(w_{k}) + \langle \nabla F(w_{k}), w_{k+1}-w_{k}\rangle + {\beta\over 2} \cdot \|w_{k+1}-w_{k} \|^2.   
\end{equation}
}}\noindent
We consider the following equivalent updating rule, where the step size $\eta$ is scaled by $1/(nq)$ in  (\ref{dpsgd_with_cliping}), 
{\small{
\[
    w_{k+1} - w_{k} = \eta \cdot \Big( \frac{1}{nq} \cdot (\sum_{j=1}^{B_k} \mathsf{CP}(g^j_k, c) + \Delta_{k+1}) \Big).
\]
}}\noindent
Here $g^j_k = \nabla f(w_k, x_{(j)}, y_{(j)})$, where $(x_{(j)}, y_{(j)})$ is the $j$-th sample selected in the minibatch $S_k$. For simplicity, we use ${g}_{k}$ to denote the random variable 
{\small{
\[
g_k =  \nabla f(w_k, x, y),
\]
}}\noindent
where $(x, y)$ is randomly selected from the dataset $\mathcal{D}$. Thus, we have the following observation that
{\small{
$$ \mathbb{E}[\mathsf{CP}(g_k, c)] = \mathbb{E}[\frac{1}{nq} \cdot \sum_{j=1}^{B_k} \mathsf{CP}(g^j_k, c)],$$
}}\noindent
due to the i.i.d. sampling of rate $q$.  Thus, with expectation we have 
{\small{
\[
\mathbb{E}[w_{k + 1} - w_k] 
= \eta \Big( \frac{1}{B_k} \sum_{j=1}^{B_k} \mathbb{E}[\mathsf{CP}(g^j_k, c)] + \mathbb{E}[\Delta_{k+1}] \Big)
= \eta \mathbb{E}[\mathsf{CP}(g_k, c)],
\]
\[
\mathbb{E}[\|w_{k+1}-w_{k} \|^2]
\leq
\eta^2 \cdot (c^2/q^2 + \mathbb{E}[\|\Delta_{k+1}\|^2]/(n^2q^2)).
\]
}}\noindent
Taking this into Equation (\ref{equ:smooth}), we have
{\small{
\begin{equation}
\label{smooth_basic}
\begin{aligned}
     \mathbb{E}[F(w_{k+1})] 
    & \leq F(w_k) - \mathbb{E}[\langle \nabla F(w_{k}), w_{k + 1} - w_k \rangle] \\
    &   \quad + {\beta\over 2} \cdot \mathbb{E}[\|w_{k+1}-w_{k} \|^2] \\
    & \leq F(w_k) - \eta \cdot\mathbb{E}[\langle \nabla F(w_{k}), \mathsf{CP}(g_{k}, c) \rangle] \\
    &   \quad + {\beta \eta^2\over 2} \cdot (c^2/q^2 + \mathbb{E}[\|\Delta_{k+1}\|^2/(n^2q^2)]).
\end{aligned}
\end{equation}
}}\noindent
Let  $\gamma_k = \min \{1, {c\over \|g_k\|}\}$, i.e., $\mathsf{CP}(g_{k}, c)= \gamma_k g_k$. 
From previous works \cite{CCS2016} on the standard Gaussian DP mechanism, we know that the variance of the noise $\Delta_k$ is bounded by $O(q^2dT\log(1/\delta) / \epsilon^2)$.
Therefore, there must exist some constant $v$ such that we can rewrite Equation (\ref{smooth_basic}) as
{\small{
\begin{equation}
\label{smooth_basic_1}
\begin{aligned}
     \mathbb{E}[F(w_{k+1})] 
     & \le F(w_k) - \eta \cdot \mathbb{E}[\langle \nabla F(w_{k}), \gamma_k g_{k} \rangle] \\
     & \quad + {\beta \eta^2\over 2} \cdot (\mathbb{E}[\|\gamma_k  g_k \|^2] + \mathbb{E}[\|\Delta_{k+1}]\|^2) \\
     & \leq  F(w_k) -  \eta\cdot \mathbb{E}[\langle \nabla F(w_{k}), \gamma_k g_{k} \rangle] \\
     & \quad + {\beta \eta^2\over 2}\cdot (c^2/q^2 + v\cdot \frac{dT\log(1/\delta)}{n^2\epsilon^2}).
\end{aligned}
\end{equation}
}}\noindent
In the following Lemma, we lower bound $\mathbb{E}[\langle \nabla F(w_{k}), \gamma_k  g_{k} \rangle]$. We use $\xi_k = g_k - \nabla F(x_k)$ to denote the stochastic gradient noise. 

\begin{lemma}
For any $p > 2$ and $c_0 > 0$ such that $c= p\cdot c_0$, it holds that
{\small{
$$\mathbb{E}[\langle \nabla F(w_{k}), \gamma_k  g_{k} \rangle] \geq \min \{a_1\|\nabla F(x_k) \|^2 - a_2,  a_3\|\nabla F(x_k)\| \},$$
}}\noindent
where $a_1$, $a_2$ and $a_3$ are constants define as below
{\small{
\[
a_1 = \frac{(1-e^{-{c_0\over \kappa}})}{2}, 
a_2 = \Big(\frac{(c_0+\kappa)\cdot e^{-{c_0\over\kappa}}}{1-e^{-{c_0\over\kappa}}}\Big)^2,
\]
\[
a_3 = c\Big( (1-e^{-{c_0\over\kappa}})(1-\frac{1}{p-1}-\frac{1}{p-2})- e^{-{c_0\over\kappa}} \Big).
\]
}}\noindent
\end{lemma}

\begin{proof}
The proof contains two parts. 

(1). First, we consider the case when the gradient is small where $\|\nabla F(x_k)\| \leq (p-1)c_0$.
{\small{
\begin{equation}
\label{small_gradient}
\begin{aligned}
      & 
      \mathbb{E}[\langle \nabla F(w_{k}), \gamma_k  g_{k} \rangle] \\
      & = \mathbb{E}[\langle \nabla F(w_{k}), \gamma_k (\nabla F(w_{k})+\xi_k ) \rangle] \\
      & = \|\nabla F(w_{k})\|^2 \cdot \mathbb{E}[\gamma_k] + \mathbb{E}[\langle \nabla F(w_k) ,\xi_k\rangle \cdot (1-\gamma_k)] \\
      & \geq \|\nabla F(w_{k})\|^2 \Pr(\|g_k\| \leq c) - \| \nabla F(w_k)\|\mathbb{E}[\|\xi_k\|\cdot \bm{1}_{\|g_k\|>c}] \\
      & \geq \|\nabla F(w_{k})\|^2 \Pr(\|\xi_k\| \leq c_0) - \| \nabla F(w_k)\|\mathbb{E}[\|\xi_k\|\cdot \bm{1}_{\|g_k\|>c}]
\end{aligned}
\end{equation}
}}\noindent
Here, the second equality is because $\mathbb{E}[\langle \nabla F(w_k), \xi_k]=0$. In the third inequality, we use the following facts. 
It is not hard to see that 
{\small{
\[
\|\nabla F(w_{k})\|^2 \cdot \mathbb{E}[\gamma_k] \geq \|\nabla F(w_{k})\|^2 \cdot \Pr(\|g_k\|\leq c).
\]
}}\noindent
Since $\|F(w_k)\| \leq (p-1)c_0$ is assumed, we have the probability of clipping 
{\small{
\[
\Pr(\|g_k\|\leq c) \geq \Pr(\|\xi_k\|\leq c_0).
\]
}}\noindent
As for the second term, when $\|g_k\|\leq c$, i.e., no clipping is performed.
In this case, $\gamma_k = 1$ and we have
{\small{
\[
\mathbb{E}[\langle \nabla F(w_k) ,\xi_k\rangle \cdot (1-\gamma_k) \cdot \bm{1}_{\|g_k\|\leq c}] = 0.
\]
}}\noindent
Thus, we only need to consider $\mathbb{E}[\langle \nabla F(w_k) ,\xi_k\rangle \cdot (1-\gamma_k) \cdot \bm{1}_{\|g_k\|\leq c}]$ when clipping is performed. 
By Cauchy-Schwartz inequality and $1-\gamma_k \leq 1$, we have the bound in Equation (\ref{small_gradient}). Now, we turn to upper bound $\mathbb{E}[\|\xi_k\|\cdot \bm{1}_{\|g_k\|>c}] $. By the {\em Integrated Tail Probability Expectation Formula}, we can derive the following variant,
{\small{
\begin{equation*}
    \begin{aligned}
    \mathbb{E}[\|\xi_k\|\bm{1}_{\|g_k\|>c}] 
    & \leq \mathbb{E}[\|\xi_k\|\bm{1}_{\|\xi_k\|>c_0}] \\
    & = c_0(1-\Pr(\|\xi_k\|<c_0)) + \int_{c_0}^{+\infty} \Pr(\|\xi_k\|>t) dt  \\
    & \leq c_0 e^{-c_0/\kappa} + \int_{c_0}^{+\infty} e^{-t/\kappa} dt  = (c_0 + \kappa) e^{-c_0/\kappa}.
    \end{aligned}
\end{equation*}
}}\noindent
Putting things together, we have 
{\small{
\begin{equation*}
\begin{aligned}
   &
   \mathbb{E}[\langle \nabla F(w_{k}), \gamma_k  g_{k} \rangle] \\
   & \geq (1-e^{-c_0/\kappa})\|\nabla F(w_k)\|^2 -(c_0 + \kappa) e^{-c_0/\kappa} \|\nabla F(x_k)\| \\
   & \geq \frac{(1-e^{-c_0/\kappa})}{2}\|\nabla F(w_k)\|^2  - (\frac{(c_0+\kappa)e^{-c_0/\kappa}}{1-e^{-c_0/\kappa}})^2.
\end{aligned}
\end{equation*}
}}\noindent
Here, the second inequality we use the following fact: $ax^2-bx \geq (a/2) x^2 -(b/a)^2,$ for any positive constants $a$. 

(2). Second, we consider the case when the gradient is larger  where $\|\nabla F(x_k) \|>(p-1)c_0$.
{\small{
\begin{equation}
\label{large_gradient}
\small{
\begin{aligned}
      & 
      \mathbb{E}[\langle \nabla F(w_{k}), \gamma_k  g_{k} \rangle] \\
      = & 
      \|\nabla F(w_{k})\|^2 \cdot \mathbb{E}[\gamma_k] + \mathbb{E}[\gamma_k\langle \nabla F(w_k) ,\xi_k\rangle] \\
      \geq &
      \|\nabla F(w_{k})\|^2 \mathbb{E}[\gamma_k  \cdot \bm{1}_{\|\xi_k\|\leq c_0}] - \| \nabla F(w_k)\|\mathbb{E}[\gamma_k \|\xi_k\|\cdot \bm{1}_{\|\xi_k\|\leq c_0}] \\
      &
      + \mathbb{E}[ \langle \nabla F(w_k),\gamma_k g_k\rangle\cdot \bm{1}_{\|\xi_k\|>c_0}] \\
      \geq &
      c \|\nabla F(w_k)\|\Pr(\|\xi_k\| \leq c_0) \cdot \big ( \frac{\| \nabla F(w_k)\| }{\| \nabla F(w_k)\|+c_0} - \frac{c_0}{\|\nabla F(w_k)\|-c_0}  \big ) \\
      &
      + \mathbb{E}[ \langle \nabla F(w_k),\gamma_k g_k\rangle\cdot \bm{1}_{\|\xi_k\|>c_0}]. \\
\end{aligned}
}
\end{equation}
}}\noindent
Here, the second inequality is because the term
{\small{
\[
\big ( \frac{\| \nabla F(w_k)\|}{\|\nabla F(w_k)\| + \|\xi_k\|} - \frac{\|\xi_k \|}{\|\nabla F(w_k)\|-\|\xi_k\|}  \big )
\]
}}\noindent
reaches the minimal when  $\|x_{i_k}\|=c_0$ for $\|x_{i_k}\| \in [0,c_0]$.
On one hand, 
{\small{
$$  \frac{\| \nabla F(w_k)\| }{\| \nabla F(w_k)\|+c_0} - \frac{c_0}{\|\nabla F(w_k)\|-c_0}  \geq  \frac{(p-1)}{p} - \frac{1}{p-2}.$$
}}\noindent
Moreover, we have 
{\small{
$$ \mathbb{E}[ \langle \nabla F(w_k),\gamma_k g_k\rangle\cdot \bm{1}_{\|\xi_k\|>c_0}] \geq -c\| \nabla F(w_k)\| \Pr(\|\xi_k\|>c_0),$$
}}\noindent
since $\|\gamma_k g_k\|<c$.
Thus, with the assumption $\Pr(\|\xi_k\|> c_0) \leq e^{-c_0/\kappa}$, we have 
{\small{
\begin{equation*}
\begin{aligned}
   &\mathbb{E}[\langle \nabla F(w_{k}), \gamma_k  g_{k} \rangle]  \\
   \geq & 
   c\|\nabla F(w_{k})\|\big( (1-e^{-c_0/\kappa})(1-\frac{1}{p-1}-\frac{1}{p-2})- e^{-c_0/\kappa} \big).
\end{aligned}
\end{equation*}
}}\noindent
\end{proof}

Now, we pick $p=4$, and we have the following corollary, i.e.,
{\small{
\begin{equation*}
\mathbb{E}[\langle \nabla F(w_{k}), \gamma_k  g_{k} \rangle] + a_2 \geq \min \{a_1\|\nabla F(x_k) \|^2,  a_3\|\nabla F(x_k)\| \},
\end{equation*}
}}\noindent
where
{\small{
\[
a_1 = \frac{(1-e^{-c/(4\kappa)})}{2},
a_2 = (\frac{(c/4+\kappa)e^{-c/(4\kappa)}}{1-e^{-c/(4\kappa)}})^2
\]
}}\noindent
and
{\small{
\[
a_3 = c\big(1/6 -7/6 \cdot e^{-c/(4\kappa)}) \big).
\]
}}\noindent
Now, we return to Equation (\ref{smooth_basic_1}) and sum up across $k=0,1,...,T-1$, and we have
{\small{
\begin{equation}
\label{dp_tradeoff}
\begin{aligned}
    &
    \mathbb{E}[ \frac{\sum_{k=0}^{T-1} \min\big\{ a_1 \|\nabla F(w_k)\|^2, a_3\| \nabla F(w_k)\| \big \}}{T}] \\
    & \leq \frac{F(w_0)-F(w_T)}{T\eta} + \frac{\beta \eta}{2} \big( \frac{c^2}{q^2} + v\frac{dT\log(1/\delta)}{n^2 \epsilon^2} \big) + a_2 \\
    & = (\frac{v}{2} +\frac{3}{2})\frac{c\sqrt{\mathcal{R}_{F}\beta d \log(1/\delta)}}{n\epsilon q^2} + a_2. 
\end{aligned}
\end{equation}
}}\noindent
Here, we set $\eta = \frac{\sqrt{\mathcal{R}_{F}d \log(1/\delta)}}{n \epsilon (c/q) \sqrt{\beta}}$ and $T = \frac{(n\epsilon)^2}{d\log(1/\delta)}$. Finally, we want to pick $c$ to ensure $a_2$ is at most in the same order of $O(\frac{\sqrt{d\log(1/\delta)}}{n\epsilon})$, while $a_1$ and $a_3$ are positive. 
To this end, we select $c$ such that
{\small{
$$ c \geq  4\kappa\log(10), 
\frac{c}{2\kappa} - \log\frac{(c/4+\kappa)^2}{c} \geq 0.4 - \log(\frac{\sqrt{\mathcal{R}_{F}\beta d \log(1/\delta)}}{n\epsilon}), $$
}}\noindent
then the right hand of (\ref{dp_tradeoff}) is further bounded by $(\frac{v}{2} +\frac{5}{2})\frac{c\sqrt{\mathcal{R}_{F}\beta d \log(1/\delta)}}{n\epsilon}$ while $a_1 = 9/20$ and $a_3 = c/20.$
\end{proof}

We still consider the case where the step size $\eta$ is scaled by a factor $1/(nq)$ for simplicity. 
For each $k \in [1:T]$, from the updating rule in Algorithm \ref{alg: ModelMix}, we have that conditional on $w_{k}$ and $w_{k-1}$, 
{\small{
\begin{equation}
\label{GD-convergence-1}
\footnotesize{
\begin{aligned}
& \mathbb{E}\big[ F(w_{k+1}) \big]  \\
& \leq \mathbb{E}\big [F(w_{k}) + \langle \nabla F(w^{k}), w^{k+1}-w^{k}\rangle + \frac{\beta}{2}\|w^{k+1} - w^{k} \|^2 \big] \\
& = F(w_{k}) + \langle \nabla F(w_k), \frac{w_{k-1}-w_{k}}{2} - \eta \mathbb{E}[\mathsf{CP}(g_k, c)] \rangle \\
& + \frac{\beta}{2} \mathbb{E}\big[\|(1-\alpha_{k+1})(w_{k-1}-w_{k}) - \eta/(np)(G_k  + \Delta_{k+1}) + e_{\tau_{k+1}} \|^2 \big] \\
& \leq  F(w_{k}) + \frac{\langle \nabla F(w_{k}),w_{k-1}-w_{k}  \rangle }{2} -\eta \langle \nabla F(w_k), \mathbb{E}[\mathsf{CP}(g_k, c)]  \rangle \\ 
& \quad + \beta\big( \frac{\|w_{k}-w_{k-1} \|^2}{3} 
+ \eta^2(c^2/q^2 + \|\Delta_{k+1}\|^2/(n^2q^2))+\|e_{\tau_{k+1}}\|^2 ] \big) . 
\end{aligned}
}
\end{equation}
}}\noindent
Here, we use the fact that $\mathbb{E}(1-\alpha)^2 = 1/3$ for a random $\alpha \in (0,1)$ and the AM-GM inequality. Still, $e_{\tau_{k+1}}$ represents the additional bias caused by the modification which ensures the per coordinate distance between $w_{k}$ and $w_{k-1}$ is at least $\tau_{k+1}$. We have  that the variance $\mathbb{E}[\|e_{\tau_{k+1}}\|^2]$ of $e_{\tau}$ is bounded by $d\tau_{k+1}^2/12$. 

Now we apply (\ref{smooth_general}) again on the term $\frac{\langle \nabla F(w^{k}),w_{k}-w_{k-1}  \rangle }{2}$, we have that 
{\small{
\begin{equation}
   \frac{\langle \nabla F(w^{k}),w_{k-1}-w_{k}  \rangle }{2} \leq \frac{1}{2} \big ( F(w_{k-1})- F(w_{k}) + \frac{\beta}{2}\|w_{k} - w_{k-1}\|^2\big).
   \label{reverse_smooth}
\end{equation}
}}\noindent
Now, substitute (\ref{reverse_smooth}) back to (\ref{GD-convergence-1}), we have
{\small{
\begin{equation}
    \begin{aligned}
     & \mathbb{E}[F(w_{k+1})]\leq \frac{F(w_k)+F(w_{k-1})}{2}  -\eta \langle \nabla F(w_k), \mathbb{E}[\mathsf{CP}(g_k, c)]  \rangle \\ 
     & \quad + \beta\big(\frac{7 \| w_k-w_{k-1}\|^2 + \tau^2_{k+1}}{12} + \frac{\eta^2}{q^2} \cdot  (c^2+ \mathbb{E}[\|\Delta_{k+1}\|^2]/n^2) \big).
    \end{aligned}
\label{GD-convergence-2}
\end{equation}
}}\noindent
In the above inequality, we use the fact that $\|\mathsf{CP}(g^j_k,c)\| \leq c$ due to the clipping. We sum up both sides of (\ref{GD-convergence-2}) and have 
{\small{
\begin{equation}
    \begin{aligned}
        & \mathbb{E}\big[\frac{\eta \sum_{k=0}^{T-1} \langle \nabla F(w_k), \mathsf{CP}(g_k,c) \rangle}{T}\big]  \leq \frac{3\mathcal{R}_{F}}{2T}+\beta \eta^2(\frac{c^2}{q^2}+\frac{\mathbb{E}[\|\Delta\|^2}{(nq)^2})\\
        & + \frac{\beta\sum_{k=0}^{T-1}\big(
        d\tau^2_{k+1} + 7\|w_{k}-w_{k-1} \|^2 \big)}{12T}  
    \end{aligned}
    \label{MM_CSGD-1}
\end{equation}
}}\noindent
where $\mathcal{R}_{F} = \sup_{w}F(w)- \inf_{w} F(w)$. 
With a similar reasoning as  (\ref{difference recursion}) in Appendix \ref{app:pr_convergence}, we have that given $w_{0}$ and $w_{-1}$,
{\small{
\begin{equation}
\begin{aligned}
 \sum_{k=0}^{T-1} \| w_k-w_{k-1}\|^2 & \leq \frac{3}{2} \|w_0-w_{-1}\|^2 + 4\eta c\|w_0-w_{-1} \| \\
 & + \frac{11T}{2}\eta^2(c^2/q^2+\mathbb{E}[\|\Delta\|^2]/(nq)^2). 
\end{aligned}
   \label{additional_term_smooth}
\end{equation}
}}\noindent

The rest of the analysis for the term  $\mathbb{E}\big[\frac{\eta \sum_{k=0}^{T-1} \langle \nabla F(w_k), \mathsf{CP}(g_k,c) \rangle}{T}\big]$ is the same as the proof of Theorem \ref{thm:dp_sgd} where virtually we handle a function whose smooth parameter becomes $\frac{101}{12}\beta$. Therefore, we still select 
$\eta = \frac{\sqrt{\mathcal{R}_{F}d \log(1/\delta)}}{n \epsilon (c/q) \sqrt{\frac{101}{12}\beta}}$ and $T = \frac{(n\epsilon)^2}{d\log(1/\delta)}$, and substitute them into (\ref{additional_term_smooth}), and then the right hand of (\ref{batch_main}) in the 
ModelMix case becomes 
{\small{
\begin{equation}
\small{
    \begin{aligned}
    (\frac{v}{2} +\frac{5}{2})&\frac{c\sqrt{\mathcal{R}_{F}\frac{101}{12}\beta d \log(1/\delta)}}{n\epsilon} + \frac{28c\beta{d\log(1/\delta)}\|w_0-w_{-1} \| }{12 q(n\epsilon)^2} \\
    &  +  \frac{c d\log(1/\delta)\sqrt{\frac{101}{12}}\beta^{3/2}}{qn\epsilon\sqrt{\mathcal{R}_{F}}} \big(\frac{\sum_{k=1}^T d\tau^2_{k}}{12} + \frac{21\|w_0-w_{-1}\|^2}{24}\big) .
    \end{aligned}
}
\end{equation}
}}\noindent

\section{Proof of Theorem \ref{thm:RDP} and Corollary \ref{cor: l_infty}}
\label{app:proof_RDP}
\noindent For simplicity, we normalize $\bar{\tau} = \tau/\eta$. We consider the following equivalent aggregation model $\mathcal{M}$. Let $\mathcal{D}=(a_1,a_2,...,a_n)$ and $\mathcal{D}'=(a_1,a_2,...,a_{n-1})$ be two adjacent datasets. $\mathcal{M}(\mathcal{D})$ ($\mathcal{M}(\mathcal{D}')$) implements as follows. First, we apply i.i.d. sampling of rate $q$ to generate a subset $S$ from $\mathcal{D}$ $(\mathcal{D}')$ and output its sum $\Sigma(S) = \sum_{a_i \in S} a_i$ perturbed by a sum of Gaussian and independent ModelMix, i.e., $\mathcal{U}[-\frac{\bar{\tau}}{2}, \frac{\bar{\tau}}{2}]^d \times \mathcal{N}(\bm{0}, \sigma^2 \cdot \bm{I}_{d}),$ where we denote this kind of mixture distribution as $\mathcal{MG}^d_{\bar{\tau},\sigma}(\bm{0})$.  In the above model, $a_i = \mathsf{CP}(\nabla f(w,x_i,y_i), c)$ represents the individual gradient of each datapoint clipped by c for any $i$, where we model the distribution of DP-SGD with ModelMix procedure equivalently as 
{\small{
\begin{equation}
 \mathcal{U}^{d}[-\frac{\tau}{2},\frac{\tau}{2}] + \eta(\sum_{ a_i \in S} a_i + \Delta).
 \label{agg_model}
\end{equation}
}}\noindent

Now, we can rewrite the distributions of $\mathcal{M}(\mathcal{D})$ and $\mathcal{M}(\mathcal{D}')$ as follows,
{\small{
$$ \mathcal{M}(\mathcal{D}') = \sum_{S} p_S \mathcal{MG}^d_{\bar{\tau},\sigma}(\Sigma(S)),$$
}}\noindent
for any $S \subset \mathcal{D} \cap \mathcal{D}'$ and $P_S$ represents the probability that $S$ gets sampled from $\mathcal{D}'$ with i.i.d. sampling of rate $q$. Correspondingly, 
{\small{$$  \mathcal{M}(\mathcal{D}) = \sum_{S} p_S \big( (1-q)\mathcal{MG}^d_{\bar{\tau},\sigma}(\Sigma(S)) + q \mathcal{MG}^d_{\bar{\tau},\sigma}(\Sigma(S)+ a_n) \big). $$}}\noindent
Therefore, when we measure the $\alpha$-Rényi divergence, we have 
{\small{
\begin{equation}
\label{Rényi-shift}
\begin{aligned}
 & \mathsf{D}_{\alpha}(\mathcal{M}(\mathcal{D})\|\mathcal{M}(\mathcal{D}'))\\
 &\leq \max_{S} \mathsf{D}_{\alpha} \big(\mathcal{MG}^d_{\bar{\tau},\sigma}(\Sigma(S))\|  (1-q)\mathcal{MG}^d_{\bar{\tau},\sigma}(\Sigma(S)) + q \mathcal{MG}^d_{\bar{\tau},\sigma}(\Sigma(S, a_n) ) \big) \\
 & \leq \max_{v,\|v\|\leq c} \mathsf{D}_{\alpha} \big(\mathcal{MG}^d_{\bar{\tau},\sigma}(\bm{0})\|  (1-q)\mathcal{MG}^d_{\bar{\tau},\sigma}(\bm{0}) + q \mathcal{MG}^d_{\bar{\tau},\sigma}(v) \big).
\end{aligned}
\end{equation}
}}\noindent
In the first and the second inequality of (\ref{Rényi-shift}), we use the quasi-convexity and translation invariance properties of Rényi divergence, respectively. Thus, in the second inequality we subtract $\Sigma(S)$ on both side and we know $\|v = \Sigma(S \cup a_n) - \Sigma(S) = a_n\|$ is upper bounded by $c$. 

To proceed, we first use the result (Theorem 5) in \cite{mironov2019r}, which suggests that $\mathsf{D}_{\alpha}(\mathcal{M}(\mathcal{D})\|\mathcal{M}(\mathcal{D}')) \geq \mathsf{D}_{\alpha}(\mathcal{M}(\mathcal{D}')\|\mathcal{M}(\mathcal{D}))$ and it suffices to consider $\mathsf{D}_{\alpha}(\mathcal{M}(\mathcal{D})\|\mathcal{M}(\mathcal{D}'))$ to derive the RDP bound in the following. It is noted that  

\[
\small{
\begin{aligned}
& (1-\alpha)\mathsf{D}_{\alpha}(\mathcal{M}(\mathcal{D})\|\mathcal{M}(\mathcal{D}')) \\
& \leq \log \max_{v} \mathbb{E}_{o \sim \mathcal{MG}^d_{\tau,\sigma}(\bm{0})} \big( (1-q) + q\frac{\mathcal{MG}^d_{\bar{\tau},\sigma}(v)(o)}{\mathcal{MG}^d_{\bar{\tau},\sigma}(\bm{0})(o)} \big)^{\alpha}\\
& = \log \max_{v} \big\{\sum_{k=0}^{\alpha} \tbinom{\alpha}{k}(1-q)^{\alpha-k}q^k \mathbb{E}_{o \sim \mathcal{MG}^d_{\bar{\tau},\sigma}(\bm{0})}[\big(\frac{\mathcal{MG}^d_{\bar{\tau},\sigma}(v)(o)}{\mathcal{MG}^d_{\bar{\tau},\sigma}(\bm{0})(o)}\big)^{k}] \big\}.
\end{aligned}
}
\]
Thus, in the following, it is equivalent to consider $\mathcal{A}_k =\mathbb{E}_{o \sim \mathcal{MG}^d_{\bar{\tau},\sigma}(\bm{0})}[\big(\frac{\mathcal{MG}^d_{\tau,\sigma}(v)(o)}{\mathcal{MG}^d_{\bar{\tau},\sigma}(\bm{0})(o)}\big)^{k}],$ which has a semi-closed form
{\small{
\begin{equation}\label{moment_mm}
\small{
\begin{aligned}
\mathcal{A}_k & = \max_{v} \sum_{j=1}^d \log \int  \frac{\big(\mathbb{P}(\mathcal{MG}^1_{\bar{\tau},\sigma}(v_j)={o}_j)\big)^{k}}{\big(\mathbb{P}(\mathcal{MG}^1_{\bar{\tau},\sigma}(0) ={o}_j)\big)^{k-1}} d o_j = \max_{v} \sum_{j=1}^d h(v_j). 
\end{aligned}
}
\end{equation}
}}\noindent
Here, $v_j$ and $o_j$ represents the $j$-th coordinated of $v$ and $o$, and we use the fact that the distribution of each coordinate of $\mathcal{MG}^d_{\bar{\tau},\sigma}(v)$ or $\mathcal{MG}^d_{\bar{\tau},\sigma}(\bm{0})$ is independent and thus its density function is a product of $\mathcal{MG}^1_{\bar{\tau},\sigma}(v_j)$ or $\mathcal{MG}^1_{\bar{\tau},\sigma}(0)$. To be specific, we have that $h(v_j)$ can be expressed as 
{\small{
\[
h(v_j) = \log \int_{o}\frac{\big(\int_{-\frac{\bar{\tau}}{2}}^{\frac{\bar{\tau}}{2}}\frac{e^{-(o-\alpha-v_j)^2/(2\sigma^2)}}{\bar{\tau}\sqrt{2\pi\sigma^2}}d\alpha \big)^{k}}{\big(\int_{-\frac{\bar{\tau}}{2}}^{\frac{\bar{\tau}}{2}}\frac{e^{-(o-a)^2/(2\sigma^2)}}{ \bar{\tau}\sqrt{2\pi\sigma^2}} d\alpha\big)^{k-1}} do.
\]
}}\noindent
Then, to characterize the worst-case divergence, it is equivalent to considering the following 
{\small{
\begin{equation}\label{opt_h}
 \max_{v} \sum_{j=1}^d h(v_j)
\text{~~ where ~~}
\sum_{j=1}^d v^2_j = \|v\|^2 \leq c^2.
\end{equation}
}}\noindent
Due to the symmetric property of $h(\cdot)$, i.e., $h(v_j) = h(-v_j)$, we only need to consider the case where $v_j \geq 0$. 
Define $g(x) = h(\sqrt{x})$, then solving Equation (\ref{opt_h}) is equivalent to solving
{\small{
\begin{equation}\label{opt_h2}
\max \sum_{j=1}^d g(\tilde{v}_j)
\text{~~ where ~~}
\forall j, \tilde{v}_j \in [0, c] 
\text{~~ and ~~}
\sum_{j=1}^d \tilde{v}_j \leq c^2.
\end{equation}
}}\noindent
With some simple calculation on the second-order derivative of $g(\tilde{v})$, $g(\tilde{v})$ is a convex function and therefore the maximal of $\sum_{j=1}^d g(\tilde{v}_j)$ over a simplex constraint must be achieved at the vertices. 
Therefore, the maximum is achieved when $\tilde{v}$ is some one-hot vector, and we only need to consider the one-dimensional case by selecting $v = (c, 0, 0,...,0)$. As a straightforward proof of Corollary \ref{cor: l_infty}, when we further ensure the $l_{\infty}$ norm sensitivity, then $\tilde{v}$ is within the intersection of an $l_1$ ball of radius $c^2$ and a hyper cube, where the length of each side is $(c^2/p)$. Then, $\max_{v} \sum_{j=1}^d h(v_j)$ is achieved when $v  = (c/\sqrt{p}, ... ,c/\sqrt{p}, 0,...,0)$ whose Hamming weight is $p$. The rest proof is a straightforward application of Theorem \ref{thm:composition-RDP} to concert RDP results to $(\epsilon, \delta)$ DP to get the composition bound claimed. 

\begin{remark}
\label{rmk:laplace}
When we apply Laplace mechanism to handle the case of $l_1$-norm sensitivity, i.e., $\sum_{j=1}^d |v_j| \leq c$, the corresponding $h(v_i)$ is still convex with respect to $v_i$. Therefore, the maximum of $\sum_{j=1}^d h(v_i)$ is still achieved when $v$ is a one-hot vector and we can reduce the multi-dimensional problem to the single dimension scenario. 
\end{remark}

\section{Proof of Theorem 3.3}\label{app:proof_amplification}
\noindent \textbf{Proof of Sketch}:  Following the arguments of Theorem \ref{thm:RDP}, still let $\mathcal{M}$ be the mechanism of one iteration of DP-SGD with ModelMix, and $w_1, w_2, ... , w_T$ be the outputs returned across $T$ iterations. To derive an $(\epsilon, \delta)$-DP bound of the cumulative privacy loss, it suffices to consider the high probability bound of the sum of point-wise privacy loss
{\small{
$$\sup_{\mathcal{D},\mathcal{D}',\text{Aux}} \sum_{k=1}^T \epsilon_{\mathcal{D},\mathcal{D}',\mathsf{Aux}}(w_k|w_{[1:k-1]}),$$
}}\noindent
for $\epsilon_{\mathcal{D},\mathcal{D}',\mathsf{Aux}}(w_k|w_{[1:k-1]}) = \log \frac{\mathbb{P}(\mathcal{M}(\mathcal{D}, \mathsf{Aux})= w_k|w_{[1:k-1]}) }{\mathbb{P}(\mathcal{M}(\mathcal{D'},\mathsf{Aux})= w_k|w_{[1:k-1]})}.$ In other words, to ensure $(\epsilon, \delta)$ DP of the cumulative privacy across $T$ iterations, it suffices to ensure
{\small{=
$$\Pr_{w_{[1:T]}}(\sup_{\mathcal{D},\mathcal{D}',\text{Aux}} \sum_{j=1}^T \epsilon_{\mathcal{D},\mathcal{D}',\mathsf{Aux}}(w_k|w_{[1:k-1]}) \leq \epsilon) \geq 1-\delta. $$
}}\noindent
In our application, we may apply Bernstein-Azuma inequality \cite{wainwright2019high} to derive a high probability bound of the sum of $\epsilon_{\mathcal{D},\mathcal{D}',\mathsf{Aux}}(w_k|w_{[1:k-1]})$ for $k=1,2,...,T$. Thus, we calculate the worst case expectation and the variance of $\epsilon_{\mathcal{D},\mathcal{D}',\mathsf{Aux}}(w_k|w_{[1:k-1]})$, respectively, and then apply Bernstein-Azuma inequality to  obtain the bound claimed. The details of calculation can be found below.

\subsection{Gaussian mechanism case}
\begin{proof}

Since we are interested in the asymptotic behavior of $\bar{\tau}$ and $q$ is a constant, it suffices to consider the case when $q=1$.
When $q$ = 1, we are essentially considering the full-batch GD and one can generalize the following analysis via the privacy amplification theorem by sampling \cite{dwork2014algorithmic}. Then, the two output distributions of (\ref{agg_model}) from $\mathcal{D}$ and $\mathcal{D}'$ we aim to compare are equivalent to  
{\small{
$$ \mathbb{P}(o) = \int_{-\frac{\bar{\tau}}{2}}^{\frac{\bar{\tau}}{2}}\frac{1}{\bar{\tau}\sqrt{2\pi\sigma^2}} e^{-\frac{(o-\alpha-c)^2}{2\sigma^2}} d \alpha,$$
}}\noindent
and 
{\small{
$$ \mathbb{P}'(o) = \int_{-\frac{\bar{\tau}}{2}}^{\frac{\bar{\tau}}{2}}\frac{1}{\bar{\tau}\sqrt{2\pi\sigma^2}} e^{-\frac{(o-\alpha)^2}{2\sigma^2}} d \alpha.$$
}}\noindent

We define the pointwise $\epsilon(o)$ loss at $o$ as
$ \epsilon(o) = \log \frac{\mathbb{P}(o)}{\mathbb{P}'(o)}.$\footnote{Due to the symmetry of both uniform and Gaussian distributions, it is sufficient to only consider the case by defining $\epsilon(o) = \log \frac{\mathbb{P}(o)}{\mathbb{P}'(o)}$.} Regarding $\epsilon(o)$, we have the following observation. First, 
{\small{\[
\mathbb{P}(o) = \frac{\Phi( \frac{o-c-\bar{\tau}/2}{\sigma})-\Phi(\frac{o-c+\bar{\tau}/2}{\sigma})}{\bar{\tau}},
\]
\[
\mathbb{P}'(o) = P(o + c) = \frac{\Phi(\frac{o-\bar{\tau}/2}{\sigma}) - \Phi( \frac{o+\bar{\tau}/2}{\sigma})}{\bar{\tau}},
\]
}}\noindent
where $\Phi(t) = \int_{t}^{\infty} e^{-t^2/2}/\sqrt{2\pi}.$ Regarding $\Phi(t)$, we have the following folk lemma that for any $t >0$,
{\small{
\begin{equation}\label{equ:errorF}
    \frac{\sqrt{2}e^{-t^2/2}}{\sqrt{\pi}(t + \sqrt{t^2+4})} \leq  \Phi(t) \leq \frac{\sqrt{2}e^{-t^2/2}}{\sqrt{\pi}(t + \sqrt{t^2+8/\pi})}.
\end{equation}
}}\noindent
For any $t < 0$, we can also bound $\Phi(t)$ since $\Phi(t) = 1 - \Phi(-t)$.

To derive the composition bound of $(\epsilon ,\delta)$, we provide the following lemma to capture the mean and variance of $\epsilon(o)$, respectively.  

\begin{lemma} \label{lem:boundEpsilon}
When $\bar{\tau} > \max\{ 2\sigma\log \bar{\tau}+ 2c, e^{3c / \sigma}, 2\}$ and $c / \sigma = \Theta(1)$, for any $\rho$ such that
{\small{
\[
\log\rho > {c\over \sigma} + \sqrt{2\log(\bar{\tau}\cdot \max\{c, 1\})},
\]
}}\noindent
we have
{\small{
\[
\begin{aligned}
\mathbb{E}_{o \sim P}~ \epsilon(o) 
=
O({c\log^2\rho \over \bar{\tau}}) 
+ O({\sigma\log\rho\over \bar{\tau}}),
\end{aligned}
\]
}}\noindent
and the variance $\mathcal{V}_0 = \text{Var}(\epsilon(o))$ satisfies
{\small{
\[
\mathcal{V}_0
= 
{1\over \bar{\tau}}\cdot O({c^4\log\rho\over \sigma^3} + {c^2\log^3\rho\over \sigma} + \sigma\log\rho).
\]
}}\noindent
\end{lemma}

\begin{proof}
Let us consider the mean first.
Note that 
{\small{
\[
\begin{aligned}
\mathbb{E}_{o \sim P}~ \epsilon(o)
& =
\int P(o) \log({P(o)\over P'(o)})\\
& \le
\int_{P(o) > P'(o)} P(o) \log({P(o)\over P'(o)})\\
& =
\int_{c/2}^{+\infty} P(o) \log({P(o)\over P'(o)}).
\end{aligned}
\]
}}\noindent
Therefore, we only consider the case when $o > c/2$.
Let $\rho$ denote a failure probability parameter that will be specified later.
We will divide the $[c/2, +\infty)$ integral into four parts 
(1) $[c/2, \bar{\tau}/2 - \sigma\log\rho)$, 
(2) $[\bar{\tau}/2 - \sigma\log\rho, \bar{\tau}/2 + 2c)$, 
(3) $[\bar{\tau}/2 + 2c, \bar{\tau}/2 + \sigma\log\rho)$ and
(4) $[\bar{\tau}/2 + \sigma\log\rho, +\infty)$
and compute the integral of these four parts correspondingly.
When $o\in [c/2, \bar{\tau}/2 - \sigma\log\rho]$,
{\small{
\[
\log{P(o)\over P'(o)}
= \log (1 + {P(o) - P'(o)\over P'(o)})
< {P(o) - P(o + c)\over P(o + c)}.
\]
}}\noindent
By definition, we have
{\small{
\[
P(o + c) 
= {1\over \bar{\tau}}\cdot (1 - \Phi({-o + \bar{\tau}/ 2 \over \sigma}) - \Phi({o + \bar{\tau}/ 2 \over \sigma}))
> {1 - 2\cdot\Phi(\log\rho) \over \bar{\tau}},
\]
\[
\begin{aligned}
P(o) - P(o + c)
=~ &~
{1\over \bar{\tau}}\cdot 
    (\Phi({-o - c + \bar{\tau}/ 2 \over \sigma}) - \Phi({-o + \bar{\tau}/ 2 \over \sigma})) \\
&  + {1\over \bar{\tau}}\cdot
    (\Phi({o + \bar{\tau}/ 2 \over \sigma}) - \Phi({o - c + \bar{\tau}/ 2 \over \sigma}))
     \\
<~ & {1\over \bar{\tau}}\cdot \Phi({-o - c + \bar{\tau}/ 2 \over \sigma})
< {1\over \bar{\tau}}\cdot \Phi(\log \rho - {c \over \sigma}).
\end{aligned}
\]
}}\noindent
Combining these with Equation (\ref{equ:errorF}), we have that for any $o\in [c/2, \bar{\tau}/2 - \sigma\log\rho)$,
{\small{
\[
\log{P(o)\over P'(o)}
< {\Phi(\log \rho- { c \over \sigma})
    \over 
    1 - 2\cdot\Phi(\log \rho)}
< {e^{-(\log \rho - c/\sigma)^2 / 2}\over \log \rho - c/\sigma}.
\]
}}\noindent
If we choose $\rho$ such that
{\small{
\begin{equation}
\label{equ:rhoReq0}
\log \rho > {c\over \sigma} + \sqrt{2\log \bar{\tau}},
\end{equation}
}}\noindent
then
{\small{
\[
e^{-(\log \rho - c/\sigma)^2 / 2} < {1\over \bar{\tau}}
\implies
\log{P(o)\over P'(o)} < {1\over \bar{\tau}\cdot \sqrt{2\log \bar{\tau}}}.
\]
}}\noindent
Thus, the integral of the first part $[c/2, \bar{\tau}/2 - \sigma\log\rho]$ is,
{\small{
\[
\int_{c/2}^{\bar{\tau}/2 - \sigma\log\rho} P(o)\log{P(o)\over P'(o)} < {1\over \bar{\tau}\cdot \sqrt{2\log \bar{\tau}}}.
\]
}}\noindent

Now, we consider the second part where $o\in [\tau/2 - \sigma\log\rho, \bar{\tau}/2 + 2c)$.
In this case, utilizing Equation (\ref{equ:errorF}), we can show that
{\small{
\[
\begin{aligned}
\log{P(o)\over P'(o)}
& <
\log{1 / \bar{\tau}\over P'(o)}
\le
\log{1\over \Phi({2c\over \sigma}) - \Phi({2c + \bar{\tau}/2 \over \sigma})}
=
O({c^2\over \sigma^2}).
\end{aligned}
\] 
}}\noindent
Therefore,
{\small{
\[
\begin{aligned}
\int_{\tau/2 - \sigma\log\rho}^{\bar{\tau}/2 + 2c} P(o)\log{P(o)\over P'(o)} d o
& =
O({c^2\over \sigma^2}) \cdot {2c + \sigma\log\rho\over \bar{\tau}} \\
& =
O({c^2 \log\rho\over \sigma\bar{\tau}}).
\end{aligned}
\]
}}\noindent
When $o \ge {\tau/2} + 2c$, we have
{\small{
\[
\begin{aligned}
{P(o)\over P'(o)}
& =
{\Phi({o - c - \bar{\tau}/2 \over \sigma}) - \Phi({o - c + \bar{\tau}/2 \over \sigma})
    \over
\Phi({o - \bar{\tau}/2 \over \sigma}) - \Phi({o + \bar{\tau}/2 \over \sigma})}
<
2 \cdot {\Phi({o - c - \bar{\tau}/2 \over \sigma})
    \over
    \Phi({o - \bar{\tau}/2 \over \sigma})} \\
& < 
4 \cdot e^{(2(o - \bar{\tau}/2)c - c^2) / (2\sigma^2)}.
\end{aligned}
\]
}}\noindent
Here, the constant factor two in the first inequality is derived based on Equation (\ref{equ:errorF}) and $\bar{\tau} > 1$.
When $\bar{\tau}$ is large, this factor will approximate 1.
We use 2 here as a simple relaxation.
Therefore,
{\small{
\[
\log {P(o)\over P'(o)}
<
{2(o - \bar{\tau}/2)c - c^2 \over 2\sigma^2} + \log(4).
\]
}}\noindent
For convenience, we will denote $z = \log(4) - c^2 / 2\sigma^2$ such that
{\small{
\[
\log {P(o)\over P'(o)}
<
{(o - \bar{\tau}/2)c \over \sigma^2} + z.
\]
}}\noindent
This allows us to compute the third and the fourth part as follows.
When $o\in [\bar{\tau}/2 + 2c, \bar{\tau}/2 + \sigma\log \rho)$,
{\small{
\[
\begin{aligned}
& \int_{\bar{\tau}/2 + 2c}^{\bar{\tau}/2 + \sigma\log \rho} P(o)\log{P(o)\over P'(o)} d o \\
&  <
\int_{\bar{\tau}/2 + 2c}^{\bar{\tau}/2 + \sigma\log \rho} P(o)\cdot ({(o - \bar{\tau}/2)c\over \sigma^2} + z)  d o \\
& <
({c\log\rho\over \sigma} + z ) \cdot \int_{\bar{\tau}/2 + 2c}^{\bar{\tau}/2 + \sigma\log\rho} P(o) d o  \\
& <
({c\log\rho\over \sigma} + z )\cdot (\sigma\log\rho - 2c)\cdot P(\bar{\tau}/2 + 2c) \\
& =
O({c\log^2\rho \over \bar{\tau}}) + O({\sigma\log\rho\over \bar{\tau}}).
\end{aligned}
\]
}}\noindent
Finally, when $o\in [\bar{\tau}/2 + \sigma\log \rho, +\infty)$,
{\small{
\[
\begin{aligned}
&
\int_{{\bar{\tau}\over 2} + \sigma\log \rho}^{+\infty} P(o)\log{P(o)\over P'(o)} d o \\
& <
\int_{{\bar{\tau}\over 2} + \sigma\log \rho}^{+\infty} P(o)\cdot ({(o - \bar{\tau}/2)c\over \sigma^2} + z) do \\
& = O(
{1\over \bar{\tau}}\int_{{\bar{\tau}\over 2} + \sigma\log \rho}^{+\infty} e^{-(o - \bar{\tau}/2 - c)^2/(2\sigma^2)}\cdot {(o - \bar{\tau}/2)c\over \sigma^2}do ) \\
& = O(c
e^{-(\sigma\log \rho - c)^2/(2\sigma^2)}).
\end{aligned}
\]
}}\noindent
In order for the integral of this part to be negligible, $\rho$ must be chosen such that
{\small{
\begin{equation}
\label{equ:rhoReq1}
ce^{-(\sigma\log \rho - c)^2/(2\sigma^2)} < {1\over \bar{\tau}}
\implies
\log\rho > {c\over \sigma} + \sqrt{2\log(\bar{\tau}c)}.
\end{equation}
}}\noindent
Combine this with the requirement in Equation (\ref{equ:rhoReq0}), we have the final requirement for $\rho$ that
{\small{
\begin{equation}
\label{equ:rhoReq}
\log\rho > {c\over \sigma} + \sqrt{2\log(\bar{\tau}\cdot \max\{c, 1\})}.
\end{equation}
}}\noindent

In conclusion, we can sum up the integral of the four parts to show that for any $\rho$ satisfying Equation (\ref{equ:rhoReq}),
{\small{
\[
\mathbb{E}_{o \sim P}~ \epsilon(o) 
= 
O({c^2 \log\rho\over \sigma\bar{\tau}}) 
+ O({c\log^2\rho \over \bar{\tau}}) 
+ O({\sigma\log\rho\over \bar{\tau}}) 
+ O({1\over \bar{\tau}}).
\]
}}\noindent
Since $\log\rho > c / \sigma$, we can simplify the above equation to
{\small{
\[
\mathbb{E}_{o \sim P}~ \epsilon(o) 
=  O({c\log^2\rho \over \bar{\tau}}) 
+ O({\sigma\log\rho\over \bar{\tau}}).
\]
}}\noindent
If we consider $c$ or $\sigma$ to be constant, then we can simply set $\rho = \Theta(\bar{\tau})$ to satisfy Equation (\ref{equ:rhoReq}) and show that
{\small{
\[
\mathbb{E}_{o \sim P}~ \epsilon(o) 
= O({c\log^2\bar{\tau} \over \bar{\tau}}) 
+ O({\sigma\log\bar{\tau}\over \bar{\tau}}).
\]
}}\noindent

Let us now consider the variance.
{\small{
\[
\begin{aligned}
\mathcal{V}_0 (\epsilon(o)) 
&
= \mathbb{E}_{o \sim P} \epsilon^2(o)
= \int_{-\infty}^{c/2} P(o)\epsilon^2(o) do + \int_{c/2}^{+\infty} P(o)\epsilon^2(o) do
\end{aligned}
\]
}}\noindent
Note that for any $o > c/2$,
{\small{
\[
\begin{aligned}
\epsilon^2(o) 
& 
= \log^2{P(o)\over P'(o)}
= \log^2{P'({o})\over P(o)} \\
& 
= \log^2{P'({-o})\over P(2c - o)}
= \log^2{P({c - o})\over P'(c - o)} \\
&
= \epsilon^2(c - o).
\end{aligned}
\]
}}\noindent
However, due to property of the Gaussian function, for any $o > c/2$,
{\small{
\[
P(o) = P(2c - o) > P(c-o).
\]
}}\noindent
Therefore,
{\small{
\[
\begin{aligned}
\int_{-\infty}^{c/2} P(o)\epsilon^2(o) do 
& =
\int_{c/2}^{+\infty} P(c-o)\epsilon^2(c-o) do \\
& =
\int_{c/2}^{+\infty} P(c-o)\epsilon^2(o) do \\
& < 
\int_{c/2}^{+\infty} P(o)\epsilon^2(o) do.
\end{aligned}
\]
}}\noindent
Since we have already bounded $\epsilon(o)$ when computing the mean for any $o > c/2$,
    we have
{\small{
\[
\begin{aligned}
\mathcal{V}_0 (\epsilon(o)) 
& < 2 \int_{c/2}^{+\infty} P(o)\epsilon^2(o) do \\
& = {1\over \bar{\tau}}\cdot O({c^4\log\rho\over \sigma^3} + {c^2\log^3\rho\over \sigma} + \sigma\log\rho).
\end{aligned}
\]
}}\noindent
\end{proof}

The proof in cemma \ref{lem:boundEpsilon} also implies that $\epsilon(o) < 1 / (\bar{\tau}\cdot \sqrt{2\log \bar{\tau}})$ in the first part $o\in [c/2, \bar{\tau}/2 - \sigma\log \rho)$.
Furthermore, the probability of $o$ being in the rest three parts is bounded by
{\small{
\[
\Pr[o \ge \bar{\tau}/2 - \sigma\log \rho]
<
{2\sigma\log\rho \over \bar{\tau}}.
\]
}}\noindent
We can choose $\rho$ to be the minimum that satisfies Equation (\ref{equ:rhoReq}),
which implies that
{\small{
\[
\Pr[o \ge \bar{\tau}/2 - \sigma\log \rho]
<
{2c + 2\sigma\sqrt{2\log (\bar{\tau} \cdot \max\{c, 1\})} \over \bar{\tau}}.
\]
}}\noindent
Therefore,
{\small{
\[
\begin{aligned}
\Pr\Big[\epsilon(o) \ge {1 \over \bar{\tau}\cdot \sqrt{2\log \bar{\tau}}}\Big]
& <
{2c + 2\sigma\sqrt{2\log (\bar{\tau} \cdot \max\{c, 1\})} \over \bar{\tau}}.
\end{aligned}
\]
}}\noindent
For simplicity, let us define 
{\small{
\[
\epsilon_0 = {1 \over \bar{\tau}\cdot \sqrt{2\log \bar{\tau}}}
\text{~~and~~}
\delta_0 = {2c + 2\sigma\sqrt{2\log (\bar{\tau} \cdot \max\{c, 1\})} \over \bar{\tau}}.
\]
}}\noindent
The conditional expectation and variance satisfy
{\small{
\[
\mathbb{E}_{o \sim P}~[\epsilon(o)~|~\epsilon(o) < \epsilon_0] = {\mathbb{E}_{o \sim P} \epsilon(o)\over 1 - \delta_0},
\]
\[
\mathbb{E}_{o \sim P}~[\epsilon^2(o)~|~\epsilon(o) < \epsilon_0] = {\mathcal{V}_0\over 1 - \delta_0}.
\]
}}\noindent

Now, we may apply the Bernstein inequality \cite{wainwright2019high} to give a concentration bound. 
Let $o_1, o_2, ... ,o_T$ be the intermediate updates from Algorithm \ref{alg: ModelMix} across $T$ iterations. Provided fresh randomness in each iteration, $\epsilon(o_k)$ forms an martingale. Conditional on that $\epsilon(o_k)$ that $|\epsilon(o_k)|\leq \epsilon_0$ for $k=1,2,...,T$, where we can apply Bernstein inequality, we have that with probability at least $1-(T\delta_0+ \tilde{\delta})$,
{\small{
\begin{equation}
\begin{aligned}
\label{berstein_gau}
& \sum_{k=1}^T \epsilon(o_k)  = O\big( \frac{T\mathbb{E}[\epsilon(o)] + \sqrt{\log(1/\tilde{\delta})(T\mathcal{V}_0+\epsilon_0)} }{1-\delta_0} \big)\\
& =  \tilde{O}\big(\frac{T(c+\sigma)}{\bar{\tau}} + \sqrt{\log(1/\tilde{\delta})}(\sqrt{\frac{c^4/\sigma^3 + c^2/\sigma + \sigma}{\bar{\tau}}} + \frac{1}{\sqrt{\tau}})\big),
\end{aligned}
\end{equation}
}}\noindent
for any $\tilde{\delta} \in (0,1)$. Finally, we scale the parameters back to the case of Algorithm \ref{alg: ModelMix}, where the sensitivity is $c/B$ and $\bar{\tau}=\tau/\eta$, we have the bound claimed. 
\end{proof}

\subsection{Laplace mechanism case}
\begin{proof}
In the rest of the proof, we turn to consider the Laplace scenario with bounded $l_1$-norm sensitivity $c$. Now we assume that the noise $\Delta$ is a Laplace distribution of parameter $\lambda$, i.e., $\mathbb{P}(\Delta = o) = \frac{\lambda}{2} e^{-\lambda|o|}$. With a similar reasoning and by Remark \ref{rmk:laplace}, it is equivalent to considering the following distribution, 
$$ o \sim \mathcal{U}[0,\bar{\tau}] + \text{cap}(\lambda).$$
With some calculation, we know its probability density function (pdf) can be expressed as follows,
{\small{
\begin{equation*}
P(o)=\left\{
\begin{aligned}
 \frac{e^{-\lambda |o|}(1-e^{-\lambda \bar{\tau}})}{2\bar{\tau}} & , & o \leq 0; \\
 \frac{2-e^{-\lambda o} - e^{-\lambda(\bar{\tau} - o)}}{2\bar{\tau}}& , & o \in (0,\bar{\tau}); \\
 \frac{e^{-\lambda |o-\bar{\tau}|}(1-e^{-\lambda \bar{\tau}})}{2\tau} & , & o \geq \bar{\tau}.
\end{aligned}
\right.
\end{equation*}
}}\noindent
Similarly, we use $P'(o)$ to denote the pdf of  $\mathcal{U}[c,c+\tau] + \text{cap}(\lambda)$. We will use $\epsilon_0 = {\lambda c}$ in the following.

\begin{lemma} When $\bar{\tau} > \max\{ 2\log \bar{\tau}/\lambda+c,1\}$, 
{\small{
\[
\begin{aligned}\mathbb{E}_{o \sim P}~ \epsilon(o) 
&\leq
\epsilon_0(e^{\epsilon_0}-1)(\frac{1-e^{-\lambda \bar{\tau}}}{\bar{\tau}\lambda} + \frac{2\log \bar{\bar{\tau}}/\lambda  +c}{\bar{\tau}}) \\
&~~~~
+ \frac{e^{\epsilon_0}-1}{2(\bar{\tau}-1)}(e^{\frac{\epsilon_0}{2(\bar{\tau}-1)}}-1)( 1-\frac{1-e^{-\lambda \bar{\tau}}}{\bar{\tau}\lambda}).
\end{aligned}
\]
}}\noindent
and the variance $\mathcal{V}_0 = \text{Var}(\epsilon(o))$ satisfies 
{\small{
$$
\mathcal{V}_0  \leq 
\epsilon^2_0(\frac{1-e^{-\lambda \bar{\tau}}}{\bar{\tau}\lambda} + \frac{2\log \bar{\tau}/\lambda +s}{\bar{\tau}}) + (\frac{e^{\epsilon_0}-1}{2(\bar{\tau}-1)})^2( 1-\frac{1-e^{-\lambda \bar{\tau}}}{\bar{\tau}\lambda}).
$$
}}\noindent
\end{lemma}
\begin{proof}
We split the calculation of $\mathbb{E}_{o \sim P} \epsilon(o)$ into two parts, where
{\small{
\begin{equation*}
\begin{aligned}
    \mathbb{E}_{o \sim P}~ \epsilon(o) 
    & = \int_{\mathbb{R}} P(o)\log \frac{P(o)}{P'(o)} ~ do \\
    & \leq \int_{\mathbb{R}} P(o)\log \frac{P(o)}{P'(o)} ~ do + \int_{\mathbb{R}} P'(o)\log \frac{P'(o)}{P(o)} ~ do  \\
    & = \int_{\mathbb{R}} \big [ P(o)(\log \frac{P(o)}{P'(o)} \\
    & + \log \frac{P'(o)}{P(o)}) + (P'(o)-P(o))\log\frac{P'(o)}{P(o)} \big] do\\
    & = \int_{\mathbb{R}}  (P'(o)-P(o))\log\frac{P'(o)}{P(o)}  do\\
    & \leq \underline{\int_{\mathbb{R}/\mathcal{I}} \sup_{o \in \mathbb{R}/\mathcal{I}}|\log\frac{P'(o)}{P(o)}|\cdot|P'(o)-P(o)| do}_{(A)}\\
    & \quad
    + \underline{\int_{\mathcal{I}} \sup_{o \in \mathcal{I}}|\log\frac{P'(o)}{P(o)}|\cdot|P'(o)-P(o)| do}_{(B)}.
\end{aligned}
\end{equation*}
}}\noindent
Here, $\mathcal{I} = [c+\log \bar{\tau}/\lambda, \bar{\tau}-\log\bar{\tau}/\lambda]$. 
We first handle (A), where it is easy to verify that ModelMix does not change the worst case $\epsilon$ bound and $ \sup_{o \in \mathbb{R}/\mathcal{I}}|\log\frac{P'(o)}{P(o)}| \leq \epsilon_0 = \lambda c$. 
Thus, $|P'(o)-P(o)| \leq (e^{\epsilon_0}-1)P(o)$. 
On the other hand, we can upper bound the probability $\Pr(o \in \mathbb{R}/\mathcal{I})$ as follows. 
It can be computed that $\Pr(o \leq 0) = \Pr(o \geq \bar{\tau} ) = \frac{1-e^{-\lambda \bar{\tau}}}{2\bar{\tau}\lambda}.$ 
As for $o \in (0,\bar{\tau})$, $P(o)$ is indeed concentrated at its mean $\bar{\tau}/2$, and thus
{\small{
\begin{equation}
\begin{aligned}
\Pr(o \in \mathbb{R}/\mathcal{I}) & \leq 2\cdot \frac{1-e^{-\lambda\bar{\tau}}}{2\bar{\tau}\lambda} + (1-\frac{|\mathcal{I}|}{\bar{\tau}})\cdot (1-2\frac{1-e^{-\lambda \bar{\tau}}}{2\bar{\tau}\lambda}) \\
& \leq \frac{1-e^{-\lambda \bar{\tau}}}{\bar{\tau}\lambda} + \frac{2\log \bar{\tau}/\lambda +c}{\bar{\tau}}. 
\end{aligned}
\end{equation}
}}\noindent
Therefore, we have $(A) \leq \epsilon_0(e^{\epsilon_0}-1)(\frac{1-e^{-\lambda \bar{\tau}}}{\bar{\tau}\lambda_0} + \frac{2\log \bar{\tau}/\lambda +c}{\bar{\tau}})$. 
\\
\\
Now, we turn to bound (B). for sufficiently large 
{\small{
\[
\bar{\tau} > \max\Big\{ c+\frac{2\log \bar{\tau}}{\lambda}, 1\Big\},
\]
}}\noindent
we have for $o \in \mathcal{I}$, 
{\small{
\begin{equation}
    \label{middle}
    \small{
    \begin{aligned}
         |\epsilon(o)| 
         & = \big|\log\big(1 + \frac{e^{-\lambda(o-c)}+e^{-\lambda(\bar{\tau}+c-o)} - e^{-\lambda(o)}-e^{-\lambda(\bar{\tau}-o)}}{2-e^{-\lambda(o-c)}-e^{-\lambda(\bar{\tau}+c-o)}}\big)\big| \\
         & = \big|\log\big(1 + \frac{ e^{-\lambda o}(e^{\lambda c}-1) + e^{-\lambda(\bar{\tau}-o)}(e^{-\lambda c}-1)}{2-e^{-\lambda(o-c)}-e^{-\lambda(\bar{\tau}+c-o)}}\big) \big| \\
         & \leq \frac{(e^{\lambda c}-1)/\bar{\tau}}{2-2/\bar{\tau}} = \frac{e^{\epsilon_0}-1}{2\bar{\tau}-2}.
    \end{aligned}
    }
\end{equation}
}}\noindent
Since 
{\small{
\[
\sup_{o \in \mathcal{I}} \epsilon(o) \leq \frac{e^{\epsilon_0}-1}{2(\bar{\tau}-1)}
\]
}}\noindent
and 
{\small{
$$\Pr(o \in \mathcal{I}) \leq 1-\Pr(o \in (-\infty,0)\cup(\bar{\tau}, +\infty)) = 1-\frac{1-e^{-\lambda \bar{\tau}}}{\tau\lambda},$$ 
}}\noindent
we have 
{\small{
\[
\begin{aligned}\mathbb{E}_{o \sim P}~ \epsilon(o) 
&\leq
\epsilon_0(e^{\epsilon_0}-1)(\frac{1-e^{-\lambda \bar{\tau}}}{\bar{\tau}\lambda} + \frac{2\log \bar{\tau}/\lambda  +c}{\bar{\tau}}) \\
&~~~~
+ \frac{e^{\epsilon_0}-1}{2(\bar{\tau}-1)}(e^{\frac{\epsilon_0}{2(\bar{\tau}-1)}}-1)( 1-\frac{1-e^{-\lambda \bar{\tau}}}{\bar{\tau}\lambda}).
\end{aligned}
\]
}}\noindent
In the following, we set out to characterize the variance of $\epsilon(o)$. It is noted that 
{\small{
\[
\begin{aligned}
\text{Var}(\epsilon(o)) 
& = \mathbb{E}\big[(\epsilon(o)-\mathbb{E}[\epsilon(o)])^2\big] 
\leq \mathbb{E}[(\epsilon(o)-0)^2] \\
& = \mathbb{E}[\epsilon(o)^2\cdot (\bm{1}_{o \in \mathcal{I}} +\bm{1}_{o \in \mathbb{R}/\mathcal{I}}) ].
\end{aligned}
\]
}}\noindent
We use the upper bound $|\epsilon(o)| \leq \frac{e^{\epsilon_0}-1}{2(\bar{\tau}-1)}$ when $o \in \mathcal{I}$ and for the rest we simply bound $|\epsilon(o)|\leq \epsilon_0$, and we have 
{\small{
\[
\begin{aligned}
\mathcal{V}_0 = \text{Var}(\epsilon(o))
\leq 
& 
\epsilon^2_0(\frac{1-e^{-\lambda \bar{\tau}}}{\bar{\tau}\lambda} + \frac{2\log \bar{\tau}/\lambda +s}{\bar{\tau}}) \\
& + (\frac{e^{\epsilon_0}-1}{2(\bar{\tau}-1)})^2( 1-\frac{1-e^{-\lambda \bar{\tau}}}{\bar{\tau}\lambda}).
\end{aligned}
\]
}}\noindent
\end{proof}

Still, applying Bernstein inequality, to ensure an $\delta$ failure probability, we may select 
{\small{
\[
\begin{aligned}
t 
& = \sqrt{2} \cdot \sqrt{T\mathcal{V}_0\log(1/\delta)+\log(1/\delta)\epsilon^2_0/9} +\log(1/\delta)\epsilon^2_0/3 \\
& = \tilde{O}(\sqrt{T\log(1/\delta)} \cdot \big(\epsilon_0  \sqrt{\frac{1}{\bar{\tau}\lambda}+\frac{s}{\bar{\tau}}} + \frac{e^{\epsilon_0}-1}{\bar{\tau}} \big),
\end{aligned}
\]
}}\noindent
and correspondingly,
{\small{
\begin{equation} 
\begin{aligned}
\epsilon & = \tilde{O}\Big( \epsilon_0\cdot {T\over \bar{\tau}}\cdot(e^{\epsilon_0}-1) + (\epsilon_0 + \frac{e^{\epsilon_0}-1}{\sqrt{\bar{\tau}}}) \cdot \sqrt{{T\over \bar{\tau}}\cdot \log(1/\delta)} \Big)\\
& = \tilde{O}(\frac{\eta T\epsilon_0(e^{\epsilon_0}-1)}{\tau} + \epsilon_0 \sqrt{\frac{\eta T \log(1/\delta)}{\tau}}).
\end{aligned}
\end{equation}
}}\noindent
\end{proof}

\section{Experiments Setups}
\label{app:exp}
\noindent \textbf{Experiments in Section \ref{sec:shallow}} (The comparisons to \cite{DanICLR2020} where we use DP-SGD to privately train a CNN on ScatterNetwork features.) 
For CIFAR-10, we reproduce the results of \cite{DanICLR2020} by choosing sampling rate $q=8192/50000$, clipping threshold $c=0.1$ with stepsize $\eta=4$ for $700$ iterations.
For FMNIST, to reproduce the results of \cite{DanICLR2020}, we choose sampling rate $q=8192/60000$, clipping threshold $c=0.1$ with stepsize $\eta=4$ for $800$ iterations. 
When we apply ModelMix, for both CIFAR10 and FMNIST, we adopt the same sampling rate as they use above.
But we use a larger clipping threshold $c=1$ with $\eta=0.4$ and run for the same number of iterations. 
As for the selection of $\tau_k$, we select $\tau_k=0.05\eta$ for the first half epochs and $\tau_k=0.025\eta$ for the rest.

The CNN model that we use is the same as that applied in \cite{DanICLR2020}. The parameters can be found in Table 8 and Table 9 in \cite{DanICLR2020}.

\noindent \textbf{Experiments in Section \ref{sec:deep}} (The experiments on training Resnet20 with DP-SGD.) 
For CIFAR-10, in the application of regular DP-SGD with per-sample clipping, we select the sampling rate $q=1500/50000$, clipping threshold $c=20$, $\eta=0.1$, and run for $3500$ iterations. When we apply ModelMix, we select the same $q=1500/50000$, $c=20$, but a larger $\eta=0.15$ and run for $3500$ iterations. We set $\tau_k=0.15\eta$ uniformly. 

For SVHN, when we test regular DP-SGD, we select the sampling rate $q=1830/73260$, clipping threshold $c=20$, $\eta=0.1$, and run for $4000$ iterations. When we apply ModelMix, we select the same $q=1830/73260$, $c=20$, but a larger $\eta=0.15$ and run for $4000$ iterations. We set $\tau_k=0.15\eta$ uniformly.

\noindent \textbf{Experiments in Section \ref{sec:pub_data}} 
(The experiments on DP-SGD applications with further access to public data.) 
In the private transfer learning application, we reproduce the results of \cite{DanICLR2020} on CIFAR-10 by applying DP-SGD to learn a linear model on the features transformed by pretrained SimCLRv2 network. 
The parameters are selected as they suggest, where $q=1024/50000$, $c=0.1$ and $\eta=4$, and run for 2700 iterations. 
When we incorporate ModelMix, we adopt the same $q=1024/50000$, $c=0.1$, but a larger stepsize $\eta=5$. $\tau$ is set as $0.0075\eta, 0.005\eta, 0.0025\eta$ for the first, second, and third 1/3 of the total 1500 iterations, respectively. Due to the privacy amplification of ModelMix, we use a smaller noise scaled by 0.85 compared to their selection. Consequently, we achieve $(\epsilon=0.64,\delta=10^{-5})$ with the same $92.7\%$ accuracy. 

When comparing to the low-dimensional embedding method presented in \cite{yu_public2020}, we adopt $q=1000/50000$, $c_0=5$ for the embedded gradient and $c_1=2$ for the residual gradient, and $\eta=0.1$ for 10000 iterations, as suggested, to reproduce $73.2\%$ accuracy on CIFAR-10 with Resnet20 with a budget $(\epsilon=8,\delta=10^{-5})$ under Gaussian Mechanism. When we apply ModelMix, we select $q=2000/50000$, $\eta=0.2$, $p=25$, and with the same clipping thresholds as above and run for 5000 iterations. With the privacy amplification of ModelMix, we uniformly scale the noise added by $0.75$ and select $\tau=0.05$ uniformly for all iterations. We achieve an $74.2\%$ accuracy at a budget $(\epsilon =2.9, \delta=10^{-5})$. With the same setup,  we can also achieve an accuracy of $79.1\%$ at a budget $(\epsilon=6.1, \delta=10^{-5})$.

\end{document}